\documentclass{imaiai}%%%%where imaiai is the template name
\usepackage{bm}
\usepackage{graphicx}
\usepackage{enumerate}
\usepackage{amssymb,amsmath,amsthm,mathrsfs}
\usepackage{algorithm}
\usepackage{amsfonts,delarray}
\usepackage{algorithm,algcompatible,amsmath}
\usepackage{rotating,color}
\usepackage{subfigure}
\usepackage{longtable}
\usepackage{bbm}
\usepackage{rotating}
\usepackage{times,url}
\newtheorem{assumption}{Assumption}

\newcommand{\var}{\mathrm{Var}}

\newcommand{\av}{\bm a}
\newcommand{\Av}{\bm A}
\newcommand{\bv}{\bm b}

\newcommand{\sign}{\mathrm{sign}}

\newcommand{\gv}{\bm g}
\newcommand{\Gv}{\bm G}
\newcommand{\hv}{\bm h}
\newcommand{\Iv}{\bm I}

\newcommand{\projv}{\bm \Pi}
\newcommand{\lambdav}{\bm \lambda}

\newcommand{\Deltav}{\bm \Delta}
\newcommand{\sv}{\bm s}
\newcommand{\Sigmav}{\bm \Sigma}
\newcommand{\uv}{\bm u}
\newcommand{\wv}{\bm w}
\newcommand{\vv}{\bm v}

\newcommand{\xv}{\bm x}
\newcommand{\yv}{\bm y}
\newcommand{\zv}{\bm z}

\DeclareMathOperator*{\argmin}{arg\,min}
\DeclareMathOperator*{\argmax}{arg\,max}

\begin{document}

\title{Does SLOPE outperform bridge regression?}

\shorttitle{Does SLOPE outperform bridge regression?} %%%for recto running head
\shortauthorlist{Shuaiwen, Haolei, and Arian} %%% for verso running head

\author{{%%%% First author details
\sc Shuaiwen Wang}\\[2pt]
Department of Statistics, Columbia University, NY 10027, USA\\
{sw2853@columbia.edu}\\[6pt]
%%%%%%% Second author details
{\sc Haolei Weng}$^*$\\[2pt]
Department of Statistics and Probability, Michigan State university, MI 48824, USA \\
{\email{wenghaol@msu.edu}}\\[6pt]
%%%%%%%
{\sc and}\\[6pt]
%%%%%%% Third author details
{\sc Arian Maleki} \\[2pt]
Department of Statistics, Columbia University, NY 10027, USA \\
{arian@stat.columbia.edu}}

\maketitle

\begin{abstract}
{    
A recently proposed SLOPE estimator \cite{bogdan2015slope} has been shown to adaptively achieve the minimax $\ell_2$ estimation rate under high-dimensional sparse linear regression models \cite{su2016slope}. Such minimax optimality holds in the regime where the sparsity level $k$, sample size $n$, and dimension $p$ satisfy $k/p\rightarrow 0, k\log p/n\rightarrow 0$. In this paper, we characterize the estimation error of SLOPE under the complementary regime where both $k$ and $n$ scale linearly with $p$, and provide new insights into the performance of SLOPE estimators. We first derive a concentration inequality for the finite sample mean square error (MSE) of SLOPE. The quantity that MSE concentrates around takes a complicated and implicit form. With delicate analysis of the quantity, we prove that among all SLOPE estimators, LASSO is optimal for estimating $k$-sparse parameter vectors that do not have tied non-zero components in the low noise scenario. On the other hand, in the large noise scenario, the family of SLOPE estimators are sub-optimal compared with bridge regression such as the Ridge estimator. 
}
{Concentration inequality, LASSO, mean square error, noise sensitivity, Ridge, SLOPE}
%%%% If classification number provided then
\\
2000 Math Subject Classification: 34K30, 35K57, 35Q80,  92D25
\end{abstract}

\section{Introduction}\label{sec:intro}

In high-dimensional statistics, one of the most fundamental problems is the estimation of $k$-sparse parameter vector $\xv \in \mathbb{R}^p$ in the linear regression model:
\begin{equation} \label{eq:model-setup0}
    \yv = \Av \xv + \zv,
\end{equation}
where $\Av \in \mathbb{R}^{n \times p}$ is the design matrix, $\yv \in \mathbb{R}^n$ is the response vector, and $\zv \in \mathbb{R}^n$ is the noise vector of independent entries with variance $\sigma_z^2$. It has been established that the minimax rate for estimating $\xv$ over the class of $k$-sparse parameters is $(k/n)\log (p/k)$ \cite{ye2010rate, raskutti2011minimax, verzelen2012minimax}. Several $\ell_1$ based methods such as the LASSO and the Dantzig selector were proved to obtain the rate $(k/n)\log p$ \cite{candes2004near, bickel2009simultaneous, negahban2012unified}. However, it was largely unknown whether there exists a computationally feasible approach to adaptively achieve the optimal minimax rate $(k/n)\log (p/k)$ untile recent years. There has been considerable progress since \cite{bogdan2015slope} introduced the
\textit{sorted-$\ell_1$ penalized estimator (SLOPE)}, defined as
\begin{equation} \label{eq:slope-solution}
    \hat{\xv} \in \argmin_{\xv} \frac{1}{2}\|\yv-\Av \xv\|_2^2+ \sum_{i=1}^p \lambda_i |\xv|_{(i)},
\end{equation}
where $\lambda_1\geq\lambda_2 \geq \cdots \geq \lambda_p \geq 0$ is a sequence of nonincreasing weights, and $|\xv|_{(i)}$ denotes the $i$\textsuperscript{th} largest value of $\{|x_i|\}_{i=1}^p$. The regularization term $\sum_{i=1}^p \lambda_i |\xv|_{(i)}$, as a function of $\xv$, is a norm in $\mathbb{R}^p$ \cite{bogdan2015slope}. Hence \eqref{eq:slope-solution} is a convex optimization problem and can be solved in polynomial time. SLOPE was originally proposed to control \textit{false discovery rate (FDR)} in the problems of multiple testing and variable selection. Shortly thereafter, \cite{su2016slope} proved that for the Gaussian random design $\Av$ with $A_{ij} \overset{i.i.d.}{\sim} \mathcal{N}(0, \frac{1}{n})$, with the choice $\lambda_i=\sigma_z\cdot \Phi^{-1}(1-(iq)/(2p))$ where $\Phi(\cdot)$ is the cdf of a standard normal and $0<q<1$ is a fixed constant, SLOPE attains the minimax estimation rate $(k/n)\log (p/k)$ in the asymptotic regime $k/p\rightarrow 0, k\log p/n\rightarrow 0$. The same result was extended to designs with independent sub-Gaussian entries in \cite{lecue2018regularization}. The recent work \cite{bellec2018slope} further showed that SLOPE continues to achieve the optimal rate for more general designs satisfying a restricted eigenvalue type condition. The authors also proved that the rate of LASSO can be improved to $k/n\log(p/k)$ with tuning parameter of order $\sqrt{\log(p/k)/n}$\footnote{This requires the knowledge of the sparsity $k$. When $k$ is unknown, the paper proposed an adaptive method for LASSO to achieve the same rate.}. According to the aforementioned results, both LASSO (optimally tuned) and SLOPE (with appropriately chosen $\{\lambda_i\}$) attain the optimal rate. The question then arises as to which one of the two estimators is better. We note that LASSO is a special case of SLOPE by choosing $\lambda_i=\lambda~(i=1,\ldots, p)$. Thus, the question can be generally formulated as the comparison of different SLOPE estimators. This problem is not only theoretically appealing, but can provide helpful guidance for practitioners to pick the right method. 

In this paper, we address the above question by providing a refined analysis of the mean square error (MSE) of SLOPE. Rather than order-wise results, the comparison of rate optimal estimators requires a sharp characterization of MSE. We will derive the sharp expression of MSE, and evaluate the expression for different SLOPE estimators. Along this line, we further leverage the high-dimensional asymptotic results of bridge regression estimators \cite{wang2017bridge} for the comparison to shed more light on the performance of SLOPE. Our main contributions can be summarized in the following:
\begin{enumerate}
    \item We provide concentration inequalities for the finite-sample MSE of SLOPE estimators under different scenarios. The quantity that MSE concentrates around is characterized by a system of non-linear equations. 
    \item We characterize the phase transition and low noise sensitivity of SLOPE. The results show that LASSO has the optimal phase transition and low noise sensitivity performance among all the SLOPE estimators, for the estimation of sparse signals without tied non-zero components.
    \item We prove that in the large noise setting, all the SLOPE estimators are outperformed by a family of bridge regression estimators such as the Ridge regression.
\end{enumerate}

\paragraph{\emph{Related Works.}} To obtain precise error characterization, we focus on the high-dimensional regime where both the sparsity $k$ and sample size $n$ scale linearly with the dimension $p$. This asymptotic framework evolved in a series of papers by Donoho and Tanner \cite{donoho2006most, donoho2005neighborliness, donoho2005sparse, donoho2006high} to characterize the phase transition curve for LASSO and some of its variants. Since then several analytical tools have been developed and adopted to study different problems under this asymptotic setting. Examples include message passing analysis \cite{donoho2009message, donoho2011noise, bayati2011dynamics, bayati2011lasso, zheng2017does, weng2018overcoming}, convex Gaussian min-max theorem \cite{stojnic2009various, stojnic2013upper, thrampoulidis2015regularized, thrampoulidis2018precise, dhifallah2017phase}, and leave-one-out analysis \cite{lei2018asymptotics, wang2018approximate, wang2018approximatelearn, sur2017likelihood}. 

In this paper, we will use convex Gaussian min-max theorem (CGMT) to help with the derivation of concentration inequality. CGMT has been developed in \cite{thrampoulidis2018precise} to obtain asymptotic expression of MSE for a large class of regularized estimators. However, due to the non-separability of the SLOPE penalty, it requires potentially strong assumption on the weight sequence $\{\lambda_i\}_{i=1}^p$ to derive the asymptotic expression for its MSE. Our concentration inequality provides a more quantitative way to evaluate MSE, requires weaker assumptions on $\{\lambda_i\}_{i=1}^p$ and covers the limiting result as a simple corollary. We should also mention that noise sensitivity analysis has been performed for some other regularized estimators such as bridge regression \cite{donoho2011noise, wang2017bridge, zheng2017does, weng2018overcoming, weng2019lownoise}. Given the fact that the regularization term in SLOPE is non-separable, the analysis for SLOPE is much more subtle. 

While we were preparing our paper, we became aware of three recent works \cite{hu2019asymptotics, celentano2019approximate,bu2019algorithmic} that are relevant to the study of SLOPE. However, there are substantial differences between the contributions of these papers and ours. The work by Hu and Lu studied SLOPE under a similar high-dimensional regime. Nevertheless, \cite{hu2019asymptotics} assumed more restrictive assumptions on $\{\lambda_i\}_{i=1}^p$ and derived the asymptotic limit of MSE, while we obtain the finite-sample concentration inequality for MSE. More importantly, the main focus of \cite{hu2019asymptotics} is on a practical algorithm that aims to search for the optimal SLOPE estimator. In contrast, our work provides an analytical comparison for different SLOPE estimators, and reveals the optimal SLOPE estimator under different noise levels. \cite{celentano2019approximate} derived a finite-sample concentration inequality for symmetrically penalized least squares including SLOPE. The concentration is measured under the Wasserstein distance for the empirical joint distribution of the estimator and the truth. There is no definite conclusion whether the concentration result in \cite{celentano2019approximate} is stronger or weaker than ours, because the constants appearing in these concentration inequalities exhibit different dependence on the model parameters and are not directly comparable. Moreover, the key issue addressed in \cite{celentano2019approximate} is the role of non-separability of the penalty for adaptive estimation, while we provide an answer to the noise sensitivity performance of different SLOPE estimators. \cite{bu2019algorithmic} developed an asymptotically exact characterization of the SLOPE estimator via the framework of approximate message passing (AMP). The performance is measured under a pseudo-Lipschitz loss function between the estimator and the truth, including MSE as a special example. However, the main result of \cite{bu2019algorithmic} is the derivation and characterization of an iterative AMP algorithm that provably (asymptotically) converges to the SLOPE solution. On the contrary, our work first provides a finite-sample concentration for the MSE of SLOPE, and proceeds with delicate noise sensitivity analysis.

\paragraph{\emph{Notations.}} Throughout the paper, we use bold and regular letters for vectors and scalars, respectively. For a given vector $\vv=(v_1,\ldots,v_p)\in \mathbb{R}^p$, $|\vv|_{(i)}$ denotes the $i$\textsuperscript{th} largest value of $\{|v_i|\}_{i=1}^p$ and $\vv_{[i:j]}$ denotes the subvector of $\vv$ with components $(v_{i}, \ldots, v_j)$. Further $\|\vv\|_0=\sum_i|v_i|^0$ with convention $0^0=0$, $\|\vv\|_2=\sqrt{\sum_{i}v_i^2}$, and $\|\vv\|_{\infty}=\max_{i}|v_i|$. We use $\|\vv\|_{\lambdav}$ to denote the sorted $\ell_1$ norm $\sum_{i} \lambda_i|\vv|_{(i)}$. When $\vv$ is random, $\|\vv\|_{\mathcal{L}_2} := \frac{1}{\sqrt{p}} \big[\mathbb{E}\|\vv\|_2^2\big]^{\frac{1}{2}}$ denotes the averaged expected $\ell_2$ norm. For a matrix $\Av\in \mathbb{R}^{p\times p}$, $\|\Av\|_2$ is its spectral norm. With two vectors $\vv_1$, $\vv_2$, we use $\langle \vv_1, \vv_2 \rangle$ and $\vv_1^\top \vv_2$ exchangeably for their inner product. For a function $f$, $\|f\|_{\mathrm{Lip}}$ denotes its Lipschitz norm. $\Phi(\cdot)$ and $\phi(\cdot)$ are the cdf and pdf of a standard normal respectively, and $\Phi^{-1}(\cdot)$ is the inverse function of $\Phi(\cdot)$. We denote $a_n=\Omega(b_n)$ when $b_n=O(a_n)$, and $a_n=\Theta(b_n)$ if and only if $a_n=\Omega(b_n)$ and $a_n=O(b_n)$. We use $\lesssim$ and $\gtrsim$ to denote less than and greater than up to an absolute constant. For $a, b\in \mathbb{R}, a\vee b=\max(a,b), a\wedge b=\min(a,b), \sign(a)$ equals $\frac{a}{|a|}$ if $a\neq 0$ and equals $0$ otherwise. For a positive integer $k, [k]=\{1,2,\ldots, k\}$. $\mathbb{R}_+=\{x\in \mathbb{R}: x\geq 0\}, \mathbb{R}_{++}=\{x\in \mathbb{R}: x>0\}$.\\

The remainder of the paper is organized as follows. Section \ref{sec:our} discusses in details our main contributions. Section \ref{sec:experiments} presents some numerical studies to validate our theoretical results and explore possible generalizations. We conclude the paper with a discussion in Section \ref{sec:discussions}, and relegate all the proofs to Section \ref{sec:proof}.

\section{Main Results} \label{sec:our}
In this section, we present our main results. We will show a concentration inequality for the MSE of SLOPE estimator in Section \ref{ssec:main:concentration}. In Section \ref{ssec:noiseless-compare}, we perform the noise sensitivity analysis of SLOPE, and provide a detailed comparison with the standard bridge estimators. Before delving into the details, we first clarify the setup of our study. In this work, we consider the following linear model:
\begin{equation} \label{eq:model-setup}
    \yv = \Av \xv + \zv,
\end{equation}
where $\Av \in \mathbb{R}^{n \times p}$ is the design matrix, $\yv \in \mathbb{R}^n$ is the response vector, $\xv\in \mathbb{R}^p$ is the unknown $k$-sparse coefficient vector that we want to estimate, and $\zv\in \mathbb{R}^n$ denotes the noise. We study the family of SLOPE estimators given by
\begin{align}
\label{new:slope:def}
 \hat{\xv}(\gamma) \in \argmin_{\xv} \frac{1}{2}\|\yv-\Av \xv\|_2^2+\gamma \|\xv\|_{\lambdav},
\end{align}
where $\gamma>0$ is a regularization parameter. Note that for notational simplicity, we have suppressed $\lambdav$ in $\hat{\xv}(\gamma)$. Given a weight vector $\lambdav\in \mathbb{R}^p$, \eqref{new:slope:def} defines a SLOPE estimator. We observe that setting $\lambda_1=\cdots=\lambda_p=1$ in \eqref{new:slope:def} yields the LASSO estimator. In our analysis of the SLOPE estimators, we make the following assumptions. Once we mention all the assumptions, we will provide a detailed discussion of why each assumption has been made. 

\begin{assumption}[Linear scaling] \label{assum:asymp} $\|\xv\|_0 = k > 0$. Furthermore, there exist $\kappa_1, \kappa_2,\kappa_3 >0$, such that
\begin{eqnarray}
\kappa_1 &\leq& \frac{n}{p} = \delta < \kappa_2, \nonumber \\
\kappa_3 &\leq& \frac{k}{p} = \epsilon < \delta. \nonumber
\end{eqnarray}
\end{assumption}

\begin{assumption}[IID Gaussian design] \label{assum:design}
    $A_{ij} \overset{i.i.d.}{\sim} \mathcal{N}(0, \frac{1}{n})$.
\end{assumption}

\begin{assumption}[Noise distribution] \label{assum:noise}
    $z_i$'s are i.i.d. sub-Gaussian with $\mathbb{E}z_i=0$, $\var(z_i) = \sigma_z^2$. Furthermore, there exists $\kappa_4>0$, such that $\|z_i\|_{\psi_2}\leq \sigma_z \kappa_4$.\footnote{The sub-Gaussian norm of a random variable $Z$ is defined as \[\|Z\|_{\psi_2} = \inf\{s>0 \ : \ \mathbb{E} ({\rm e}^{Z^2/s^2}-1) \leq 1\}. 
    \]}
\end{assumption}

\begin{assumption}[Bounded signal] \label{assum:bounded-signal}
    There exist $\kappa_5, \kappa_6 >0$ such that $\kappa_5 \leq \frac{\|\xv\|_2}{\sqrt{p}} \leq \kappa_6$.
\end{assumption}

\begin{assumption}[Reasonable weights] \label{assum:weight}
    $\lambda_1 \leq 1$, $\frac{\|\lambdav\|_2^2}{p} \geq \kappa_7$, for some $\kappa_7>0$.
\end{assumption}

Before we proceed to our main results, let us discuss these assumptions.  

Assumption \ref{assum:asymp} specifies the high-dimensional regime that our analysis will focus on. As we discussed in the last section, the usefulness of this regime has led many researchers to adopt this framework \cite{donoho2009message,  bayati2011lasso, weng2018overcoming, stojnic2009various, stojnic2013upper, thrampoulidis2015regularized, thrampoulidis2018precise, dhifallah2017phase,lei2018asymptotics, wang2018approximate, wang2018approximatelearn, sur2017likelihood}. The condition $\epsilon<\delta$ is very mild, as it merely requires the sample size to be larger than the number of non-zero elements of the signal. This is the information-theoretic limit for the exact recovery of a sparse signal from noiseless undersampled linear measurements \cite{wu2010renyi}.  

Assumption \ref{assum:design} is also a standard assumption that has been made in the linear asymptotic studies we cited above. While this assumption is admittedly restrictive, it has allowed a careful analysis of many estimators/algorithms and provided an accurate prediction of phenomena that are observed in high-dimensional settings, such as phase transitions. Furthermore, extensive simulation results reported elsewhere (see e.g. \cite{wang2017bridge,mousavi2018consistent}) have confirmed that the conclusions drawn for iid matrices hold for much broader classes of matrices. We will also report simulation results in Section \ref{sec:experiments} that show our main conclusions regarding SLOPE continue to hold for dependent and non-Gaussian designs. 

Assumption \ref{assum:noise} is another standard assumption in high-dimensional asymptotics. This assumption can possibly be weakened at the expense of obtaining slower concentration. However, to keep the discussion as simple as possible we consider sub-Gaussian noises.

In Assumption \ref{assum:bounded-signal}, the normalized $\ell_2$ norm square of the signal $\xv$ is assumed to be of order one. This together with Assumptions \ref{assum:asymp}, \ref{assum:design}, and \ref{assum:noise}  guarantees that the signal-to-noise ratio in each observation remains bounded. To clarify why one would like to keep the signal-to-noise ratio of order one, let us consider the well-studied LASSO problem. If the signal-to-noise ratio in each observation goes to $\infty$, then as $n,p \rightarrow \infty$ it is known that the estimation error of LASSO converges to zero above its phase transition\footnote{``Above (below) phase transition" refers to the success (failure) regime for exact recovery. A more specific explanation will be given in Section \ref{sec:experiments}.} (hence the problem is very similar to the noiseless setting). Furthermore, if the signal-to-noise ratio goes to zero, then no estimator can provide an accurate estimation of the signal under the linear asymptotic regime \cite{wu2010renyi}. Hence, this assumption ensures that the signal-to-noise ratio is bounded and the estimation problem does not have a trivial estimation error. 

%  Assumption \ref{assum:bounded-signal} also imposes the constraint that the elements of $\xv$ are uniformly bounded. It might be possible to remove this condition and work out the explicit dependence of the results on $\|\xv\|_{\infty}$. However, for notational simplicity, we stick to the boundedness assumption. 

Assumption \ref{assum:weight} imposes some constraints on the weights so that neither the loss function nor the penalty term in \eqref{new:slope:def} dominate. The upper bound in Assumption \ref{assum:weight} can be assumed without loss of generality due to the existence of the tuning parameter $\gamma$. 

Under these assumptions we study the performance of SLOPE in the next section. 

\subsection{A Concentration Inequality for SLOPE estimator}
\label{ssec:main:concentration}

Define the proximal operator of the SLOPE norm $\|\cdot\|_{\lambdav}$ as
\begin{equation*}
    \eta(\uv; \gamma, \lambdav)
    =
    \argmin_{\xv}\frac{1}{2} \|\uv-\xv \|_2^2+\gamma \|\xv\|_{\lambdav}.
\end{equation*}
When the weight sequence $\lambdav$ is clear from the context, we suppress the dependency of $\eta$ on $\lambdav$ and use the notation $\eta(\uv; \gamma)$ for the proximal operator. Below, we present a concentration inequality for the SLOPE estimators.
\begin{theorem} \label{thm:concentration0}
    Assume $\sigma_z, \gamma > 0$. Let $\hv \sim \mathcal{N}(0, \Iv_{p})$. There exist positive constants $\{C_i\}_{i=1}^5$ only possibly depending on the $\kappa_i$'s in Assumptions \ref{assum:asymp}-\ref{assum:weight} such that the following holds 
    \begin{equation} \label{eq:concentration-main}
        \mathbb{P}\bigg(\Big|\frac{1}{\sqrt{p}} \|\hat{\xv}(\gamma) - \xv \|_2
        - \alpha^* \Big| > t \bigg)
        \leq
        C_1\exp\Big(-C_2p\min\Big\{t^4A_1(\gamma, \sigma_z),t^2A_2(\gamma,\sigma_z)\Big\}\Big)
    \end{equation}
    for $0\leq t \leq \alpha^* A_3(\gamma,\sigma_z)$. Here, $\{A_i\}_{i=1}^3$ are functions that will be specified below under different scenarios. The quantity $\alpha^*$ takes the form $\alpha^*=\sqrt{\frac{1}{p}\mathbb{E}\|\eta(\xv + \sigma^*\hv; \sigma^*\chi^*) - \xv\|_2^2}$ where the unknown parameters $(\sigma^*, \chi^*)$ can be solved from the following two equations:
    \begin{align}
        (\sigma^*)^2 =& \sigma_z^2 + \frac{1}{\delta p}\mathbb{E}\|\eta(\xv + \sigma^*\hv; \sigma^*\chi^*) - \xv\|^2, \label{eq:state-evolution} \\
        \gamma =& \sigma^*\chi^* \Big(1 - \frac{1}{\delta \sigma^* p} \mathbb{E}\langle \eta(\xv + \sigma^*\hv; \sigma^*\chi^*), \hv \rangle \Big).
        \label{eq:calibration}
    \end{align}
    These two equations will be referred to as state evolution throughout the paper.\footnote{State evolution is  a term that is used for these two equations in the message passing literature \cite{donoho2009message}.} Below we consider three different scenarios and explain how $A_1, A_2,A_3$ are set in each case. The importance of these three cases in our paper will be clarified right after the theorem. Define the quantity
        \begin{align}
        \label{the:threshold}
        M_{\lambdav}(\chi^*) := \lim_{\sigma \rightarrow 0} \frac{1}{p} \mathbb{E}\|\eta(\xv/\sigma + \hv; \chi^*) -\xv/\sigma\|_2^2.     
        \end{align}
        Note that $M_{\lambdav}(\chi^*)$ depends on $\gamma,\sigma_z$ through $\chi^*$.
    \begin{enumerate}[(i)]
        \item \label{item:main-low-noise-abovePT0} Suppose $M_{\lambdav}(\chi^*) < \delta$ and $\xv$ has no tied nonzero components. Denote $C_M = \sqrt{\frac{\delta}{\delta - M_{\lambdav}(\chi^*)}}, C_\epsilon = \sqrt{\frac{\delta - \epsilon}{\epsilon}}$. Then we have
        \begin{align*}
            &A_1(\gamma, \sigma_z)=\big(C_M^6\sigma_z^2+C_M^8\sigma_z^4+C_M^8\sigma_z^4C_{\epsilon}^2(1+C_{\epsilon})^2\log p\big)^{-1} ,~~~A_2(\gamma, \sigma_z)=\sigma_z^{-2}C_M^{-4}, \\
            &A_3(\gamma, \sigma_z)=1 \wedge \sqrt{\frac{C_3 \gamma^2\sigma_z^{-2}}{C_{\epsilon}^2(1+C_4e^{C_5C_{\epsilon}^2})}}
        \end{align*}

        \item \label{item:main-low-noise-belowPT}
        Suppose $M_{\lambdav}(\chi^*) > \delta$ and $\xv$ has no tied nonzero components. Let $\sigma_0$ be the value that satisfies $\delta \sigma_0^2 = \frac{1}{p} \mathbb{E}\big\|\eta\big(\xv + \sigma_0\hv; \sigma_0 \chi^*\big) - \xv \big\|_2^2$, 
         and $b_0 = \frac{\partial}{\partial\sigma^2}\frac{1}{p} \mathbb{E}\big\|\eta\big(\xv + \sigma\hv; \sigma\chi^*\big) - \xv \big\|_2^2 \big|_{\sigma = \sigma_0}$. Denote $\bar{D}_M = \sqrt{\frac{\sigma_0^2}{\sigma_z^2}+\frac{\delta}{\delta-b_0}}, D_M = \sqrt{\frac{M_{\lambdav}(\chi^*) - \epsilon}{\epsilon}}$. We have 
            \begin{align*}
            &A_1(\gamma, \sigma_z)=\big(\bar{D}_M^6\sigma_z^2+\bar{D}_M^8\sigma_z^4+\bar{D}_M^8\sigma_z^4D_M^2(1+D_M)^2\log p\big)^{-1} ,~~~A_2(\gamma, \sigma_z)=\sigma_z^{-2}\bar{D}_M^{-4},\\
            &A_3(\gamma, \sigma_z)=1 \wedge \sqrt{\frac{C_3 \gamma^2\sigma_z^{-2}}{D_M^2(1+C_4e^{C_5D_M^2})}}.
        \end{align*}
        \item \label{item:main-large-noise-deltaB1}
         If $\frac{\gamma}{\sigma_z} > \frac{1}{\|\lambdav\|_2^2 / p} \sqrt{0 \vee \log \frac{16\delta + 8} {\delta^2}}$ and $\sigma_z > \frac{\sqrt{2(\delta + 1)} \|\xv\|_2}{\delta \sqrt{p}}$, then we have
        \begin{align*}
            & A_1(\gamma, \sigma_z)=\big(\sigma_z^2+\sigma_z^4+(\sigma_z^2\gamma^2+\gamma^4)\log p\big)^{-1},~~A_2(\gamma, \sigma_z)=\sigma_z^{-2}, ~~A_3(\gamma, \sigma_z)=1 \wedge \sqrt{\frac{C_3}{1+C_4e^{C_5\gamma \sigma_z^{-1}}}}
        \end{align*}
    \end{enumerate}
\end{theorem}
The proof of Theorem \ref{thm:concentration0} is presented in Section \ref{ssec:proof-concen}.

\begin{figure}[!t]
    \begin{center}
        \setlength\tabcolsep{2pt}
        \renewcommand{\arraystretch}{0.3}
        \begin{tabular}{rcccc}
            &
            \footnotesize{$\delta = 2$, $\sigma_z=1$}
            &
            \footnotesize{$\delta = 2$, $\sigma_z=0$}
            &
            \footnotesize{$\delta = 0.8$, $\sigma_z=1$}
            &
            \footnotesize{$\delta = 0.8$, $\sigma_z=0$} \\
            \rotatebox{90}{\qquad\qquad\quad\scriptsize{$\mathrm{MSE}$}}
            &
            \includegraphics[scale=0.3]{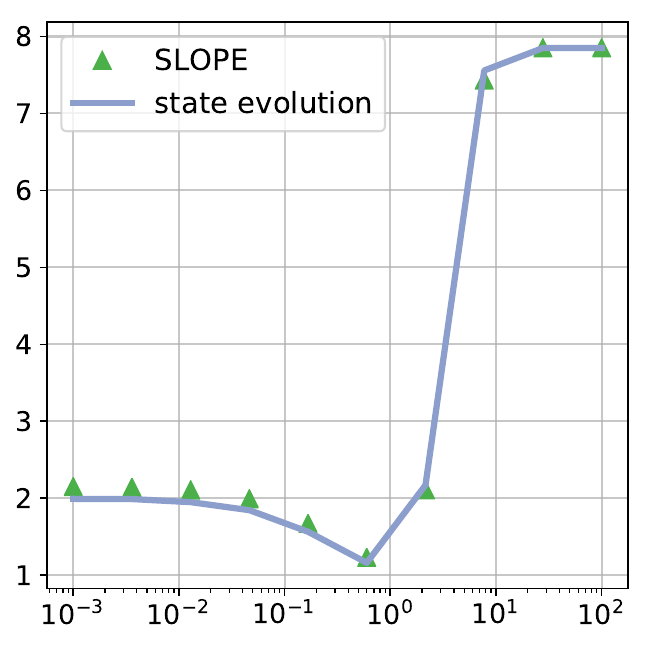}
            &
            \includegraphics[scale=0.3]{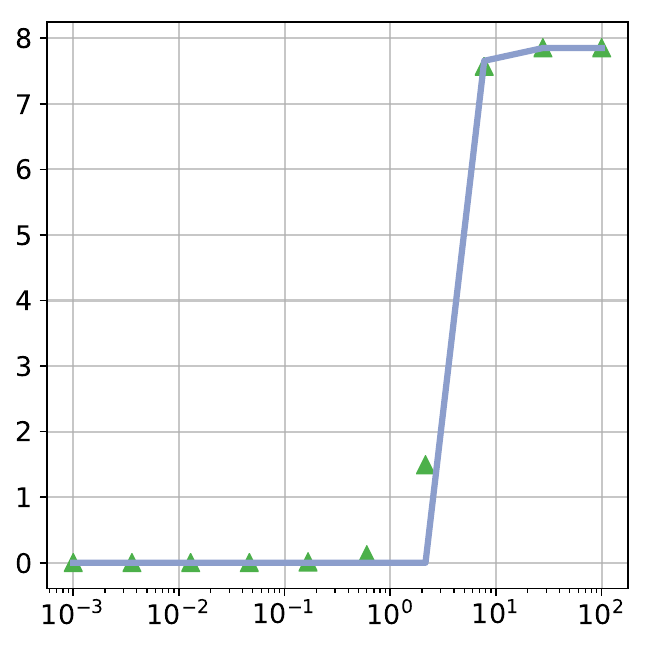}
            &
            \includegraphics[scale=0.3]{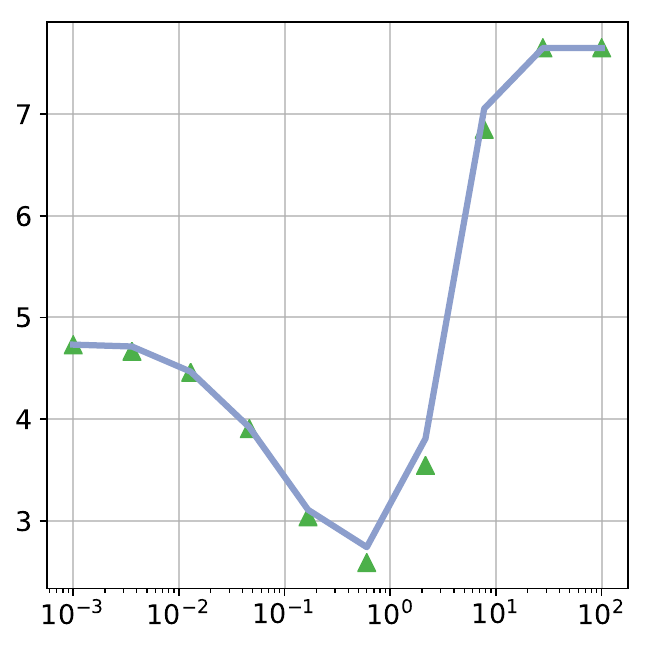}
            &
            \includegraphics[scale=0.3]{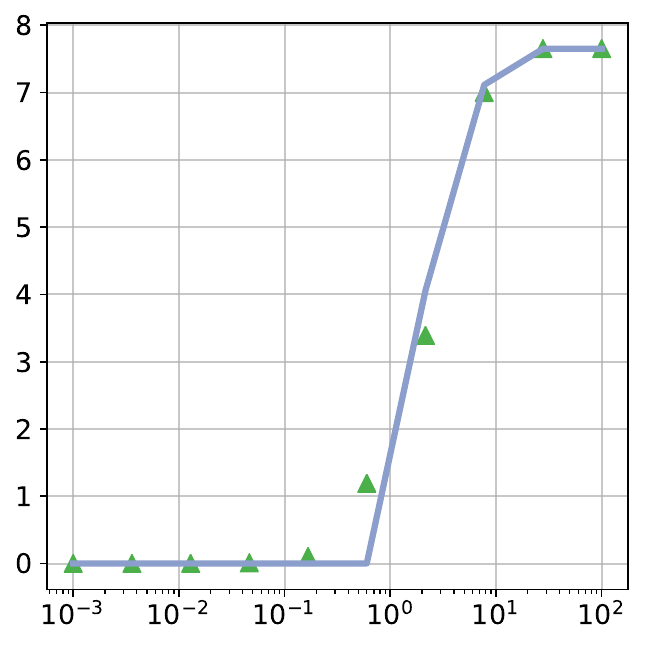} \\
            &
            \footnotesize{$\gamma$}
            &
            \footnotesize{$\gamma$}
            &
            \footnotesize{$\gamma$}
            &
            \footnotesize{$\gamma$}
        \end{tabular}
        \caption{Comparing the MSE of SLOPE estimator and the quantity
            $\delta((\sigma^*)^2 - \sigma_z^2)$ from state evolution
            equations \eqref{eq:state-evolution} and \eqref{eq:calibration}. The sequence of weights $\{\lambda_i\}$ are equally
            spaced within $[0.01, 1]$. We set $p=1000$, the components of $\xv$ are i.i.d. samples from
            $5 \ast \mathrm{Bernoulli}(\mathrm{prob}=0.3)$, and the components of
            $\zv$ are i.i.d. samples from $\mathcal{N}(0, \sigma_z^2)$.
        } \label{fig:sample-expected-mse}
    \end{center}
\end{figure}

We make several important remarks below to interpret and discuss the results of Theorem \ref{thm:concentration0}. 

\begin{remark}
Theorem \ref{thm:concentration0} shows that the MSE of a given SLOPE estimator concentrates tightly around $\frac{1}{p}\mathbb{E}\|\eta(\xv + \sigma^*\hv; \sigma^*\chi^*) - \xv\|^2$ which equals to $\delta((\sigma^*)^2 - \sigma_z^2)$ from \eqref{eq:state-evolution}. Given all the model and SLOPE parameters, we can compute the preceding quantity from the state evolution equations \eqref{eq:state-evolution} and \eqref{eq:calibration}. Such a quantity is expected to be an accurate prediction for MSE. This is empirically verified in Figure \ref{fig:sample-expected-mse}.  
\end{remark}

\begin{remark}
The condition on $t$ is technical and might be weakened by a more sophisticated analysis. However, since $A_3(\gamma,\sigma_z)\leq 1$, the concentration inequality holds for $t$ that is smaller than $\alpha^*$, which is the most interesting regime given that $\alpha^*$ is the location where the concentration is around. The rate $t^4$ in the exponent of \eqref{eq:concentration-main} emerges from our analysis of the objective function in \eqref{new:slope:def} based on convex Gaussian min-max theorem (CGMT). We conjecture that the sharp dependency on $t$ is $t^2$ instead of $t^4$, although proving it seems challenging under the CGMT framework. We leave a thorough analysis of the optimal rate on $t$ for a future research.
\end{remark}

\begin{remark}
Given that the SLOPE estimation problem involves several important parameters, such as the noise level $\sigma_z$ and the tuning parameter $\gamma$, we should expect these quantities to play a role in the concentration of the mean square error. Hence, obtaining a single concentration inequality that exhibits the accurate dependence on all the parameters seems to be remarkably challenging. As described in Theorem \ref{thm:concentration0}, to overcome this difficulty, we have chosen to present concentration results under three different scenarios. We now discuss the result of each scenario below.
\begin{enumerate}[(1)]
\item \label{item:discuss-above-PT} Scenarios (\ref{item:main-low-noise-abovePT0}) and (\ref{item:main-low-noise-belowPT}) are concerned with the concentration in the low noise regime. Scenario (\ref{item:main-low-noise-abovePT0}) considers the case in which the sample size (per dimension), $\delta$, is above the threshold $M_{\lambdav}(\chi^*)$. Note that in this case, it is clear that the probability bound becomes smaller as the noise level decreases, which captures qualitatively correct effect of $\sigma_z$. Moreover, if we choose $\gamma=\Theta(\sigma_z)$, $A_3(\gamma,\sigma_z)$ is of order one. As will be seen in the proof of Theorem \ref{thm:noiseless-phase-transition}, the condition $\gamma=\Theta(\sigma_z)$ holds for the optimal tuning of the parameter $\gamma$. The assumption that $\xv$ has no tied nonzero components is crucial for the comparison of different SLOPE estimators. We will discuss this assumption in more details in Section \ref{ssec:low-noise-compare}. Finally, note that the condition $\delta>M_{\lambdav}(\chi^*)$ ensures that SLOPE is ``performing above its phase transition", i.e., as the noise level $\sigma_z \rightarrow 0$, the MSE $\frac{1}{p}\mathbb{E}\|\eta(\xv + \sigma^*\hv; \sigma^*\chi^*) - \xv\|^2$ goes to zero as well. For studying the important features of the phase transition, the reader may refer to \cite{weng2018overcoming}.

\item \label{item:discuss-below-PT} Scenario (\ref{item:main-low-noise-belowPT}) characterizes the behavior when $\delta$ is below the threshold $M_{\lambdav}(\chi^*)$. In this regime, the mean square error of SLOPE does not vanish (in fact converges to $\delta \sigma_0^2$ from \eqref{eq:state-evolution}) when the noise level $\sigma_z\rightarrow 0$, so that SLOPE is ``performing below its phase transition". As a result, the probability bound we derived in this scenario becomes degenerate as $\sigma_z$ approaches zero, hence does not reveal the accurate expression of the noise level in the concentration inequality. Nevertheless, the concentration inequality is still valid in terms of the dimension or sample size, holding all the other parameters fixed. Moreover, as will be clear in Section \ref{ssec:noiseless-compare}, this scenario is not of particular interest for our low noise sensitivity analysis.

\item \label{item:discuss-large-noise} Scenario (\ref{item:main-large-noise-deltaB1}) shows the concentration result in the large noise regime. The requirement on the tuning $\gamma\geq \frac{1}{\|\lambdav\|_2^2 / p} \sqrt{0 \vee \log \frac{16\delta + 8} {\delta^2}}\cdot \sigma_z$ is reasonable in this setting, because it is desirable to set a large value of the tuning to reduce the variance of the SLOPE estimate, when the noise level is high. In particular, as we will discuss in Section \ref{ssec:noiseless-compare}, the condition is satisfied by the optimal tuning. Note that as the system has larger noise ($\sigma_z$ increases), the concentration is expected to become worse. Our probability bound is consistent with such intuition. 
\end{enumerate}
\end{remark}

\begin{remark}
In the proof of Theorem \ref{thm:concentration0}, we have derived a more general concentration theorem (c.f. Theorem \ref{master:thm}) including the three scenarios from Theorem \ref{thm:concentration0} as special cases. Nevertheless, the probability bound in the general concentration result depends on additional parameters $(\sigma^*,\chi^*)$, thus does not reveal an explicit dependency on the noise level $\sigma_z$. Since the paper is focused on the noise sensitivity analysis, the concentration results in Theorem \ref{thm:concentration0} are more interpretable and relevant.
\end{remark}

\begin{remark}
The non-separability of the sorted $\ell_1$ norm in SLOPE and the complicated form of the equations \eqref{eq:state-evolution} \eqref{eq:calibration} bring substantial difficulty to derive the concentration inequality. Hence we do not claim our results to be the optimal ones. For example, there might exist a sharper result for LASSO due to its amenable structure.  
\end{remark}

\begin{remark}
\cite{hu2019asymptotics} has showed that as $n\rightarrow \infty$, the MSE of a given SLOPE estimator converges to the limit of $\frac{1}{p}\mathbb{E}\|\eta(\xv + \sigma^*\hv; \sigma^*\chi^*) - \xv\|^2$ for specialized weight sequence $\{\lambda_i\}$. Using Borel-Cantelli lemma, such asymptotic result is directly obtained from the concentration inequality \eqref{eq:concentration-main}. Moreover, setting $\lambda_1=\cdots=\lambda_p=1$ recovers the asymptotic result of LASSO \cite{donoho2011noise, bayati2011lasso}.
\end{remark}

\subsection{Noise sensitivity analysis of SLOPE}
\label{ssec:noiseless-compare}
The concentration inequality in Theorem \ref{thm:concentration0} accurately characterizes the behavior of SLOPE estimator under different noise levels. In this section, we aim to employ this result and obtain a fair comparison among different SLOPE estimators. Toward this goal, define 
\begin{align} \label{amse:like}
    e_{\lambdav}(\gamma, \sigma_z)=\frac{1}{p}\mathbb{E}\|\eta(\xv+\sigma^* \hv;\sigma^* \chi^*)-\xv\|_2^2, 
\end{align}
where $(\sigma^*, \chi^*)$ is the solution to the state evolution equations \eqref{eq:state-evolution} and \eqref{eq:calibration}. According to Theorem \ref{thm:concentration0} and as empirically verified in Figure \ref{fig:sample-expected-mse}, the squared error $\frac{1}{p}\|\hat{\xv}(\gamma)-\xv\|_2^2$ of the SLOPE estimator $\hat{\xv}(\gamma)$ concentrates tightly around $e_{\lambdav}(\gamma, \sigma_z)$. Hence, we use $e_{\lambdav}(\gamma, \sigma_z)$ to evaluate the quality of the estimate $\hat{\xv}(\gamma)$. 

As is clear from the expressions in \eqref{eq:state-evolution}, \eqref{eq:calibration}, and \eqref{amse:like}, the value of $e_{\lambdav}(\gamma, \sigma_z)$ depends on the signal $\xv$, the noise level $\sigma_z$, the regularization parameter $\gamma$, and the sample size $\delta$ (per dimension) in an implicit, nonlinear and complicated way. Hence, in order to gain useful information about the performance of $\hat{\xv}(\gamma)$, we will focus our study on the impact of the noise level $\sigma_z$ on $e_{\lambdav}(\gamma, \sigma_z)$. In particular, we analyze $e_{\lambdav}(\gamma, \sigma_z)$ under the low noise and large noise scenarios in Sections \ref{ssec:low-noise-compare} and \ref{ssec:large-noise-compare}, respectively. Our delicate noise sensitivity analysis will turn $e_{\lambdav}(\gamma, \sigma_z)$ into explicit and informative quantities that provide interesting insights into the behavior of the family of SLOPE estimators. Towards that goal, we consider the value of $\gamma$ that minimizes $e_{\lambdav}(\gamma, \sigma_z)$,
\begin{align*}
    \gamma_{\lambdav}^*=\argmin_{\gamma>0}~ e_{\lambdav}(\gamma, \sigma_z).
\end{align*}
Thus, $e_{\lambdav}(\gamma_{\lambdav}^*,\sigma_z)$ characterizes the performance of $\hat{\xv}(\gamma_{\lambdav}^*)$, i.e., the SLOPE estimator under the optimal tuning $\gamma=\gamma_{\lambdav}^*$ that minimizes the mean square error (or equivalently prediction error). This is the best MSE that each SLOPE estimator can possibly achieve. Our subsequent analyses and results are tailored to estimators with the regularization parameter $\gamma$ being optimally tuned.
\subsubsection{Low noise sensitivity analysis of SLOPE}
\label{ssec:low-noise-compare}

In this section, we aim to perform a noise sensitivity analysis of SLOPE. In this analysis, we consider the noise level $\sigma_z$ to be very small, and calculate the asymptotic MSE $e_{\lambdav}(\gamma_{\lambdav}^*,\sigma_z)$. The following theorem summarizes our main result regarding the low noise sensitivity analysis of SLOPE:

\begin{theorem} \label{thm:noiseless-phase-transition}
Let $k = \|\xv\|_0$ and suppose $\xv\in \mathbb{R}^p$ does not have tied non-zero elements. Define
\begin{equation} \label{eq:M-lambda}
    M_{\lambdav} = \frac{1}{p}\inf_{\alpha>0}\bigg\{k+\alpha^2 \sum_{i=1}^k\lambda_i^2+\mathbb{E}\|\eta(\tilde{\hv};\alpha, \lambdav_{[k+1:p]})\|_2^2\bigg\},
\end{equation}
where $\tilde{\hv}\in \mathbb{R}^{p-k}\sim \mathcal{N}(0,\Iv_{p-k})$. Then, we have 
\begin{enumerate}[(a)]
\item \label{item:optimal-mse-above-PT}
\begin{align*}
\lim_{\sigma_z\rightarrow 0} e_{\lambdav}(\gamma_{\lambdav}^*,\sigma_z) =
\begin{cases}
>0, & \mbox{~~if~} \delta < M_{\lambdav},\\
=0, & \mbox{~~if~}\delta > M_{\lambdav}.
\end{cases}
\end{align*}
\item \label{item:optimal-mse-below-PT} Furthermore, 
\begin{align*}
\lim_{\sigma_z\rightarrow 0} \frac{e_{\lambdav}(\gamma_{\lambdav}^*,\sigma_z)}{\sigma_z^2}=
\begin{cases}
\infty, & \mbox{~~if~} \delta < M_{\lambdav},\\
\frac{\delta M_{\lambdav}}{\delta-M_{\lambdav}}, & \mbox{~~if~}\delta > M_{\lambdav}.
\end{cases}
\end{align*}

\end{enumerate}
\end{theorem}

The proof of this theorem can be found in Section \ref{sssec:proof-low-noise}. Several remarks are in order.

\begin{remark}
The low noise sensitivity analysis is aligned with the concentration results of Scenarios (\ref{item:main-low-noise-abovePT0}) and (\ref{item:main-low-noise-belowPT}) in Theorem \ref{thm:concentration0}. As will be shown in Lemma \ref{pp4}, $M_{\lambdav}$ defined in \eqref{eq:M-lambda} equals to $M_{\lambdav}(\chi^*)$ in \eqref{the:threshold} under optimal tuning $\gamma=\gamma^*_{\lambdav}$. Thus, the cases $\delta> M_{\lambdav}$ and $\delta<M_{\lambdav}$ correspond to Scenarios (\ref{item:main-low-noise-abovePT0}) and (\ref{item:main-low-noise-belowPT}), respectively. 
\end{remark}

\begin{remark}
    Part (\ref{item:optimal-mse-above-PT}) in Theorem \ref{thm:noiseless-phase-transition} characterizes the phase transition of SLOPE estimators. Specifically, as the noise vanishes, SLOPE can fully recover the $k$-sparse signal $\xv$ if and only if $\delta >M_{\lambdav}$. Thus, $M_{\lambdav}$ is the sharp threshold of SLOPE for exact recovery.
\end{remark}

\begin{remark}
    Part (\ref{item:optimal-mse-below-PT}) in Theorem \ref{thm:noiseless-phase-transition} further reveals the low noise sensitivity of SLOPE. Above phase transition where exact recovery is attainable, the error $e_{\lambdav}(\gamma_{\lambdav}^*,\sigma_z)$ of all the SLOPE estimators reduces to zero in the same rate of $\sigma_z^2$. Hence the constant $\frac{\delta M_{\lambdav}}{\delta-M_{\lambdav}}$ represents the noise sensitivity of each SLOPE estimator. The smaller $M_{\lambdav}$ is, the smaller the constant is. 
\end{remark}

The explicit formulas we derived in Theorem \ref{thm:noiseless-phase-transition} enable us to compare different SLOPE estimators with each other and also with more standard estimators such as bridge regression. According to this theorem, the key quantity that determines the performance of SLOPE is $M_{\lambdav}$. Hence, in order to find the best SLOPE estimator we should find a sequence $\lambdav$ that minimizes $M_{\lambdav}$. The following proposition addresses this issue.

\begin{proposition} \label{lem:best_slope_small_noise}
  $M_{\lambdav}$ as a function of $\lambdav$, is minimized when $\lambda_1=\cdots=\lambda_p$. 
\end{proposition}

The proof of this proposition can be found in Section \ref{sssec:proof-low-noise}.
 
According to this proposition, we can conclude that LASSO is optimal among all SLOPE estimators in the low noise scenario. Note that it has been proved that LASSO outperforms all the convex bridge estimators in the low-noise regime \citep{weng2018overcoming}, but not necessarily the non-convex bridge estimators \citep{zheng2017does}.  

\begin{remark}
We should emphasize that the requirement that the unknown signal $\xv$ does not have tied non-zero components is critical for both Theorem \ref{thm:noiseless-phase-transition} and Proposition \ref{lem:best_slope_small_noise}. Intuitively speaking, for signal $\xv$ with tied non-zero components, given the fact that setting unequal weights $\{\lambda_i\}$ can produce estimators having tied non-zero elements (cf. Lemma \ref{property:primal} Part (\ref{lemma:item:prox-form})), a SLOPE estimator (with appropriately chosen weights) makes better use of the signal structure than LASSO does. Hence, the optimality of LASSO will not hold for such signals. We provide some empirical results in Section \ref{sec:experiments} to support this claim. That being said, it is also important to point out that the assumption about signals without tied non-zero components is not necessarily required for characterizing the mean squared error of each SLOPE estimator. See, for example, the general concentration inequality (Theorem \ref{master:thm}) we have derived in Section \ref{master:for:all:case}. This assumption is made to enable a sharp comparison among all SLOPE estimators and reveal the optimality of LASSO. 
\end{remark}

\subsubsection{Large noise sensitivity analysis of SLOPE}
\label{ssec:large-noise-compare}

In the last section, we discussed the performance of the SLOPE estimators in the situations where the noise in the observations is small. Under such circumstances we showed that the LASSO is the best SLOPE estimator. In this section, we aim to study the SLOPE estimators in the low signal-to-noise ratio regimes. The following theorem summarizes our result in the low signal-to-noise ratio regime:

\begin{theorem} \label{thm:large-noise-ridge-better}
As $\sigma_z \rightarrow \infty$,
\begin{align} \label{large:noise:for:slope}
e_{\lambdav}(\gamma_{\lambdav}^*,\sigma_z)=\frac{1}{p}\|\xv\|_2^2+O(\exp(-c \sigma_z^2)),
\end{align}
where $c>0$ is a constant possibly depending on $\kappa_5,\kappa_6$, and $\kappa_7$ in Assumptions \ref{assum:bounded-signal} and \ref{assum:weight}.  
\end{theorem}

The proof can be found in Section \ref{proof:of:theorem:large:noise}. The large noise sensitivity analysis in this theorem is consistent with Scenario (\ref{item:main-large-noise-deltaB1}) in Theorem \ref{thm:concentration0}. As will be seen in the proof the optimal tuning $\gamma^*_{\lambdav}=\Omega(\sigma_z^2)$, thus satisfying the requirement of the tuning in Scenario (\ref{item:main-large-noise-deltaB1}). To provide a good benchmark to understand and interpret Theorem \ref{thm:large-noise-ridge-better}, let us mention the large noise sensitivity result for bridge regression from \cite{wang2017bridge}. Consider the bridge estimator
\begin{align*}
\hat{\xv}(\gamma) \in \argmin_{\xv} \frac{1}{2}\|\yv-\Av \xv\|_2^2+\gamma \cdot \sum_{i=1}^p|x_i|^q.
\end{align*}
Let $e_q(\gamma, \sigma_z)$ denote the (asymptotically) exact expression of $\frac{1}{p} \| \hat{\xv}(\gamma) - \xv\|_2^2$ and define 
\[
\gamma_q^*=\argmin_{\gamma>0}e_q(\gamma, \sigma_z). 
\]
Thus, $e_{q}(\gamma_{q}^*,\sigma_z)$ measures the performance of the bridge estimator under optimal tuning. It has been proved \citep{wang2017bridge} that 
\begin{align} \label{large:noise:for:bridge}
    e_{q}(\gamma_{q}^*,\sigma_z)=\frac{1}{p}\|\xv\|_2^2- \frac{c_q\|\xv\|_2^4p^{-2}}{\sigma_z^2}+o(\sigma_z^{-2}),
    \qquad \text{for} \quad
    q \in (1, \infty).
\end{align}
Here, the positive constant $c_q$ only depends on $q$. Combing the results \eqref{large:noise:for:slope} and \eqref{large:noise:for:bridge}, we reach the following conclusions:
\begin{enumerate}[1.]
\item The SLOPE and bridge estimators share the same first order term $\|\xv\|_2^2/p$. This is expected because as the noise level goes to infinity, the variance will dominate the estimation error and thus the optimal estimator will eventually converge to zero.
\item The second order term is exponentially small for all SLOPE estimators, while it is negative and polynomially small for all bridge estimators with $q\in (1,\infty)$. Hence, bridge estimators outperform all the SLOPE estimators in the large noise scenario. Moreover, \cite{wang2017bridge} showed that the constant $c_q$ in \eqref{large:noise:for:bridge} attains the maximum at $q=2$. Therefore, Ridge regression turns out to be the optimal bridge estimator in the large noise scenario. In Section \ref{sec:experiments}, we use the Ridge estimator as a representative bridge estimator for numerical studies.
\item Theorem \ref{thm:large-noise-ridge-better} does not answer which SLOPE estimator is optimal. However, together with the result \eqref{large:noise:for:bridge} it reveals that the family of SLOPE estimators generally do not perform well compared with bridge estimators. We may prefer using bridge regression such as Ridge to estimate the sparse vector $\xv$ in the large noise scenario. 
\end{enumerate}

\section{Numerical Experiments} \label{sec:experiments}
In this section, we present our numerical studies. We pursue the following goals in our simulations:
\begin{enumerate}
\item Check the accuracy of our conclusions for finite sample sizes.
\item Show that the main conclusions hold even if some of the assumptions that we made in our theoretical studies, such as the independence or Gaussianity of the elements of $\Av$, are violated. 

\item Show that if the non-zero elements of $\xv$ are equal, then LASSO might not be the optimal SLOPE estimator in the low noise regime. Hence, the assumption that $\xv$ has no tied non-zero components in Theorem \ref{thm:noiseless-phase-transition} and Proposition \ref{lem:best_slope_small_noise} is necessary in this sense. 

\end{enumerate}

We consider the following simulation setups:
\begin{itemize}
    \item Design: $\Av = \tilde{\Av}\Sigmav^{\frac{1}{2}}$ where the $\tilde{A}_{ij}$'s (up to a scaling) are iid $t$-distributed with degrees of freedom equal 3 to test the validity of our conclusions when the elements of $\Av$ have a heavy-tailed distribution, and iid Gaussian otherwise. The elements $\tilde{A}_{ij}$ are re-scaled by $\sqrt{n}\;\mathrm{std}(\tilde{A}_{ij})$. Furthermore, in our simulation results we will consider two choices of $\Sigma$: $\Sigma_{ij} = \rho^{i-j}$ with (i) $\rho=0.8$, and (ii) $\rho=0$. The first choice will test the validity of our conclusions for the case that the elements of $\Av$ are dependent. 
    
    \item Noise: $\zv \sim \mathcal{N}(0, \sigma_z^2 \Iv_n)$. The values for $\sigma_z$ will be specified in each simulation below.
    
    \item Signal: for a given value of $\epsilon$ and $p$, we randomly set $(1 - \epsilon) p$ components of $\xv$ as 0. For the rest non-zero components, two configurations are considered: (i) iid samples from $\mathrm{Unif}[0, 5]$; (ii) all equal to 5. We use the second case to show that when the non-zero coefficients are equal, then LASSO might not be optimal in the low noise scenario. 
    
    %Note that according to Theorem \ref{thm:noiseless-phase-transition}, in low noise scenarios, we expect LASSO to outperform other SLOPE estimator under setting (i). For setting (ii) where the signal has tied non-zero components, Theorem \ref{thm:noiseless-phase-transition} might not hold. we will explore the SLOPE estimator with specially picked weights $\lambdav$ that may outperform LASSO.
    
    \item $p=500$, $n=\delta p$. $\delta$ and $\epsilon$ will be determined later.
    
    \item Once each problem instance is set, we will run our simulations $m =20$ times, and we will report the average MSE and the standard error bars. 
    
    \item Recall that the comparison results in Section \ref{ssec:noiseless-compare} are valid for optimally-tuned estimators. In our simulations, we use $5$-fold cross-validation to find the optimal tuning parameters.  
\end{itemize}

Figure \ref{fig:sample-mse-compare-corr-tail} shows the MSE of SLOPE, LASSO and Ridge estimators under different types of design matrices. The estimator denoted by SLOPE:BH is the SLOPE estimator that was proposed in \cite{bogdan2015slope} and shown to be minimax optimal in \cite{su2016slope, bellec2018slope}. We first discuss the results for iid Gaussian designs in the first plot. We set the parameters $(\delta, \epsilon) = (0.9, 0.5)$ so that the setting is above phase transition for LASSO, and below phase transition for the two SLOPE estimators.\footnote{From Theorem \ref{thm:noiseless-phase-transition} we know that $\delta >M_{\lambdav}$ means the corresponding setting is above phase transition. For LASSO, the inequality can be simplified as $\delta \geq \inf_{\chi} 2(1 - \epsilon)((1 + \chi^2)\Phi(-\chi) - \chi \phi(\chi)) + \epsilon(1 + \chi^2)$, and analytically verified. For the two SLOPE estimators, since $M_{\lambdav}$ can not be directly evaluated, we conclude it is below phase transition based on the numerical results in the figure.} It is clear that LASSO outperforms the SLOPE estimators when the noise level is low, as  predicted by Theorem \ref{thm:noiseless-phase-transition} and Proposition \ref{lem:best_slope_small_noise}. Moreover, as the noise level increases above $\sigma_z=2$, Ridge starts to have a smaller MSE compared to LASSO and SLOPE. This is consistent with the result from Theorem \ref{thm:large-noise-ridge-better}. These phenomena are also observed in the other three plots where iid Gaussian assumptions are not satisfied on the design matrix. Such empirical results suggest that the main comparison conclusions drawn from Proposition \ref{lem:best_slope_small_noise} and Theorems \ref{thm:noiseless-phase-transition} and \ref{thm:large-noise-ridge-better} are valid for non-Gaussian and correlated designs too. We leave a precise analysis of such designs as an open avenue for a future research. For the performance of other bridge regression estimators, we refer to the extensive simulations in \cite{wang2017bridge}.

\begin{figure}[thb!]
    \begin{center}
        \setlength\tabcolsep{2pt}
        \renewcommand{\arraystretch}{0.3}
        \begin{tabular}{rcccc}
            &
            \footnotesize{iid}
            &
            \footnotesize{correlated}
            &
            \footnotesize{heavy tail}
            &
            \footnotesize{correlated + heavy tail}\\
            \rotatebox{90}{\qquad\qquad\quad\scriptsize{$\mathrm{MSE}$}}
            &
            \includegraphics[scale=0.285]{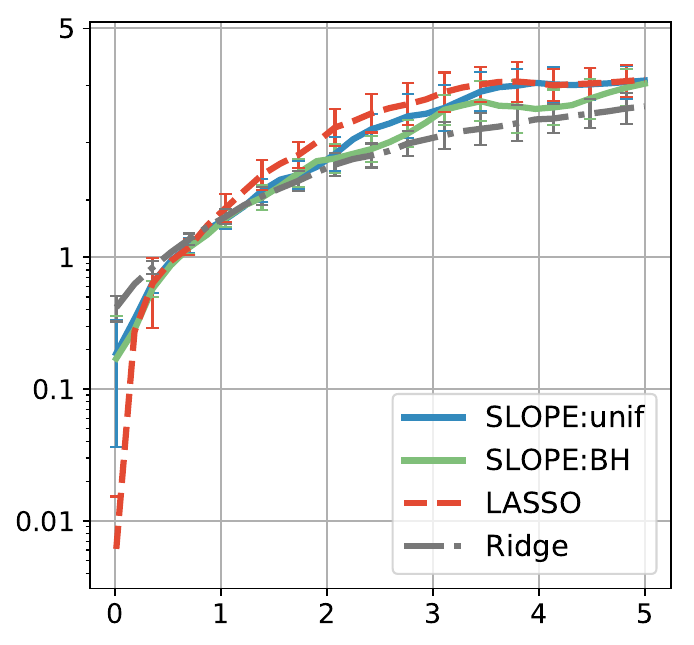}
            &
            \includegraphics[scale=0.285]{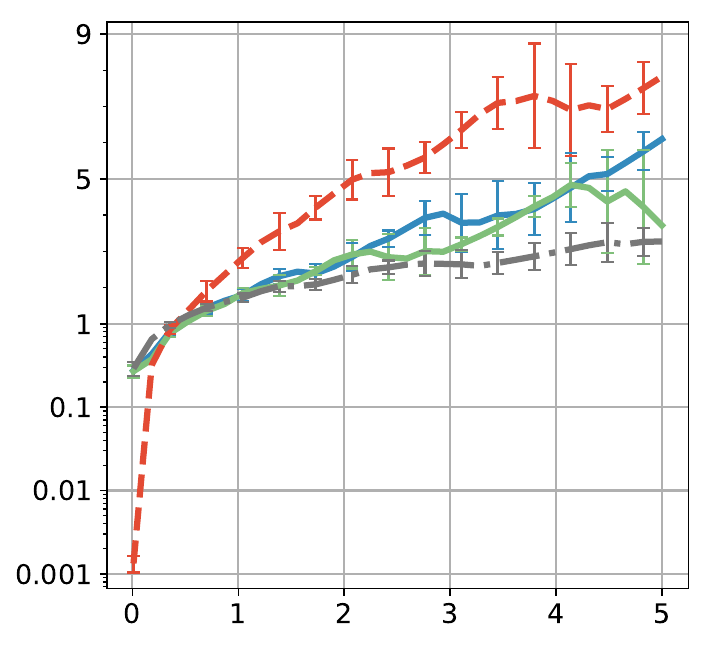}
            &
            \includegraphics[scale=0.285]{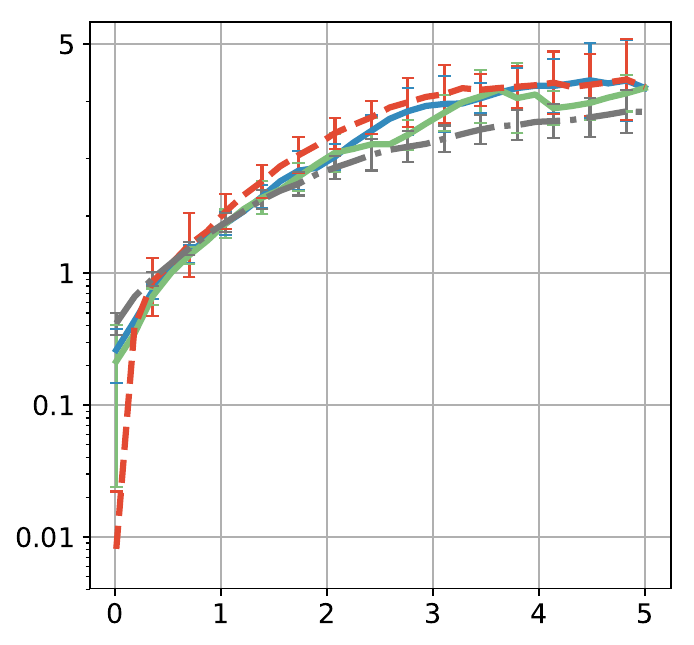}
            &
            \includegraphics[scale=0.285]{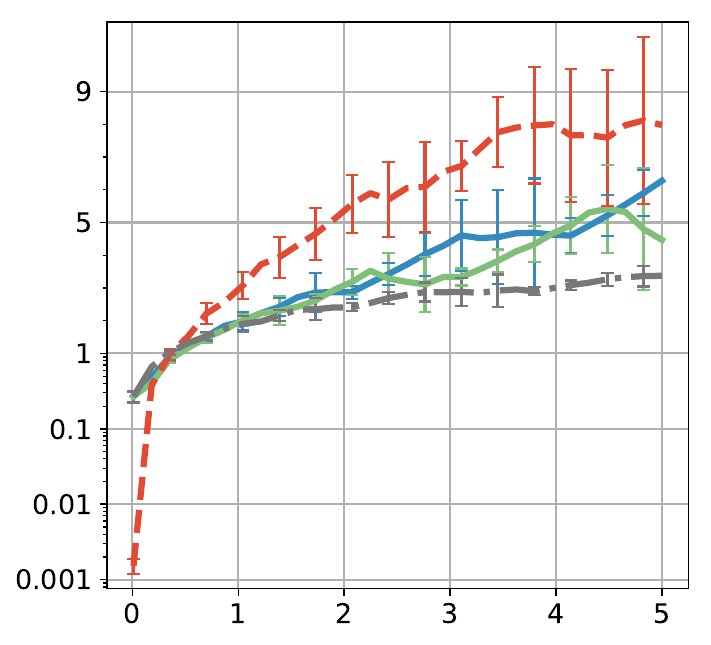} \\
            &
            \footnotesize{$\sigma_z$}
            &
            \footnotesize{$\sigma_z$}
            &
            \footnotesize{$\sigma_z$}
            &
            \footnotesize{$\sigma_z$}
        \end{tabular}
        \caption{MSE of SLOPE, LASSO and Ridge estimators. SLOPE:BH and SLOPE:unif denote the SLOPE estimators with weights $\lambda_i = \Phi^{-1}(1 - \frac{iq}{2p})/\Phi^{-1}(1 - \frac{q}{2p})$ with $q=0.5$ and $\lambda_i = 1 - 0.99(i-1)/p$, respectively. Other model parameters are $\delta=0.9$, $\epsilon=0.5$; The nonzero components of $\xv$ are iid samples from $\mathrm{Uniform}[0, 5]$; $\sigma_z \in [0, 5]$.
        } \label{fig:sample-mse-compare-corr-tail}
    \end{center}
\end{figure}

In Figure \ref{fig:sample-mse-compare-above-pt1}, we further compare the MSE of LASSO with that of SLOPE in two cases when the system is above phase transition for both SLOPE and LASSO. As is clear from the first column, for iid Gaussian designs, LASSO has a smaller MSE when $\sigma_z$ is small, which is accurately characterized in Theorem \ref{thm:noiseless-phase-transition} and Proposition \ref{lem:best_slope_small_noise}. Again, similar result seems to hold under more general settings, including correlated design, heavy tail design and a combination of the two, as shown in the rest of the graphs.

\begin{figure}[!t]
    \begin{center}
        \setlength\tabcolsep{2pt}
        \renewcommand{\arraystretch}{0.3}
        \begin{tabular}{rcccc}
            &
            \footnotesize{iid}
            &
            \footnotesize{correlated}
            &
            \footnotesize{heavy tail}
            &
            \footnotesize{correlated + heavy tail} \\
            \rotatebox{90}{\qquad\scriptsize{$(\delta, \epsilon) = (0.9, 0.2)$}}
            &
            \includegraphics[scale=0.285]{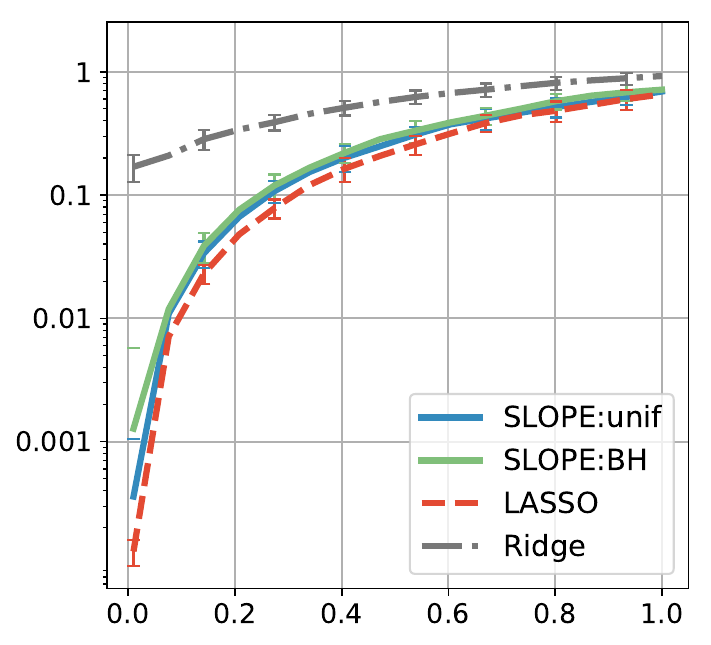}
            &
            \includegraphics[scale=0.285]{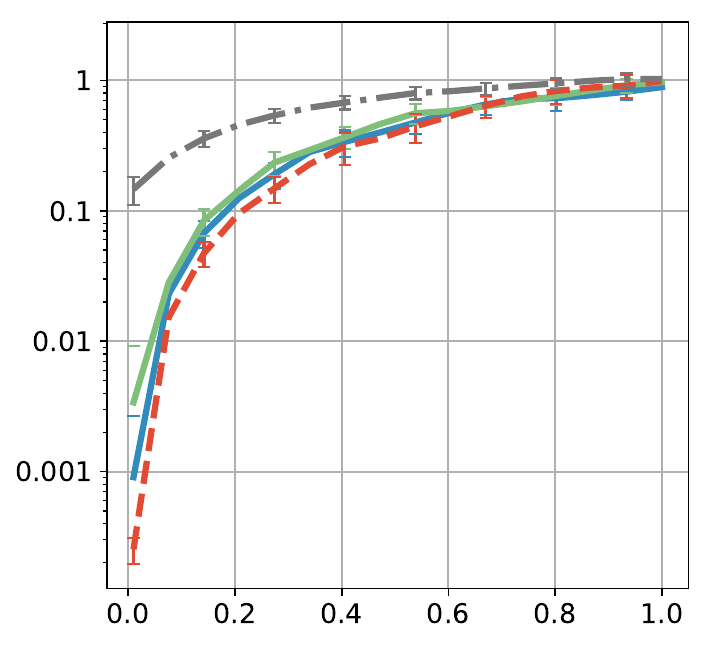}
            &
            \includegraphics[scale=0.285]{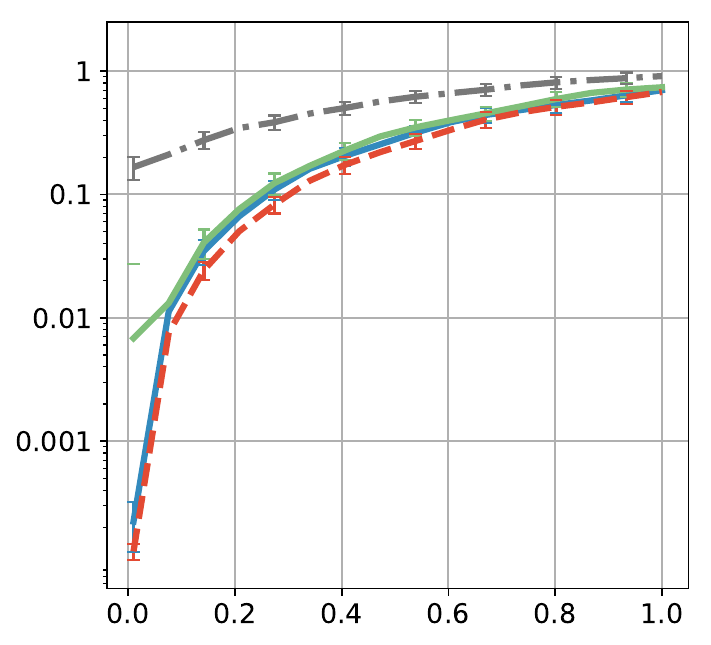}
            &
            \includegraphics[scale=0.285]{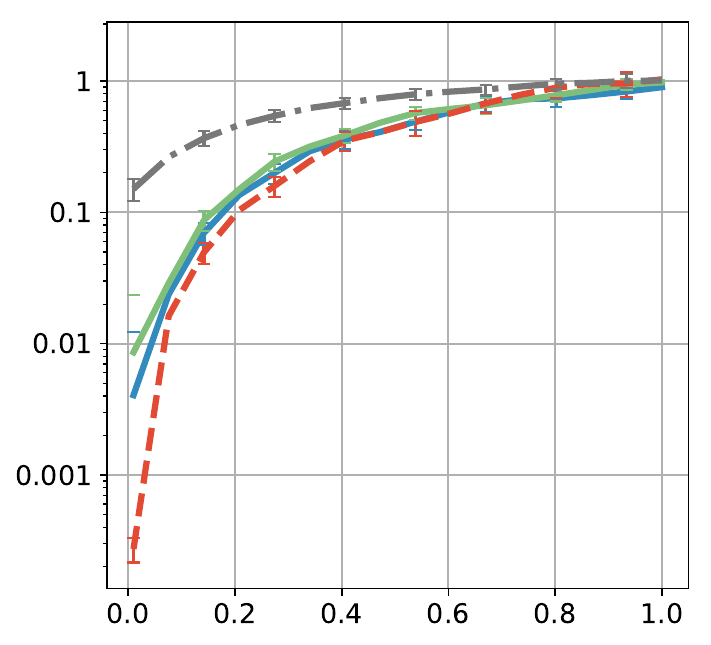} \\
            \rotatebox{90}{\qquad\scriptsize{$(\delta, \epsilon) = (1.1, 0.5)$}}
            &
            \includegraphics[scale=0.285]{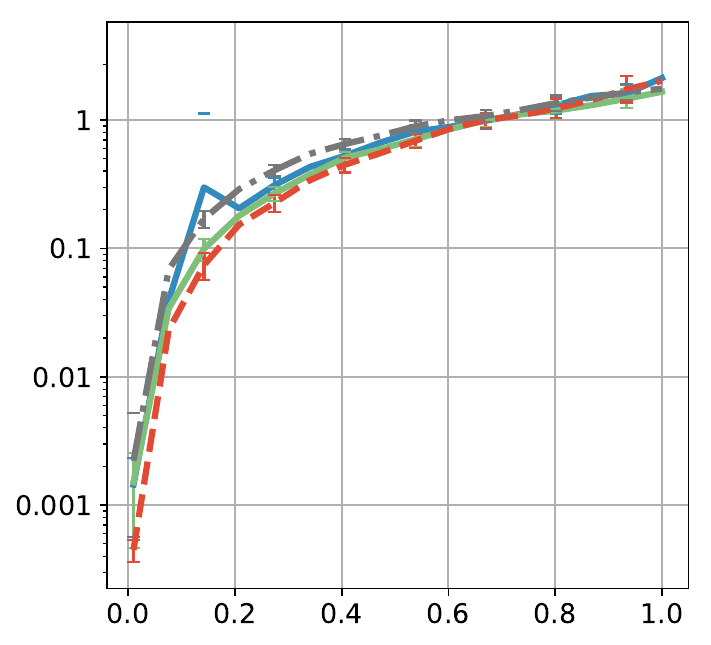}
            &
            \includegraphics[scale=0.285]{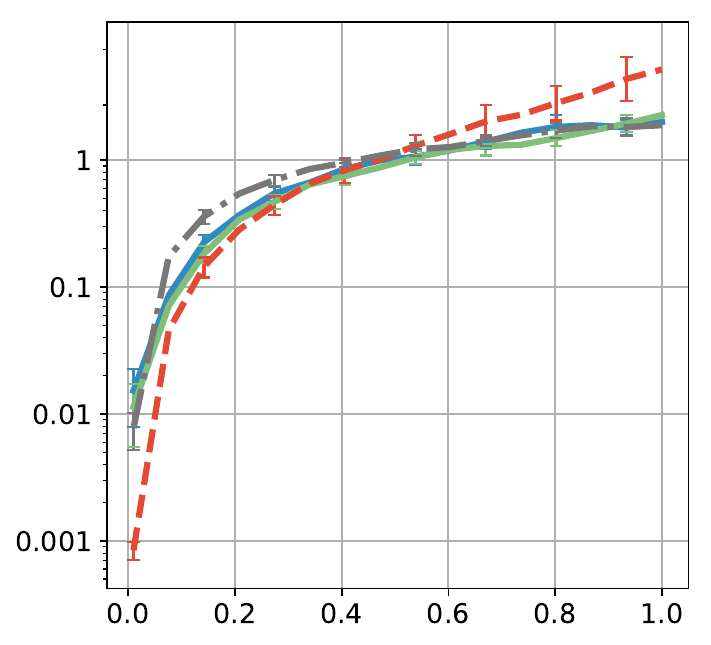}
            &
            \includegraphics[scale=0.285]{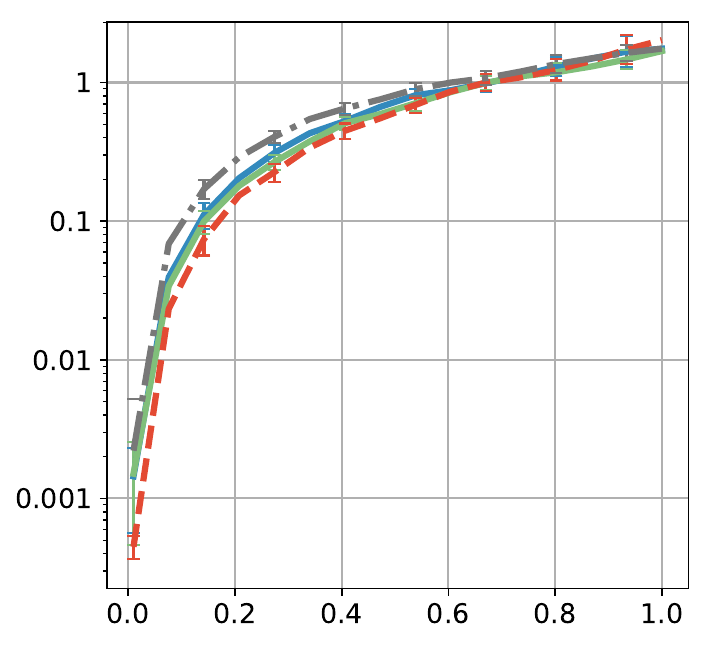}
            &
            \includegraphics[scale=0.285]{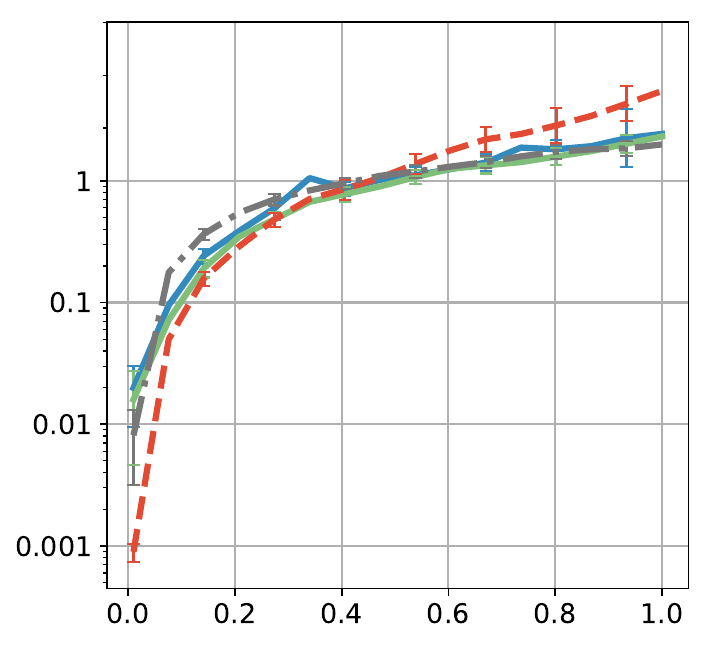} \\
            &
            \footnotesize{$\sigma_z$}
            &
            \footnotesize{$\sigma_z$}
            &
            \footnotesize{$\sigma_z$}
            &
            \footnotesize{$\sigma_z$}
        \end{tabular}
        \caption{MSE of SLOPE, LASSO and Ridge estimators, when the system is above phase transition for both SLOPE and LASSO. A case of $\delta < 1~(\epsilon=0.2)$ is presented in the upper panel, while one for $\delta > 1~(\epsilon=0.5)$ is in the lower panel. The other parameters are the same as in Figure \ref{fig:sample-mse-compare-corr-tail}.} \label{fig:sample-mse-compare-above-pt1}
    \end{center}
\end{figure}

Finally, we examine the condition that the signal $\xv$ does not have tied non-zero components, as required in Theorem \ref{thm:noiseless-phase-transition} and Proposition \ref{lem:best_slope_small_noise}. We empirically demonstrate in Figure \ref{fig:sample-mse-compare-tie-signal} that the condition is necessary for Theorem \ref{thm:noiseless-phase-transition} and Proposition \ref{lem:best_slope_small_noise} to hold. As is clear from the figure, for the signal $\xv$ of which the non-zero components are all equal to 5, LASSO is significantly outperformed by the SLOPE estimator (SLOPE:max2) with $\lambda_1 = \lambda_2 = 1 > 0 = \lambda_3 = \ldots = \lambda_p$ in the low noise scenario. This is because the sorted $\ell_1$ penalty in SLOPE (with appropriately chosen weights) promotes estimators that have tied non-zero elements, while $\ell_1$ penalty can only promote sparsity. Therefore, SLOPE better exploits the existing structures in the signals. Note that the choice of the penalty weights is critical for SLOPE to take full advantage of the signal structures. For example, the other SLOPE estimator (SLOPE:unif), with the (unordered) weights being uniformly sampled, does not behave as well as SLOPE:max2.

\begin{figure}[t!]
    \begin{center}
        \setlength\tabcolsep{2pt}
        \renewcommand{\arraystretch}{0.3}
        \begin{tabular}{rcc}
            &
            \footnotesize{iid}
            &
            \footnotesize{correlated} \\
            \rotatebox{90}{\qquad\qquad\quad\scriptsize{$\mathrm{MSE}$}}
            &
            \includegraphics[scale=0.33]{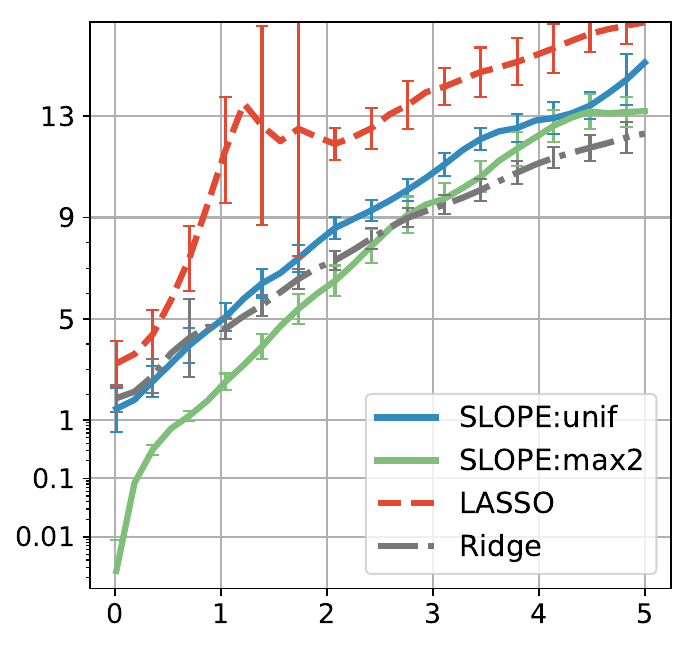}
            &
            \includegraphics[scale=0.33]{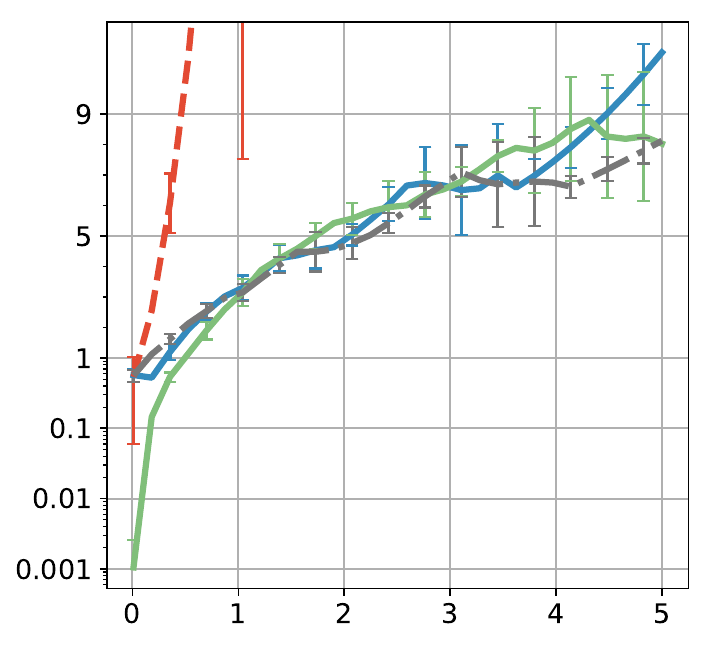} \nonumber \\
            &
            \footnotesize{$\sigma_z$}
            &
            \footnotesize{$\sigma_z$}
        \end{tabular}
        \caption{MSE of SLOPE, LASSO and Ridge estimators, when there are tied non-zero elements in the signal. SLOPE:max2 denotes the SLOPE estimator with weights $\lambda_1 = \lambda_2=1$ and $\lambda_i=0$ for $i \geq 3$. SLOPE:unif is the same as in Figure \ref{fig:sample-mse-compare-corr-tail}. We set $\delta=0.9$, $\epsilon=0.7$. The non-zero components of $\xv$ all equal to 5.} \label{fig:sample-mse-compare-tie-signal}
    \end{center}
\end{figure}

\section{Discussions} \label{sec:discussions}
We have studied the MSE of SLOPE estimators in the high-dimensional regime where both $k$ and $n$ scale linearly with $p$. With an accurate characterization of MSE, we demonstrated that LASSO and Ridge outperform all the SLOPE estimators in the low and large noise scenarios, respectively. Several important directions are left open.
\begin{enumerate}[(1)]
%\item The concentration inequality we obtained might not be sharp in terms of some model or SLOPE parameters such as noise level $\sigma_z$ and regularization parameter $\gamma$. Due to the non-separability of the SLOPE penalty, deriving a concentration inequality that is optimal in all the parameters is a very challenging problem.
\item Our results are proved under the critical condition that $\Av$ is i.i.d. Gaussian design. In Section \ref{sec:experiments}, numerical results showed that the main conclusions remain valid for dependent and non-Gaussian designs. An important and interesting future research is to derive the precise results for more general designs. 
\item In this paper, our focus is on the impact of noise level. Some other model parameters such as sparsity level play an important role in affecting the performance of SLOPE as well. It is of great interest to understand how SLOPE estimators perform and which one is optimal under different types of scenarios that are described by these parameters. Towards this goal, the general concentration result we have derived in Theorem \ref{master:thm} remains valid and the key is to conduct a different form of sensitivity analysis. A recent work \cite{wang2017bridge} has analyzed the impact of different model parameters (including noise level $\sigma_z$, signal sparsity $\epsilon$, and sampling rate $\delta$) on the variable selection performance of bridge regression via approximate message passing (AMP). The CGMT framework is well tailored for characterizing the mean squared error. To study SLOPE under more complicated error metrics like false discovery rate via CGMT is an interesting and probably challenging future research. 

\end{enumerate}

\section{Proof} \label{sec:proof}
In this section, we present the proofs of Proposition \ref{lem:best_slope_small_noise}, and Theorems \ref{thm:concentration0}, \ref{thm:noiseless-phase-transition}, and \ref{thm:large-noise-ridge-better}. The proof of Theorem \ref{thm:concentration0} is presented in Section \ref{ssec:proof-concen}. Proofs of Proposition \ref{lem:best_slope_small_noise} and Theorems \ref{thm:noiseless-phase-transition} and \ref{thm:large-noise-ridge-better} are then given in Section \ref{sec:proof:thm-large-noise}. Some basic properties of the SLOPE proximal operator $\eta$ that are frequently used in the main proofs are provided in Section \ref{ssec:preliminary}. Lastly, Section \ref{ssec:reference} collects some reference materials used in the proofs.

Before proceeding, we introduce some notations that will be extensively used in the proofs. Recall 
\begin{equation}\label{primal:prox}
    \eta(\uv;\gamma, \lambdav)
    =
    \argmin_{\xv}\frac{1}{2}\|\uv-\xv\|_2^2 + \gamma \|\xv\|_{\lambdav}.
\end{equation}
When the value of $\lambdav$ is clear from the context, we suppress $\lambdav$ and simply use $\eta(\uv; \gamma)$ to denote the proximal operator. We also denote by $\mathcal{D}_\gamma$ the dual SLOPE norm ball with radius $\gamma$:
\begin{align} \label{eq:dual-ball}
    \mathcal{D}_{\gamma}
    :=& \{\vv: \|\vv\|_{\lambdav*} \leq \gamma\} \nonumber \\
    =& \bigg\{\vv: \sum_{i=1}^j|\vv|_{(i)} \leq \gamma \sum_{i=1}^j \lambda_i, ~1\leq j \leq p\bigg\}
\end{align}
with $\|\cdot\|_{\lambdav*}$ being the dual norm of $\|\cdot\|_{\lambdav}$. The characterization of $\mathcal{D}_\gamma$ in \eqref{eq:dual-ball} is proved in Lemma \ref{p1}. Furthermore, in Lemma \ref{property:primal} we will show that the sorted components of $\{|\eta_i|\}_{i=1}^p$ are piecewise constant. Hence, for each $i=1,\ldots, p$, we define 
\begin{equation} \label{eq:tie-set}
    \mathcal{I}_i = \{1\leq j\leq p: |\eta_j(\uv; \gamma, \lambdav)|
    = |\eta_i(\uv; \gamma, \lambdav)|\}.
\end{equation}
This induces a partition $\mathcal{P}$ of $[p]$, defined as
\begin{equation*}
    \mathcal{P} = \{\mathcal{I}_i, 1 \leq i \leq p\}.
\end{equation*}
We note that $\mathcal{P}$ only keeps the unique values of $\{\mathcal{I}_i\}$. Further we define $\mathcal{P}_0$ as a subset of $\mathcal{P}$:
\begin{equation*}
    \mathcal{P}_0 = \{\mathcal{I} \in \mathcal{P}: \eta_i \neq 0 \text{ for } i \in \mathcal{I}\}.
\end{equation*}
It is important to note that $\{\mathcal{I}_i\}_i$, $\mathcal{P}$ and $\mathcal{P}_0$ all depend on $\uv$, $\gamma$ and $\lambdav$. Since this dependency is often clear from the context, we typically suppress this dependency in the notations.

Finally, given a closed set $\mathcal{C} \subset \mathbb{R}^p$ and a point $\uv \in \mathbb{R}^p$, we use $\projv_{\mathcal{C}}(\uv)$ to denote the projection of $\uv$ on $\mathcal{C}$, and $\mathbb{I}_\mathcal{C}(\uv)$ to denote the function with value 0 when $\uv \in \mathcal{C}$ and $\infty$ otherwise. We also reserve the notation $\hv \sim \mathcal{N}(0, \Iv_p)$ and $\gv \sim \mathcal{N}(0, \Iv_n)$.

\subsection{Proof of Theorem \ref{thm:concentration0}} \label{ssec:proof-concen}

 Since the proof is rather involved, we first summarize the main proof ideas in Section \ref{sssec:proof-concen-sketch}. We then expand our arguments in the rest of this section.

\subsubsection{Sketch of the proof} \label{sssec:proof-concen-sketch}
  
Recall $\hat{\xv}=\argmin_{\xv}\frac{1}{2}\|\yv - \Av \xv\|_2^2+\gamma \|\xv\|_{\lambdav}$. Denote $\hat{\wv}=\frac{\hat{\xv}-\xv}{\sqrt{p}}, m_n=\frac{1}{p}\mathbb{E}\|\eta(\xv+\sigma^* \hv;\sigma^* \chi^*)-\xv\|_2^2$, where $(\sigma^*, \chi^*)$ is specified in Theorem \ref{thm:concentration0}. We aim to show $\|\hat{\wv}\|_2$ concentrates around $\sqrt{m_n}$. First, it is straightforward to confirm that
\begin{equation*}
\hat{\wv}=\argmin_{\wv} \frac{1}{2}\|\sqrt{p}\Av \wv-\zv\|_2^2+\gamma \|\sqrt{p}\wv+\xv\|_{\lambdav} := \argmin_{\wv} F_n(\wv).
\end{equation*}

For given $t\geq 0$, define the sets 
\begin{equation*}
    S_w=\{\wv: \|\wv\|_2\leq 2\sqrt{m_n}+t\},
    \qquad
    H_{t} = \{\wv:  |\|\wv\|_2 - \sqrt{m_n}|\geq t\}.
\end{equation*}

If we are able to prove that
\begin{equation} \label{eq:minimize-cond}
    \min_{\wv \in S_w}F_n(\wv)< \min_{\wv \in S_w \cap H_{t}} F_n(\wv),
\end{equation}
then $|\|\bm \hat{\wv}\|_2 - \sqrt{m_n}|\leq t$. To see why this is true, it is clear that \eqref{eq:minimize-cond} implies $\hat{\wv} \in H^c_{t} \cup S_w^c$. Suppose $\hat{\wv}\in S_w^c$, and denote $\wv^*=\argmin_{\wv\in S_w}F_n(\wv)$. Since $\wv^* \in H_{t}^c \subsetneq S_w$ and $\hat{\wv}\in S_w^c$, there exists a constant $\lambda\in (0,1)$ such that $\lambda \wv^*+(1-\lambda)\bm \hat{\wv} \in S_w\cap H_{t}$. By the convexity of $F_n(\wv)$, it holds that
\begin{equation*}
    \min_{\wv\in S_w\cap H_{t}} F_n(\wv)
    \leq
    F_n(\lambda \wv^*+(1-\lambda) \hat{\wv})
    \leq
    \lambda F_n(\wv^*)+(1-\lambda)F_n(\hat{\wv})
    \leq
    \min_{\wv\in S_w}F_n(\wv).
\end{equation*}

This is a contradiction. Hence, $\hat{\wv} \in H_{t}^c$. 

Based on the preceding arguments, it is sufficient to obtain $\min_{\wv \in S_w}F_n(\wv)< \min_{\wv \in S_w \cap H_{t}} F_n(\wv)$ w.h.p. Towards this goal, in Section \ref{sssec:proof-concen-upperlower}, we will associate the primal optimization problem with an auxiliary optimization problem $\min_{\alpha}\max_{\beta, T_h}\hat{\Lambda}(\alpha, \beta, T_h)$ and use it to establish a tight ``upper bound'' for $\min_{\wv \in S_w} F_n(\wv)$ and a ``lower bound'' for $\min_{\wv \in S_w \cap H_{t}} F_n(\wv)$ in the following way:
\begin{align}
    \frac{1}{p}\min_{\wv \in S_w}F_n(\wv) \leq_p& \min_{0\leq \alpha \leq 2\sqrt{m_n}+t} \max_{\beta\geq 0,T_h>0} \hat{\Lambda}(\alpha,\beta,T_h), \label{minimax:first} \\
    \frac{1}{p}\min_{\wv \in S_w \cap H_{t}}F_n(\wv) \geq_p& \min_{\substack{0\leq \alpha \leq 2\sqrt{m_n}+t \\ |\alpha-\sqrt{m_n}|\geq t}} \max_{\beta\geq 0,T_h>0} \hat{\Lambda}(\alpha,\beta,T_h) \label{minimax:two}
\end{align}
The above derivation is based on the convex Gaussian minimax theorem (CGMT) approach (Theorem \ref{thm:cgmt}) which was developed in its full generality in \cite{thrampoulidis2015regularized, thrampoulidis2018precise}. An accurate explanation of $\leq_p$ and $\geq_p$ is presented in Lemma \ref{lemma:cgmt-upper-lower}. For now one may treat them as normal $\leq$ and $\geq$. As a result, as long as we can further compare the upper and lower bounds from \eqref{minimax:first} and \eqref{minimax:two} in the form like
\begin{equation} \label{eq:cgmt-sample-smaller}
 \min_{0\leq \alpha \leq 2\sqrt{m_n}+t} \max_{\beta\geq 0,T_h>0} \hat{\Lambda}(\alpha,\beta,T_h) < \min_{\substack{0\leq \alpha \leq 2\sqrt{m_n}+t \\ |\alpha-\sqrt{m_n}|\geq t}} \max_{\beta\geq 0,T_h>0} \hat{\Lambda}(\alpha,\beta,T_h) , \quad {\rm w.h.p}
\end{equation}
our goal is achieved. To obtain this result, in Sections \ref{sssec:expect-solution-analysis} and \ref{sssec:proof-concen-uniform}, we establish a uniform concentration of $\hat{\Lambda}(\alpha, \beta, T_h)$ around its population version denoted as $\Lambda(\alpha, \beta, T_h)$, and show that $\sqrt{m_n}$ belongs to the saddle point of the minimax problem $\min_{\alpha}\max_{\beta,T_h}\Lambda(\alpha, \beta, T_h)$ so that
\begin{equation*}
     \min_{0\leq \alpha \leq 2\sqrt{m_n}+t} \max_{\beta\geq 0,T_h>0} \Lambda(\alpha,\beta,T_h) < \min_{\substack{0\leq \alpha \leq 2\sqrt{m_n}+t \\ |\alpha-\sqrt{m_n}|\geq t}} \max_{\beta\geq 0,T_h>0} \Lambda(\alpha,\beta,T_h),
\end{equation*}
which leads to \eqref{eq:cgmt-sample-smaller} through the uniform concentration by choosing appropriate values of $t$. We will make this argument formal and precise in Section \ref{master:for:all:case}.

\subsubsection{The upper and lower bounds involving $\hat{\Lambda}$} \label{sssec:proof-concen-upperlower}
Recall the notations $\yv= \Av \xv +\zv$ and $\hv \sim \mathcal{N}(0, \Iv_p), \gv \sim \mathcal{N}(0, \Iv_n)$. Define the function $\hat{\Lambda}$ in the following way: 
\begin{align} \label{eq:Lambda}
    &\hat{\Lambda}(\alpha,\beta,T_h)  \\
    =&
    \begin{cases}
    \frac{\sqrt{n}}{p} \|\zv\|_2\beta-\frac{n\beta^2}{2p}+\frac{\gamma}{p} \|\xv\|_{\lambdav}, & \text{if }\alpha =0,\beta\geq 0, T_h>0 \nonumber \\
    -\frac{n}{2p} \beta^2+\frac{\|\sqrt{p}\alpha \gv - \sqrt{n}\zv\|_2}{p}\beta - \frac{\alpha T_h}{2} + \frac{\hv^\top \xv}{p}\beta +\frac{T_h}{2\alpha p} (\|\xv\|_2^2-\|\eta(\xv + \frac{\alpha \beta}{T_h}\hv; \frac{\alpha \gamma}{T_h})\|_2^2), & \text{if } \alpha>0,\beta\geq 0, T_h>0.  
    \end{cases}
\end{align}
The role of this quantity in our analysis was described in the last section. The following lemma relates $F_n(\wv)$ with $\hat{\Lambda}(\alpha,\beta,T_h)$.
\begin{lemma} \label{lemma:cgmt-upper-lower}
    For any given constant $c\in \mathbb{R}$, the following inequalities hold 
    \begin{align*}
\mathbb{P}\Big( \frac{1}{p}\min_{\wv \in S_w}F_n(\wv)\geq c \Big)\leq & 2\mathbb{P}\Big(\min_{0\leq \alpha \leq 2\sqrt{m_n}+t} \max_{\beta \geq 0, T_h>0} \hat{\Lambda}(\alpha,\beta,T_h)\geq c\Big), \\
\mathbb{P}\Big( \frac{1}{p}\min_{\wv \in S_w \cap H_{t}}F_n(\wv) \leq c\Big)\leq & 2\mathbb{P}\Big(\min_{\substack{0\leq \alpha \leq 2\sqrt{m_n}+t \\ |\alpha-\sqrt{m_n}|\geq t}} \max_{\beta \geq 0, T_h>0} \hat{\Lambda}(\alpha,\beta,T_h)\leq c\Big).
    \end{align*}
\end{lemma}

\begin{proof}
    We prove these two bounds separately. 
    \paragraph{The upper bound:}
    Denote $S_r=\{\uv: \|\uv\|_2\leq r\}$. Using the identity $\frac{1}{2}\|\bv\|_2^2=\max_{\uv} \sqrt{n}\uv^\top \bv -\frac{n}{2}\|\uv\|_2^2$ with $\uv \in \mathbb{R}^n$, we obtain
    \begin{align}
        \min_{\wv \in S_w}F_n(\wv)=&\lim_{r\rightarrow \infty} \min_{\wv \in S_w}\max_{\uv\in S_r} \sqrt{n} \uv^\top(\sqrt{p}\Av \wv-\zv)-\frac{n}{2}\|\uv\|_2^2+\gamma \|\sqrt{p}\wv+\xv\|_{\lambdav}, \nonumber \\%\label{take:limit:end} 
    =&\lim_{r\rightarrow \infty} \min_{\wv \in S_w}\max_{\uv \in S_r} \sqrt{p}\uv^\top \bm \tilde{A} \wv-\sqrt{n}\uv^\top\zv-\frac{n}{2}\|\uv\|_2^2+\gamma \|\sqrt{p}\wv+\xv\|_{\lambdav}, \nonumber \\
    =&\lim_{r\rightarrow \infty} \min_{\wv \in S_w}\max_{\uv \in S_r} \max_{\sv \in \mathcal{D}_{\gamma}}\underbrace{\sqrt{p}\uv^\top \bm \tilde{A} \wv-\sqrt{n}\uv^\top\zv-\frac{n}{2}\|\uv\|_2^2+\sv^\top(\sqrt{p}\wv+\xv)}_{:=f(\wv,\uv, \sv)}, \label{primal:format}
\end{align}
where $\tilde{\Av}=\sqrt{n}\Av$ has independent standard normal entries, and $\mathcal{D}_\gamma$ is defined in \eqref{eq:dual-ball}. Note that the third equality is due to the fact that $\mathcal{D}_\gamma$ is the dual norm (w.r.t. $\|\cdot\|_{\lambdav}$) ball with radius $\gamma$. According to the CGMT (Part (\ref{thm:item:cgmt-lower}) in Theorem \ref{thm:cgmt}), the expression in \eqref{primal:format} is closely related to 
\begin{equation*}
    \max_{\uv \in S_r} \max_{\sv \in \mathcal{D}_{\gamma}}\min_{\wv \in S_w} \underbrace{\sqrt{p}\|\wv\|_2 \gv^\top\uv+ \sqrt{p}\|\uv\|_2 \hv^\top \wv-\sqrt{n} \uv^\top \zv-\frac{n}{2}\|\uv\|^2_2+\sv^\top(\sqrt{p}\wv+\xv)}_{:=\tilde{f}(\wv,\uv,\sv)}.
\end{equation*}
Specifically, 
\begin{align}
\label{use:limit:end}
\mathbb{P}\Big(\min_{\wv \in S_w}\max_{\uv \in S_r} \max_{\sv \in \mathcal{D}_{\gamma}} f(\wv,\uv,\sv)\geq c\Big)\leq 2 \mathbb{P}\Big(\max_{\uv \in S_r} \max_{\sv \in \mathcal{D}_{\gamma}}\min_{\wv \in S_w} \tilde{f}(\wv,\uv,\sv) \geq c\Big)
\end{align}

We now further upper bound $\max_{\uv \in S_r} \max_{\sv \in \mathcal{D}_{\gamma}}\min_{\wv \in S_w} \tilde{f}(\wv,\uv,\sv)$ to obtain simpler expressions.
\begin{align*}
    &\max_{\uv \in S_r} \max_{\sv \in \mathcal{D}_{\gamma}}\min_{\wv \in S_w} \tilde{f}(\wv,\uv,\sv)=\max_{\uv \in S_r} \max_{\sv \in \mathcal{D}_{\gamma}}\min_{0\leq \alpha\leq 2\sqrt{m_n}+t} \min_{\|\wv\|_2=\alpha}\tilde{f}(\wv,\uv,\sv) \\
    =&\max_{\uv \in S_r} \max_{\sv \in \mathcal{D}_{\gamma}}\min_{0\leq \alpha\leq 2\sqrt{m_n}+t} \sqrt{p}\alpha \gv^\top \uv- \sqrt{p} \big\| \|\uv\|_2\hv+\sv \big\|_2 \alpha -\sqrt{n}\uv^\top \zv-\frac{n}{2}\|\uv\|_2^2+\sv^\top \xv \\
    \leq& \max_{\sv \in \mathcal{D}_{\gamma}} \min_{0\leq \alpha \leq 2\sqrt{m_n}+t} \max_{0\leq \beta \leq r} \max_{\|\uv\|_2=\beta} \sqrt{p}\alpha \gv^\top \uv- \sqrt{p} \| \|\uv\|_2\hv+\sv\|_2 \alpha -\sqrt{n}\uv^\top \zv-\frac{n}{2}\|\uv\|_2^2+\sv^\top \xv \\
    =& \max_{\sv \in \mathcal{D}_{\gamma}} \min_{0\leq \alpha \leq 2\sqrt{m_n}+t} \max_{0\leq \beta \leq r} -\sqrt{p}\|\beta \hv+\sv\|_2\alpha -\frac{n}{2} \beta^2+\|\sqrt{p}\alpha \gv-\sqrt{n}\zv\|_2\beta+\sv^\top \xv \\
    \leq& \min_{0\leq \alpha \leq 2\sqrt{m_n}+t} \max_{0\leq \beta \leq r} \max_{\sv \in \mathcal{D}_{\gamma}}  -\sqrt{p}\|\beta \hv+\sv\|_2\alpha -\frac{n}{2} \beta^2+\|\sqrt{p}\alpha \gv-\sqrt{n}\zv\|_2\beta+\sv^\top \xv \\
    =&\min_{0\leq \alpha \leq 2\sqrt{m_n}+t} \max_{0\leq \beta \leq r} \max_{\sv\in \mathcal{D}_{\gamma}} \max_{T_h>0}    -\frac{\alpha \sqrt{p}}{2}\Big(\frac{\|\beta \hv+\sv\|_2^2}{\sqrt{p}T_h}+\sqrt{p}T_h \Big)-\frac{n}{2} \beta^2+\|\sqrt{p}\alpha \gv-\sqrt{n}\zv\|_2\beta+\sv^\top \xv \\
    =&\min_{0\leq \alpha \leq 2\sqrt{m_n}+t} \max_{0\leq \beta \leq r} \max_{T_h>0}  \underbrace{\max_{\sv \in \mathcal{D}_{\gamma}}  -\frac{\alpha \sqrt{p}}{2}\Big(\frac{\|\beta \hv+\sv\|_2^2}{\sqrt{p}T_h}+\sqrt{p}T_h \Big)-\frac{n}{2} \beta^2+\|\sqrt{p}\alpha \gv-\sqrt{n}\zv\|_2\beta+\sv^\top \xv}_{:=\hat{f}(\alpha,\beta, T_h)},
\end{align*}
where the two inequalities above follow from the weak duality. The next step is to simplify $\hat{f}(\alpha,\beta, T_h)$. It is clear that $\hat{f}(0,\beta, T_h)=\sqrt{n}\|\zv\|_2\beta-\frac{n}{2}\beta^2+\gamma \|\xv\|_{\lambdav}$. When $\alpha> 0, \beta\geq 0, T_h>0$, we have
\begin{align*}
\hat{f}(\alpha,\beta, T_h)
=&-\frac{n}{2} \beta^2+\|\sqrt{p}\alpha \gv-\sqrt{n}\zv\|_2\beta-\frac{\alpha pT_h}{2} + \frac{T_h \|\xv\|_2^2-2\hv^\top \xv \alpha \beta }{2\alpha }+
\max_{\sv\in \mathcal{D}_{\gamma}} \frac{- \alpha}{2 T_h}\bigg\|\sv + \beta \hv-\frac{T_h \xv}{\alpha } \bigg\|_2^2 \\
=&-\frac{n}{2} \beta^2+\|\sqrt{p}\alpha \gv-\sqrt{n}\zv\|_2\beta-\frac{\alpha pT_h}{2}+\frac{T_h \|\xv\|_2^2-2\hv^\top \xv \alpha \beta }{2\alpha} -\frac{\alpha }{2T_h} \bigg\|\eta \bigg(\frac{T_h \xv}{\alpha }-\beta \hv; \gamma\bigg)\bigg\|_2^2 \\
=&-\frac{n}{2} \beta^2+\|\sqrt{p}\alpha \gv-\sqrt{n}\zv\|_2\beta-\frac{\alpha pT_h}{2} - \hv^\top \xv\beta + \frac{T_h}{2\alpha }\bigg(\|\xv\|_2^2 - \bigg\|\eta\bigg(\xv-\frac{\alpha \beta}{T_h}\hv;\frac{\alpha \gamma}{T_h} \bigg) \bigg\|_2^2\bigg).
\end{align*}
The last two equalities are due to Lemma \ref{p1} and Lemma \ref{property:primal} (\ref{lemma:item:prox-scalar}), respectively. These results combined with \eqref{use:limit:end} yield that 
\begin{align*}
\mathbb{P}\Big(\frac{1}{p}\min_{\wv \in S_w}\max_{\uv \in S_r} \max_{\sv \in \mathcal{D}_{\gamma}} f(\wv,\uv,\sv)\geq c\Big)\leq 2 \mathbb{P}\Big(\min_{0\leq \alpha \leq 2\sqrt{m_n}+t} \max_{0\leq \beta \leq r}\max_{T_h>0} \hat{\Lambda}(\alpha,\beta,T_h)\geq c\Big), ~~ \forall r>0.
\end{align*} 
According to dominated convergence theorem, letting $r\rightarrow \infty$ on both sides of the above inequality proves the upper bound. 

\paragraph{The lower bound:}
Similar to \eqref{primal:format} we have 
\begin{equation*}
\min_{\wv \in S_w \cap H_t}F_n(\wv)=\lim_{r\rightarrow \infty} \min_{\wv\in S_w \cap H_{t}}\max_{\uv \in S_r}\max_{\sv \in \mathcal{D}_{\gamma}}f(\wv,\uv, \sv)
\end{equation*}
From the CGMT (Part (\ref{thm:item:cgmt-upper}) in Theorem \ref{thm:cgmt}), 
\begin{align*}
\mathbb{P}\Big(\min_{\wv\in S_w \cap H_{t}}\max_{\uv \in S_r}\max_{\sv \in \mathcal{D}_{\gamma}} f(\wv,\uv,\sv)\leq c\Big)\leq 2 \mathbb{P}\Big(\min_{\wv\in S_w \cap H_{t}}\max_{\uv \in S_r}\max_{\sv \in \mathcal{D}_{\gamma}} \tilde{f}(\wv,\uv,\sv) \leq c\Big)
\end{align*}
We would like to find a lower bound:
\begin{align*}
    &\min_{\wv\in S_w \cap H_{t}}\max_{\uv \in S_r}\max_{\sv \in \mathcal{D}_{\gamma}} \tilde{f}(\wv, \uv, \sv) =\min_{\wv\in S_w \cap H_{t}} \max_{\sv \in \mathcal{D}_{\gamma}, 0\leq \beta \leq r} \max_{\|\uv\|_2=\beta} \tilde{f}(\wv, \uv, \sv) \\
=& \min_{\wv\in S_w \cap H_{t}} \max_{\sv \in \mathcal{D}_{\gamma}, 0\leq \beta \leq r} \sqrt{p}\hv^\top \wv\beta-\frac{n}{2}\beta^2+\|\sqrt{p}\|\wv\|_2\gv-\sqrt{n}\zv\|_2\beta+\sv^\top(\sqrt{p}\wv+ \xv) \\
\geq& \min_{\substack{0\leq \alpha \leq 2\sqrt{m_n}+t \\ |\alpha-\sqrt{m_n}|\geq t}}\max_{\sv \in \mathcal{D}_{\gamma}, 0\leq \beta \leq r} \min_{\|\wv\|_2=\alpha}  \sqrt{p}\hv^\top \wv\beta-\frac{n}{2}\beta^2+\|\sqrt{p}\|\wv\|_2\gv-\sqrt{n}\zv\|_2\beta+\sv^\top(\sqrt{p}\wv+ \xv) \\
=& \min_{\substack{0\leq \alpha \leq 2\sqrt{m_n}+t \\ |\alpha-\sqrt{m_n}|\geq t}}\max_{0\leq \beta \leq r}\max_{\sv \in \mathcal{D}_{\gamma}}  -\frac{n}{2}\beta^2+\|\sqrt{p}\alpha \gv-\sqrt{n} \zv\|_2\beta-\sqrt{p} \|\beta \hv+\sv\|_2\alpha+\sv^\top \xv.
\end{align*}
The rest of the proof is the same as the one for the upper bound. 
\end{proof}

\subsubsection{Solution analysis of $\Lambda$} \label{sssec:expect-solution-analysis}
The bounds we obtained in Section \ref{sssec:proof-concen-upperlower} are in the min-max form of the function $\hat{\Lambda}(\alpha,\beta,T_h)$. To simplify the bounds further, we will connect $\hat{\Lambda}$ with its population version $\Lambda$. In this section, we analyze the properties of the saddle point of $\Lambda$. Then in Section \ref{sssec:proof-concen-uniform}, we study the uniform concentration of $\hat{\Lambda}(\alpha, \beta, T_h)$ around $\Lambda(\alpha, \beta, T_h)$. Let $\delta=\frac{n}{p}$ and define 
\begin{align} \label{eq:Gamma}
   & \Lambda(\alpha,\beta,T_h) \\
    =&
    \begin{cases}
          \delta \sigma_z \beta-\frac{\delta}{2} \beta^2+\frac{\gamma \|\xv\|_{\lambdav}}{p}, & \text{if } \alpha=0,\beta\geq 0,T_h>0, \\ 
        \sqrt{\alpha^2\delta +\delta^2\sigma_z^2}\beta-\frac{\alpha T_h}{2} - \frac{\delta}{2}\beta^2+\frac{T_h}{2\alpha}\frac{\|\xv\|^2_2 - \mathbb{E}\|\eta(\xv+\frac{\alpha\beta \hv}{T_h}; \frac{\alpha \gamma}{T_h})\|_2^2}{p} & \text{if } \alpha > 0, \beta\geq 0, T_h>0.
    \end{cases}
    \nonumber 
\end{align}

\begin{lemma}\label{saddle:point:thm} 
Consider the min-max problem,
\begin{equation*}
    \min_{\alpha \geq 0}\max_{\beta \geq 0}\max_{T_h > 0} \Lambda(\alpha,\beta,T_h).
\end{equation*}
For $\sigma_z \geq 0$, $\gamma > 0$, the following results hold:
    \begin{enumerate}[(i)]
    \item \label{lemma:item:lambda-convex-concave} $\Lambda(\alpha, \beta, T_h)$
        is convex in $\alpha$ and jointly concave in $(\beta, T_h)$.
    \item \label{lemma:item:saddle-points} The set of saddle points is non-empty and compact. 
    \item \label{lemma:item:solu-analysis} Let $(\alpha^*, \beta^*, T_h^*)$ be
        a saddle point of the system. Then we have $\alpha^*, \beta^*, T_h^* > 0$.
    \item \label{lemma:item:se} Any saddle point $(\alpha^*,
        \beta^*, T^*_h)$ satisfies the following system of equations:
        \begin{equation}\label{threeeq:mark} 
            \begin{cases}
            (\alpha^*)^2=\frac{1}{p}\mathbb{E}\|\eta(\xv+\frac{\alpha^*\beta^*}{T_h^*}\hv;\frac{\alpha^*\gamma}{T_h^*})-\xv\|_2^2, &\\
            \frac{1}{p}\mathbb{E}\big\langle \hv, \eta\big(\xv+\frac{\alpha^*\beta^*}{T_h^*}\hv;\frac{\alpha^*\gamma}{T_h^*}\big)\big\rangle=\sqrt{(\alpha^*)^2\delta+\delta^2\sigma_z^2}-\delta \beta^*, & \\
            \frac{\alpha^* \beta^*}{T_h^*}=\frac{\sqrt{(\alpha^*)^2\delta+\delta^2\sigma_z^2}}{\delta}. &
            \end{cases}
        \end{equation}
        Moreover, by setting $\alpha^*=\sqrt{\delta((\sigma^*)^2-\sigma_z^2)},\beta^*=\frac{\gamma}{\chi^*}, T_h^*=\frac{\gamma\sqrt{\delta((\sigma^*)^2-\sigma_z^2)}}{\sigma^* \chi^*}$, the above three equations are simplified to
        \begin{equation} \label{twoeq:mark}
            \begin{cases}
                (\sigma^*)^2 = \sigma_z^2 + \frac{1}{\delta p} \mathbb{E}\|\eta(\xv+\sigma^* \hv; \sigma^* \chi^* )-\xv\|_2^2, \\
                \gamma = \sigma^* \chi^* \Big(1 - \frac{1}{\delta \sigma^* p}\mathbb{E}\langle \eta(\xv+\sigma^* \hv;\sigma^* \chi^*), \hv \rangle \Big).
            \end{cases}
        \end{equation}
        
    \end{enumerate}
\end{lemma}

\begin{proof}
    \textbf{\emph{Part (\ref{lemma:item:lambda-convex-concave})}}:
    According to Lemma \ref{p1}, we have
    \begin{align*}
    \Big\|\eta\Big(\xv+\frac{\alpha\beta \hv}{T_h}; \frac{\alpha \gamma}{T_h}\Big)\Big\|_2^2=\min_{\sv\in \mathcal{D}_{\alpha \gamma/T_h}}\Big\|\xv+\frac{\alpha\beta \hv}{T_h}-\sv\Big\|_2^2=\frac{\alpha^2\gamma^2}{T_h^2}\min_{\sv \in \mathcal{D}_1}\Big\|\frac{T_h}{\alpha\gamma}\xv+\frac{\beta\hv}{\gamma}-\sv\Big\|_2^2
    \end{align*}
  We then obtain the following form of $\Lambda$ when $\alpha>0,\beta\geq 0, T_h>0$:
    \begin{align}
        \Lambda(\alpha, \beta, T_h) =
        \beta \sqrt{\alpha^2\delta + \sigma_z^2\delta^2} - \frac{\alpha T_h}{2} - \frac{\delta\beta^2}{2} - \frac{\gamma}{p} \mathbb{E}\min_{\sv \in \mathcal{D}_1} \bigg\{ \Big\langle\xv, \frac{\beta}{\gamma}\hv - \sv \Big\rangle + \frac{\alpha\gamma}{2T_h}\Big\|\frac{\beta}{\gamma}\hv - \sv \Big\|_2^2 \bigg\}  \label{eq:Gamma-form3}
    \end{align}
From the form \eqref{eq:Gamma-form3}, it is straightforward to verify that $\Lambda(\alpha, \beta, T_h)$ is convex in $\alpha\in (0,\infty)$ and continuous at $\alpha=0$, thus $\Lambda(\alpha, \beta, T_h)$ is convex in $\alpha$ over $[0,\infty)$. Furthermore, since the perspective operation preserves convexity, it is direct to confirm that $\langle\xv, \frac{\beta}{\gamma}\hv - \sv \rangle + \frac{\alpha\gamma}{2T_h}\|\frac{\beta}{\gamma}\hv - \sv \|_2^2$ is jointly convex in $(\beta, T_h,\sv)$, which further implies the joint concavity of $\Lambda(\alpha, \beta, T_h)$ in $(\beta, T_h)$ if $\alpha>0$. When $\alpha=0$, it is clear from \eqref{eq:Gamma} that $\Lambda(\alpha, \beta, T_h)$ is concave in $(\beta, T_h)$.

    \textbf{\emph{Part (\ref{lemma:item:saddle-points})}}:  We aim to apply the Saddle Point Theorem (Theorem \ref{saddle:thm}). To satisfy the closeness condition, we introduce an extended definition of $\Lambda$ as follows:
\begin{equation*} 
    \Lambda(\alpha,\beta,T_h)
    =
    \begin{cases}
          \delta \sigma_z \beta-\frac{\delta}{2} \beta^2+\frac{\gamma \|\xv\|_{\lambdav}}{p}, & \text{if } \alpha=0,\beta\geq 0, T_h>0\\
             \sqrt{\alpha^2\delta +\delta^2\sigma_z^2}\beta-\frac{\alpha T_h}{2} - \frac{\delta}{2}\beta^2+\frac{T_h}{2\alpha}\frac{\|\xv\|^2_2 - \mathbb{E}\|\eta(\xv+\frac{\alpha\beta \hv}{T_h}; \frac{\alpha \gamma}{T_h})\|_2^2}{p} & \text{if } \alpha > 0, \beta\geq 0, T_h>0 \\
           0 & \text{if } \alpha \geq 0, \beta=0, T_h=0 \\
           -\infty & \text{if }\alpha \geq 0, \beta>0, T_h=0 \\
    \end{cases}
\end{equation*}   
   It is direct to confirm that the saddle points remain unchanged after the extension. Hence in the rest of the proof, we will refer to $\Lambda(\alpha,\beta,T_h)$ as the above extended function. Based on Part (\ref{lemma:item:lambda-convex-concave}), it is straightforward to verify that the convexity and closeness conditions are satisfied by $\Lambda(\alpha,\beta,T_h)$. To invoke the Saddle Point Theorem, we further find $(\bar{\alpha},\bar{\beta},\bar{T}_h)\in \mathbb{R}_+\times \mathbb{R}_+ \times \mathbb{R}_{+}, \bar{c}\in \mathbb{R}$ such that the following two sets are nonempty and compact:
\begin{equation*}
    \mathcal{H}_1=\{\alpha \geq 0: \Lambda(\alpha,\bar{\beta},\bar{T}_h)\leq \bar{c}\},
    \quad
    \mathcal{H}_2=\{\beta\geq 0, T_h \geq 0: \Lambda(\bar{\alpha}, \beta, T_h)\geq \bar{c} \}.
\end{equation*}

Since $\gamma > 0$, we are able to choose $\bar{T}_h>0$ small enough so that $\frac{1}{p}\mathbb{E}\|\eta(\frac{\hv}{\sqrt{\delta}};\frac{\gamma}{\bar{T}_h})\|_2^2 < \frac{1}{3}$ and $\bar{\beta}=\frac{\bar{T}_h}{\sqrt{\delta}}$. Then we have for $\alpha>0$
\begin{equation}\label{construct:one}
    \Lambda(\alpha,\bar{\beta},\bar{T}_h)
    = \alpha \bar{T}_h\bigg(\sqrt{1+\frac{\delta \sigma_z^2}{\alpha^2}}-\frac{1}{2}-\frac{\bar{T}_h}{2\alpha}+\frac{\|\xv\|_2^2}{2p\alpha^2}-\frac{\mathbb{E}\|\eta(\frac{\xv}{\alpha}+\frac{\hv}{\sqrt{\delta}};\frac{\gamma}{\bar{T}_h})\|^2_2}{2p}\bigg).
\end{equation}
Also, by Lemma \ref{slope:prop} Part (\ref{lemma:item:prox-nonexpansive0}), we conclude that
\begin{align}\label{construct:two}
    \bigg\|\eta\bigg(\frac{\xv}{\alpha}+\frac{\hv}{\sqrt{\delta}};\frac{\gamma}{\bar{T}_h} \bigg)\bigg\|_{\mathcal{L}_2}^2
   & \leq 2\bigg\|\eta\bigg(\frac{\xv}{\alpha}+\frac{\hv}{\sqrt{\delta}};\frac{\gamma}{\bar{T}_h}\bigg)-\eta\bigg(\frac{\hv}{\sqrt{\delta}};\frac{\gamma}{\bar{T}_h}\bigg)\bigg\|_{\mathcal{L}_2}^2 + 2\bigg\| \eta \bigg(\frac{\hv} {\sqrt{\delta}}; \frac{\gamma}{\bar{T}_h}\bigg)\bigg\|_{\mathcal{L}_2}^2 \nonumber \\
    &\leq \frac{2\|\xv\|_2^2}{p \alpha^2} + \frac{2}{3}.
\end{align}

Combining \eqref{construct:one} and \eqref{construct:two} we know that $\lim_{\alpha \rightarrow \infty} \Lambda(\alpha,\bar{\beta},\bar{T}_h)=+\infty$. Hence, we can choose $\bar{\alpha}>0$ and $\bar{c}<\infty$ such that 
\begin{equation}\label{construct:three}
    \bar{c}=\Lambda(\bar{\alpha}, \bar{\beta},\bar{T}_h)>1.
\end{equation}

Under our choice of $(\bar{\alpha}, \bar{\beta},\bar{T}_h)$ and $\bar{c}$, clearly $\mathcal{H}_1,\mathcal{H}_2$ are nonempty. To obtain the compactness of $\mathcal{H}_1$, it is sufficient to show $\mathcal{H}_1$ is bounded because $\Lambda(\alpha, \bar{\beta},\bar{T}_h)$ is continuous in $\alpha$ over $\mathbb{R}_+$. The boundedness is further guaranteed by $\lim_{\alpha \rightarrow \infty} \Lambda(\alpha,\bar{\beta},\bar{T}_h)=+\infty$. Regarding $\mathcal{H}_2$, we first show it is bounded. If this is not true, there exists a sequence $\{(\beta_k, T_{h,k})\} \subset \mathcal{H}_2$ and one of the following three cases has to hold: (1) $\beta_k\rightarrow \infty, T_{h,k} \rightarrow c_0 <\infty$; (2) $\beta_k\rightarrow \infty, T_{h,k} \rightarrow \infty$; (3) $\beta_k \rightarrow c_0< \infty, T_{h,k} \rightarrow \infty$. Assuming case (1) holds, then 
\begin{equation*}
    \varlimsup_{k\rightarrow \infty}\Lambda(\bar{\alpha}, \beta_k,T_{h,k})
    \leq \frac{c_0 \|\xv\|_2^2}{2\bar{\alpha}p}+\lim_{k\rightarrow
    \infty}\bigg(\sqrt{\bar{\alpha}^2\delta+\delta^2\sigma_z^2}\beta_k-\frac{\delta
    \beta_k^2}{2}\bigg)=-\infty,
\end{equation*}
    contradicting $\inf_k \Lambda(\bar{\alpha}, \beta_k,T_{h,k}) \geq \bar{c}$. For the other two cases, the same contradiction can be drawn based on the Cauchy-Schwarz inequality $\|\uv\|_{\lambdav}\leq \|\lambdav\|_2\|\uv\|_2$ and the following decomposition:
    \begin{align*}
       & \frac{\|\xv\|^2_2}{p} - \bigg\|\eta\bigg(\xv+\frac{\alpha \beta
        \hv}{T_h};\frac{\alpha \gamma}{T_h}\bigg)\bigg\|_{\mathcal{L}_2}^2 \\
        =&
        \underbrace{\bigg\|\eta\bigg(\xv+\frac{\alpha \beta \hv}{T_h}; \frac{\alpha \gamma}{T_h}\bigg) - \xv-\frac{\alpha \beta \hv} {T_h} \bigg\|_{\mathcal{L}_2}^2}_{O(T_h^{-2})} - \frac{\alpha^2\beta^2}{T_h^2}+\frac{2\alpha \gamma \mathbb{E}\| \eta(\xv+  \frac{\alpha \beta \hv} {T_h}; \frac{\alpha \gamma}{T_h})\|_{\lambdav}} {p T_h},
    \end{align*}
    where we have used Lemma \ref{slope:prop} (\ref{lemma:item:prox-identity1}) and (\ref{lemma:item:prox-lip2}) (setting $\gamma_2 = 0$ therein). Now given that $\mathcal{H}_2$ is bounded and $\Lambda(\bar{\alpha}, \beta, T_h)$ is continuous in $(\beta, T_h)$ over $[0,\infty)\times (0,\infty)$, if $\mathcal{H}_2$ is not compact, there must exist a sequence $\{(\beta_k, T_{h,k})\}\subset \mathcal{H}_2$, such that $T_{h,k} \rightarrow 0$ as $k\rightarrow \infty$. In this case, if $\beta_k \rightarrow c_0>0$ then
    \begin{equation*}
        \varlimsup_{k\rightarrow \infty}\Lambda(\bar{\alpha},\beta_k,T_{h,k})
        \leq c_0\sqrt{\bar{\alpha}^2\delta + \sigma_z^2\delta^2} - \lim_{k\rightarrow 0}\frac{1}{2\bar{\alpha}pT_{h,k}}\mathbb{E}\|\eta(T_{h,k}\xv+\bar{\alpha}\beta_k\hv;\bar{\alpha}\gamma)\|_2^2=-\infty.
    \end{equation*}
    If $\beta_k \rightarrow 0$, then 
    \begin{equation*}
        \varlimsup_{k\rightarrow \infty}\Lambda(\bar{\alpha},\beta_k,T_{h,k})
        \leq
        \lim_{k\rightarrow \infty} \bigg[\sqrt{\bar{\alpha}^2\delta+\delta^2\sigma_z^2} \beta_k-\frac{\delta}{2}\beta_k^2-\frac{\bar{\alpha}T_{h,k}}{2}+\frac{T_{h,k} \|\xv\|_2^2} {2\bar{\alpha}p}\bigg]=0.
    \end{equation*}
    Both contradict with the fact that $\inf_k \Lambda(\bar{\alpha},\beta_k,T_{h,k})\geq \bar{c}>1$. This completes our proof of Part (\ref{lemma:item:saddle-points}).\\

    \textbf{\emph{Part (\ref{lemma:item:solu-analysis})}}:
    We proceed by analyzing the first order conditions of $\Lambda$ w.r.t. $\alpha$, $\beta$ and $T_h$ respectively. Lemma \ref{lemma:prox-square-partial} enables us to obtain the following equations for $\alpha>0,\beta\geq 0,T_h>0$:
    \begin{align}
        \frac{\partial \Lambda}{\partial \alpha}
        =&
        -\frac{T_h}{2}+\frac{\alpha \beta \delta}{\sqrt{\alpha^2 \delta + \delta^2 \sigma_z^2}} - \frac{T_h}{2\alpha^2 p} \mathbb{E} \bigg\|\eta\bigg(\xv + \frac{\alpha \beta} {T_h}\hv;\frac{\alpha \gamma}{T_h}\bigg)-\xv\bigg\|_2^2, \label{eq:partial-lambda-alpha} \\ \frac{\partial \Lambda}{\partial \beta}
        =&
        \sqrt{(\alpha)^2\delta+\delta^2\sigma_z^2}-\delta \beta - \frac{1}{p}\mathbb{E} \bigg \langle \eta\bigg(\xv+\frac{\alpha \beta}{T_h}\hv;\frac{\alpha \gamma} {T_h} \bigg), \hv \bigg\rangle, \label{eq:partial-lambda-beta} \\
        \frac{\partial \Lambda}{\partial T_h}
        =&
        -\frac{\alpha}{2}+\frac{1}{2\alpha p} \mathbb{E} \bigg\|\eta \bigg(\xv+\frac{\alpha \beta}{T_h}\hv; \frac{\alpha \gamma}{T_h}\bigg)-\xv\bigg\|_2^2. \label{eq:partial-lambda-th}
    \end{align}

    We first prove $\alpha^* > 0$ by contradiction. Suppose $\alpha^* = 0$. From \eqref{eq:Gamma}, we know that $\beta^* = \sigma_z$ and $T_h^* > 0$ and $\Lambda(0, \sigma_z, T_h^*) = \frac{\delta\sigma_z^2}{2} + \frac{\gamma\|\xv\|_{\lambdav}}{p}$. However, based on \eqref{eq:partial-lambda-alpha}, we know that $\frac{\partial\Lambda}{\partial\alpha}\big|_{\beta=\beta^*,T_h =T_h^*}<0$ when $\alpha>0$ is sufficiently small. This combined with the fact that $\Lambda(\alpha,\beta^*,T_h^*)$ is continuous at $\alpha=0$ implies that $\Lambda(\bar{\alpha}, \sigma_z, T_h^*) < \Lambda(0, \sigma_z, T_h^*)$ for some small enough $\bar{\alpha} > 0$ which contradicts with the fact that $(0, \sigma_z, T_h^*)$ is a saddle point.

    Now for $\alpha^* > 0$, we want to prove $\beta^*, T_h^* > 0$. Referring to the extended function $\Lambda$ in Part (\ref{lemma:item:saddle-points}), it is obvious that $(\beta^*, T_h^*) \notin (0,\infty) \times \{0\}$ since $\Lambda(\alpha^*, 0, 0) = 0$. Further for any given $\bar{T}_h > 0$, \eqref{eq:partial-lambda-beta} reveals that $\frac{\partial \Lambda}{\partial\beta}\big|_{\alpha=\alpha^*, T_h=\bar{T}_h}  > 0$ when $\beta$ is small enough, which implies $(\beta^*, T_h^*) \notin \{0\} \times (0,\infty)$. Hence if we show $(\beta^*, T_h^*) \neq (0, 0)$, then can claim $\beta^*>0,T_h^*>0$. Towards this goal, since $\alpha^*, \gamma > 0$, we can set $T_h$ small enough such that $\mathbb{E}\big\|\eta(\xv + \frac{\alpha^*} {\sqrt{\delta}} \hv; \frac{\alpha^*\gamma}{T_h})\big\|_2^2 < \frac{\|\xv\|_2^2}{2}$, and $\beta = \frac{T_h}{\sqrt{\delta}}$. We are thus able to obtain for sufficiently small $T_h>0$:
    \begin{equation*}
        \Lambda\bigg(\alpha^*, \frac{T_h}{\sqrt{\delta}}, T_h\bigg)
        > \bigg(\sqrt{(\alpha^*)^2 + \delta\sigma_z^2} - \frac{\alpha^*}{2} + \frac{\|\xv\|_2^2}{4\alpha^* p} \bigg) T_h - \frac{T_h^2}{2}
         > 0=\Lambda(\alpha^*, 0, 0).
    \end{equation*}
    This indicates that $(\alpha^*, 0, 0)$ is not the optima when $\alpha^* > 0$.

\textbf{\emph{Part (\ref{lemma:item:se})}}: For any saddle point $(\alpha^*,\beta^*,T_h^*)$, our results in part (\ref{lemma:item:solu-analysis}) make sure they are interior points of the domain. As a result we have $\frac{\partial \Lambda}{\partial \alpha}(\alpha^*, \beta^*, T_h^*) = 0$, $\frac{\partial \Lambda}{\partial \beta}(\alpha^*, \beta^*, T_h^*) = 0$, $\frac{\partial \Lambda}{\partial T_h}(\alpha^*, \beta^*, T_h^*) = 0$. By further making use of \eqref{eq:partial-lambda-alpha}, \eqref{eq:partial-lambda-beta}, \eqref{eq:partial-lambda-th}, it is straightforward to confirm that these first order condition equations can be simplified to \eqref{threeeq:mark}. The equivalence between the three-equation system \eqref{threeeq:mark} and the two-equation system \eqref{twoeq:mark} can be directly verified.
\end{proof}

\subsubsection{Concentration of $\hat{\Lambda}$ around $\Lambda$} \label{sssec:proof-concen-uniform}

\begin{lemma} \label{lemma:concentration-sub-results}              
    Recall that $\hv \sim \mathcal{N}(0, \Iv_p), \gv \sim \mathcal{N}(0, \Iv_n)$, and $\zv \in \mathbb{R}^n$ is the noise vector in the model satisfying Assumption \ref{assum:noise}. Let $c, C>0$ denote some absolute constants, which may vary from place to place. We have the following concentration results:
    \begin{enumerate}[(i)]
            \item \label{item:concen-cross}
            We have that
            \begin{equation} \label{eq:sup-concen-signal-cross}
                \mathbb{P}\Big(\frac{1}{p}|\langle \hv, \xv \rangle| > t\Big) \leq 2 e^{-\frac{cpt^2}{\|\xv\|_2^2 / p}} \leq 2e^{\frac{-cpt^2}{\kappa_6^2}} \quad \forall t \geq 0.
            \end{equation}
        \item \label{item:concen-norm}
            For any given $U_{\alpha}>0$, it holds that $\forall t\geq 0$
            \begin{equation} \label{eq:sup-concen-gauss-norm}
                \mathbb{P}\Big(\sup_{0\leq \alpha \leq U_\alpha} \Big|\frac{1}{p} \|\sqrt{p}\alpha \gv - \sqrt{n} \zv\|_2 - \sqrt{\delta\alpha^2 + \delta^2\sigma_z^2}\Big| > t \Big) \leq 2e^{\frac{-cpt^2}{(1+\kappa_4)^4(U_{\alpha}^2+\delta \sigma_z^2)}}.
            \end{equation}
        \item \label{item:concen-prox-norm2}
           Denote $g(\hv,\alpha,\beta,T_h)=\frac{T_h}{2\alpha p} (\|\xv\|_2^2-\|\eta(\xv + \frac{\alpha \beta}{T_h}\hv; \frac{\alpha \gamma}{T_h})\|_2^2)$. For any given $L_{\beta},U_{\beta},U>0$, define $\mathit{H}=\{(\alpha,\beta,T_h): L_{\beta}\leq \beta \leq U_{\beta}, 0\leq \alpha \beta/T_h\leq U\}$. It holds that $\forall t\geq 0$,
    \begin{align}
        &\mathbb{P}\bigg(\sup_{(\alpha,\beta,T_h)\in \mathit{H}}|g(\hv,\alpha,\beta,T_h)-\mathbb{E}g(\hv,\alpha,\beta,T_h)|>t\bigg) \nonumber \\ 
        \leq  &2e^{\frac{-cpt^2}{U^2_{\beta}\kappa_6^2}}+2e^{-cp\cdot \min\big(\frac{t^2}{U^2U_{\beta}^2},\frac{t}{UU_{\beta}}\big)}+ 2e^{\frac{-cL_{\beta}^2pt^2}{(L_{\beta}+\gamma)^2U^2\gamma^2\log p}} \label{eq:sup-concen-prox2}
    \end{align}
    \item \label{item:uniform:preconcen}  %\label{lemma:obj-concen-scale}
    For given $U_{\alpha},L_{\beta},U_{\beta},U>0$, define the set $\mathit{K}=\{(\alpha, \beta, T_h): 0\leq \alpha \leq U_{\alpha}, L_{\beta}\leq \beta \leq U_{\beta}, 0\leq \alpha \beta/T_h\leq U\}$. We have that $\forall t\geq 0$,
    \begin{align} \label{eq:obj-concen1}
       & \mathbb{P}\bigg(\sup_{(\alpha, \beta, T_h) \in \mathit{K}} \big|\hat{\Lambda}(\alpha,\beta,T_h) - \Lambda(\alpha,\beta,T_h)\big|> t \bigg)  \\
     \leq &2e^{\frac{-cpt^2}{(1+\kappa_4)^4(U^2_{\alpha}+\delta \sigma_z^2)U_{\beta}^2}}+4e^{\frac{-cpt^2}{U^2_{\beta}\kappa_6^2}}+2e^{-cp\cdot \min\big(\frac{t^2}{U^2U_{\beta}^2},\frac{t}{UU_{\beta}}\big)}+ 2e^{\frac{-cL_{\beta}^2pt^2}{(L_{\beta}+\gamma)^2U^2\gamma^2\log p}}:=\mathit{Q}(t; U_{\alpha},L_{\beta},U_{\beta},U) \nonumber
    \end{align}
    \end{enumerate}
\end{lemma}
\begin{proof}
    \textbf{Proof of (\ref{item:concen-cross}).} We note that $\langle\hv, \xv\rangle$ is Lipschitz in $\hv$. The result then follows by applying the Gaussian concentration result (Theorem \ref{thm:lip-gauss}).

      \textbf{Proof of (\ref{item:concen-norm}).} We aim to apply the matrix deviation inequality (Theorem \ref{thm:matdev}). Denote 
      \begin{align}
      \label{construct}
      \Av=
      \begin{pmatrix}
      g_1 & \sigma_z^{-1}z_1 \\
      \vdots & \vdots \\
      g_n & \sigma_z^{-1}z_n
      \end{pmatrix}
      \in \mathbb{R}^{n\times 2}
     % , \quad \xv=
      %\begin{pmatrix}
      %\frac{\alpha}{\sqrt{p}}\\
      %\frac{-\sqrt{n}\sigma_z}{p}
      %\end{pmatrix}
       ,\quad \mathit{T}=\big\{\xv\in \mathbb{R}^2: 0\leq x_1 \leq U_{\alpha}/\sqrt{p}, x_2=-\sqrt{n}\sigma_z/p\big\}.
      \end{align}
      It is straightforward to confirm that 
      \[
      \sup_{0\leq \alpha \leq U_\alpha} \Big|\frac{1}{p} \|\sqrt{p}\alpha \gv - \sqrt{n} \zv\|_2 - \sqrt{\delta\alpha^2 + \delta^2\sigma_z^2}\Big|=\sup_{\xv \in \mathit{T}}\Big|\|\Av \xv\|_2-\sqrt{n}\|\xv\|_2\Big|.
      \]
      Since $\gv\sim \mathcal{N}(0,\Iv_n)$ is independent from $\zv$ and $z_i$'s are sub-Gaussian under Assumption \ref{assum:noise}, the rows $\Av_i$ of the constructed $\Av$ in \eqref{construct} are independent, isotropic and sub-Gaussian with $\max_i\|\Av_i\|_{\psi_2}\leq C(1+\kappa_4)$. Hence according to Theorem \ref{thm:matdev}, $\forall u\geq 0$, with probability at least $1-2e^{-u^2}$ it holds that
      \begin{align}
      \label{eqqe:one}
      \sup_{0\leq \alpha \leq U_\alpha} \Big|\frac{1}{p} \|\sqrt{p}\alpha \gv - \sqrt{n} \zv\|_2 - \sqrt{\delta\alpha^2 + \delta^2\sigma_z^2}\Big|\leq C(1+\kappa_4)^2(w(\mathit{T})+u\cdot {\rm rad}(\mathit{T})).
      \end{align}
      Moreover, for the $\mathit{T}$ in \eqref{construct}, it is direct to bound $w(\mathit{T}), {\rm rad}(\mathit{T})$ as follows:
      \begin{align}
      \label{eqqe:two}
      w(\mathit{T})\leq \frac{U_{\alpha}}{\sqrt{p}}\mathbb{E}|g_1|\leq \frac{U_{\alpha}}{\sqrt{p}}, ~~{\rm rad}(\mathit{T})\leq \sqrt{\frac{pU^2_{\alpha}+n\sigma_z^2}{p^2}}
      \end{align}
      Putting together \eqref{eqqe:one} and \eqref{eqqe:two} proves the result in \eqref{eq:sup-concen-gauss-norm} with $c=\frac{\log 2}{4C^2}$. 
      
    \textbf{Proof of (\ref{item:concen-prox-norm2}).} In the proof of Lemma \ref{saddle:point:thm} Part (\ref{lemma:item:lambda-convex-concave}), we have obtained
    \begin{align*}
    g(\hv,\alpha,\beta,T_h)=-\frac{\beta}{p}\langle \xv, \hv\rangle-\frac{\alpha \beta^2}{2T_h p}\|\hv\|_2^2-\underbrace{\frac{\gamma}{p}\min_{\sv \in \mathcal{D}_1}\Big\{-\langle \xv, \sv \rangle +\frac{\alpha \gamma}{2T_h}\|\sv\|_2^2-\frac{\alpha \beta}{T_h}\langle \hv, \sv \rangle \Big\}}_{:=\tilde{g}(\hv,\alpha,\beta,T_h)}.
    \end{align*}
    Therefore, we have
    \begin{align}
   & \sup_{(\alpha,\beta,T_h)\in \mathit{H}}|g(\hv,\alpha,\beta,T_h)-\mathbb{E}g(\hv,\alpha,\beta,T_h)|\leq \frac{U_{\beta}}{p}|\langle \xv, \hv \rangle|+ \frac{UU_{\beta}}{2p}\big|\|\hv\|_2^2-p\big|+ \nonumber \\
    &\hspace{3cm}\underbrace{\sup_{(\alpha,\beta,T_h)\in \mathit{H}}|\tilde{g}(\hv,\alpha,\beta,T_h)-\mathbb{E}\tilde{g}(\hv,\alpha,\beta,T_h)|}_{:=\bar{g}(\hv)}. \label{piece:1}
    \end{align}
    The concentration of the first term in the above bound has been derived in Part (\ref{item:concen-cross}). Regarding the second term, we apply Bernstein's inequality (Theorem \ref{thm:concen-bernstein}) to derive
    \begin{align*}
 %   \label{piece:2}
    \mathbb{P}\Big(\frac{1}{p}\big|\|\hv\|_2^2-p\big|> t\Big)\leq 2\exp\big(-Cp\min(t^2,t)\big),~~~\forall t\geq 0.
    \end{align*}
    We now focus on bounding the third term. For any $\hv, \tilde{\hv}\in \mathbb{R}^p$, it is direct to verify that 
    \begin{align*}
    |\bar{g}(\hv)-\bar{g}(\tilde{\hv})|\leq \sup_{(\alpha,\beta,T_h)\in \mathit{H}}|\tilde{g}(\hv,\alpha,\beta,T_h)-\tilde{g}(\tilde{\hv},\alpha,\beta,T_h)|\leq \frac{\gamma U \|\lambdav\|_2}{p}\|\hv-\tilde{\hv}\|_2,
    \end{align*}
    thus $\bar{g}(\cdot)$ is a Lipschitz function with Lipschitz constant $\|\bar{g}\|_{\mathrm{Lip}}\leq \frac{\gamma U }{\sqrt{p}}$ by Assumption \ref{assum:weight}. We can then use the Gaussian concentration result (Theorem \ref{thm:lip-gauss}) to obtain
    \begin{align}
    \label{supeq:one}
    \mathbb{P}(|\bar{g}(\hv)-\mathbb{E}\bar{g}(\hv)|> t)\leq 2\exp\big(-C\gamma^{-2}U^{-2}pt^2\big), ~~~\forall t\geq 0.
   \end{align}
   Next we bound $\mathbb{E}\bar{g}(\hv)$. Since $\mathit{H}\subseteq \{(\alpha,\beta,T_h): 0\leq \alpha/T_h\leq U/L_{\beta}, 0\leq \alpha \beta/T_h\leq U\}$, it is clear that there exists a $\epsilon$-net $\mathit{H}_{\epsilon}\subseteq \mathit{H}$ such that $|\mathit{H}_{\epsilon}|\leq p$, and $\forall (\alpha,\beta, T_h)\in \mathit{H}, \exists (\alpha',\beta', T'_h)\in \mathit{H}_{\epsilon}, {\rm s.t.} |\alpha/T_h-\alpha'/T_h'|\leq UL^{-1}_{\beta}p^{-1/2}, |\alpha\beta/T_h-\alpha'\beta'/T_h'|\leq Up^{-1/2}$. Hence, 
   \begin{align*}
&   |\tilde{g}(\hv,\alpha,\beta,T_h)-\tilde{g}(\hv,\alpha',\beta',T'_h)|\leq \frac{\gamma}{p}\Big(\frac{\gamma U \|\lambdav\|_2^2}{2L_{\beta}\sqrt{p}}+\frac{U\|\lambdav\|_2\|\hv\|_2}{\sqrt{p}}\Big)\leq \frac{\gamma^2U}{2L_{\beta}\sqrt{p}}+\frac{\gamma U \|\hv\|_2}{p}, \\
&  |\mathbb{E}\tilde{g}(\hv,\alpha,\beta,T_h)-\mathbb{E}\tilde{g}(\hv,\alpha',\beta',T'_h)|\leq \frac{\gamma^2U}{2L_{\beta}\sqrt{p}}+\frac{\gamma U }{\sqrt{p}}.
   \end{align*}
   The above results further imply that
   \begin{align}
\mathbb{E}\bar{g}(\hv)&\leq \frac{\gamma^2 U}{L_{\beta}\sqrt{p}}+\frac{2\gamma U}{\sqrt{p}}+\mathbb{E}\sup_{(\alpha,\beta,T_h)\in \mathit{H}_{\epsilon}}|\tilde{g}(\hv,\alpha,\beta,T_h)-\mathbb{E}\tilde{g}(\hv,\alpha,\beta,T_h)| \nonumber \\
&\leq \frac{\gamma^2 U}{L_{\beta}\sqrt{p}}+\frac{2\gamma U}{\sqrt{p}}+ \frac{C\sqrt{\log p}\gamma U}{\sqrt{p}}, \label{supeq:two}
   \end{align}
   where in the last inequality we have used the fact that $\tilde{g}(\cdot,\alpha,\beta,T_h)$ is Lipschitz with constant $\frac{\gamma U}{\sqrt{p}}$ so that $\|\tilde{g}(\hv,\alpha,\beta,T_h)-\mathbb{E}\tilde{g}(\hv,\alpha,\beta,T_h)\|_{\psi_2}\leq \frac{C\gamma U}{\sqrt{p}} $. Combining \eqref{supeq:one} and \eqref{supeq:two} with some straightforward calculations yields the following tail bound for $\bar{g}(\hv)$,
   \begin{align*}
   %\label{piece:3}
   \mathbb{P}(\bar{g}(\hv)>t)\leq 2\exp\bigg(\frac{-CL_{\beta}^2pt^2}{(L_{\beta}+\gamma)^2U^2\gamma^2\log p}\bigg), ~~~\forall t\geq 0.
   \end{align*}
   Finally, putting together the concentration results we have derived for the three terms in \eqref{piece:1} completes the proof. 
   
   \textbf{Proof of (\ref{item:uniform:preconcen}).} Adopt the notation from Part (\ref{item:concen-prox-norm2}). We first have that
    \begin{align*}
      \hat{\Lambda}(\alpha,\beta, T_h) - \Lambda(\alpha,\beta, T_h)=&\underbrace{ \frac{\hv^\top \xv}{p}}_{:=\mathcal{J}_1}\beta+
        \Big(\underbrace{\frac{\|\sqrt{p}\alpha \gv - \sqrt{n} \zv\|_2}{p} - \sqrt{\alpha^2\delta+\delta^2\sigma_z^2}}_{:=\mathcal{J}_2}\Big) \beta \\
        &+ \underbrace{g(\hv,\alpha,\beta, T_h)-\mathbb{E}g(\hv,\alpha,\beta, T_h)}_{:=\mathcal{J}_3}.
    \end{align*}
    This leads to the union bound,
    \begin{align*}
     &  \mathbb{P}\bigg(\sup_{(\alpha, \beta, T_h) \in \mathit{K}} \big|\hat{\Lambda}(\alpha,\beta,T_h) - \Lambda(\alpha,\beta,T_h)\big|> t \bigg) \\
        \leq &\mathbb{P}\bigg( |\mathcal{J}_1| > \frac{ t}{3U_{\beta}}\bigg) + \mathbb{P}\bigg(\sup_{0 < \alpha \leq U_\alpha} |\mathcal{J}_2| > \frac{t}{3U_{\beta}}\bigg) + \mathbb{P}\bigg( \sup_{(\alpha, \beta, T_h) \in \mathit{H}}|\mathcal{J}_3| > \frac{t}{3}\bigg).
    \end{align*}

    The result then follows from Parts (\ref{item:concen-cross})-(\ref{item:concen-prox-norm2}).
\end{proof}

%\label{lemma:sup-prox-square-concentration1}

\subsubsection{A master theorem} \label{master:for:all:case}

We prove a master theorem in this section and then use it to derive the results of Theorem \ref{thm:concentration0} in the next section. Recall several notations: $(\sigma^*,\chi^*)$ is the solution satisfying \eqref{eq:state-evolution} and \eqref{eq:calibration}; $(\alpha^*,\beta^*,T_h^*)$ is the saddle point of $\min_{\alpha \geq 0}\max_{\beta \geq 0,T_h > 0} \Lambda(\alpha,\beta,T_h)$; from Lemma \ref{saddle:point:thm} Part \eqref{lemma:item:se} we know that $\sigma^*=\frac{\alpha^*\beta^*}{T_h^*}, \chi^*=\frac{\gamma}{\beta^*}, \sqrt{m_n}=\alpha^*$ where $m_n$ is introduced in Section \ref{sssec:proof-concen-sketch}.

\begin{theorem}\label{master:thm}
There exist positive constants $\{C_i\}_{i=1}^4$ only possibly depending on the $\kappa_i$'s in Assumptions \ref{assum:noise}-\ref{assum:bounded-signal} such that the following holds
%\begin{align*}
%  \mathbb{P}\Bigg(\Big| \frac{1}{\sqrt{p}}\|\hat{\xv}(\gamma) - \xv \|_2
 %       - \alpha^* \Big| > t \Bigg)\leq 4Q\bigg(\frac{t^2\gamma \sigma_z^2}{\chi^*(\sigma^*)^3}; 3\alpha^*,\gamma/(2\chi^*),2\gamma/\chi^*,3\sigma^*\bigg),
%\end{align*}
%where the function $Q$ is defined in Lemma \ref{lemma:concentration-sub-results} Part \eqref{item:uniform:preconcen}, and 
\begin{align*}
 & \mathbb{P}\Bigg(\Big| \frac{1}{\sqrt{p}}\|\hat{\xv}(\gamma) - \xv \|_2
        - \alpha^* \Big| > t \Bigg) \\
        \leq &C_1\exp\Big(-C_2p\min\Big\{\frac{t^4\sigma_z^4}{(\sigma^*)^6+(1+\delta)(\sigma^*)^8+(\sigma^*)^8(\chi^*)^2(1+\chi^*)^2\log p},\frac{t^2\sigma_z^2}{(\sigma^*)^4}\Big\}\Big),
\end{align*}
where $t$ can be any non-negative constant that satisfies
%\begin{align*}
%0\leq t^2 \leq (\alpha^*)^2 \wedge \frac{C_1(\gamma+\gamma(\alpha^*)^2(\sigma^*)^{-2})(\delta(\alpha^*)^2+\delta^2\sigma_z^2)^{3/2}}{\sigma_z^2\delta^2(\chi^*+C_2(\alpha^*)^2(\delta \gamma+\sigma^*\chi^*)(\sigma^*)^{-3}\|\eta(\hv; 2\chi^*)\|_{\mathcal{L}_2}^{-2})}.
%\end{align*} 
\begin{align*}
0\leq t^2 \leq (\alpha^*)^2 \wedge\Bigg(\frac{C_3(\alpha^*)^2\delta \gamma \sigma^*\sigma_z^{-2}}{\chi^*(1+C_4\delta(1+\delta)\|\eta(\hv; 2\chi^*)\|_{\mathcal{L}_2}^{-2})}\Bigg).
\end{align*}
\end{theorem}

\begin{proof} 
As described in Section \ref{sssec:proof-concen-sketch}, we aim to show $\frac{1}{p}\min_{\wv \in S_w}F_n(\wv)< \frac{1}{p} \min_{\wv \in S_w \cap H_{t}} F_n(\wv)$ w.h.p. We first derive an upper bound for $\frac{1}{p}\min_{\wv \in S_w}F_n(\wv)$, and then a lower bound for $\frac{1}{p} \min_{\wv \in S_w \cap H_{t}} F_n(\wv)$.   
\paragraph{The upper bound:}
According to Lemma \ref{lemma:cgmt-upper-lower}, it is sufficient to upper bound 
\begin{align*}
%\label{upper:start:1}
\min_{0\leq \alpha \leq 2\sqrt{m_n}+t} \max_{\beta \geq 0, T_h>0} \hat{\Lambda}(\alpha,\beta,T_h)=\min_{0\leq \alpha \leq 2\alpha^*+t} \max_{\beta \geq 0, T_h>0} \hat{\Lambda}(\alpha,\beta,T_h)\leq  \max_{\beta \geq 0, T_h>0} \hat{\Lambda}(\alpha^*,\beta,T_h).
\end{align*}
Define $\mathit{J}:=\{(\beta,T_h): \frac{1}{2}\beta^*\leq \beta \leq \frac{3}{2}\beta^*, \frac{1}{2}T_h^* \leq T_h \leq \frac{3}{2}T_h^*\}$, and $(\hat{\beta},\hat{T}_h)=\argmax_{(\beta,T_h)\in \mathit{J}}\hat{\Lambda}(\alpha^*,\beta,T_h)$. Note that since $(\alpha^*,\beta^*,T_h^*)$ is the saddle point, we know $(\beta^*,T^*_h)=\argmax_{(\beta,T_h)\in \mathit{J}}\Lambda(\alpha^*,\beta,T_h)$. Hence,
\begin{align*}
&\frac{\min_{(\beta,T_h)\in \mathit{J}}\lambda_{\min}(\beta,T_h)}{2}((\hat{\beta}-\beta^*)^2+(\hat{T}_h-T_h^*)^2)\leq \Lambda(\alpha^*,\beta^*,T^*_h)-\Lambda(\alpha^*,\hat{\beta},\hat{T_h})\leq \\
& \Lambda(\alpha^*,\beta^*,T^*_h) -\hat{\Lambda}(\alpha^*,\beta^*,T^*_h)+\hat{\Lambda}(\alpha^*,\hat{\beta},\hat{T}_h)-\Lambda(\alpha^*,\hat{\beta},\hat{T}_h) \leq 2\sup_{(\beta,T_h)\in \mathit{J}}|\hat{\Lambda}(\alpha^*,\beta,T_h)-\Lambda(\alpha^*,\beta,T_h)|,
\end{align*}
where $\bar{\mu}:=\min_{(\beta,T_h)\in \mathit{J}}\lambda_{\min}(\beta,T_h)$ denotes the minimum smallest eigenvalue of the negative Hessian matrix of $\Lambda(\alpha^*,\beta,T_h)$ w.r.t. $(\beta,T_h)$ over $\mathit{J}$. This result combined with Lemma \ref{lemma:concentration-sub-results} Part \eqref{item:uniform:preconcen} shows that for $0\leq \Delta\leq C\bar{\mu}((\beta^*)^2+(T_h^*)^2)$ where $C>0$ is a small absolute constant, the following hold with probability at least $1-\mathit{Q}(\Delta; \alpha^*,\beta^*/2,3\beta^*/2, 3\alpha^*\beta^*/T_h^*)$:
\begin{itemize}
\item[(a)] $\sup_{(\beta,T_h)\in \mathit{J}}|\hat{\Lambda}(\alpha^*,\beta,T_h)-\Lambda(\alpha^*,\beta,T_h)|\leq \Delta$.
\item[(b)] $(\hat{\beta},\hat{T}_h)$ is an interior point of $\mathit{J}$ hence it is a local maximizer of $\hat{\Lambda}(\alpha^*,\beta,T_h)$ over $[0,\infty)\times (0,\infty)$.
\end{itemize}
Given that $\hat{\Lambda}(\alpha^*,\beta,T_h)$ is concave in $(\beta, T_h)$ as can be verified using the same argument in the proof of Lemma \ref{saddle:point:thm} Part \eqref{lemma:item:lambda-convex-concave}, (b) implies $(\hat{\beta},\hat{T}_h)$ is in fact a global maximizer so that $\max_{\beta \geq 0, T_h>0} \hat{\Lambda}(\alpha^*,\beta,T_h)=\max_{(\beta, T_h)\in \mathit{J}} \hat{\Lambda}(\alpha^*,\beta,T_h)$. This further enables us to obtain the upper bound 
\begin{align*}
\max_{\beta \geq 0, T_h>0} \hat{\Lambda}(\alpha^*,\beta,T_h)=& \hat{\Lambda}(\alpha^*,\hat{\beta},\hat{T}_h)=(\hat{\Lambda}(\alpha^*,\hat{\beta},\hat{T}_h)-\Lambda(\alpha^*,\hat{\beta},\hat{T}_h))+\\
&(\Lambda(\alpha^*,\hat{\beta},\hat{T}_h)-\Lambda(\alpha^*,\beta^*,T^*_h))+ \Lambda(\alpha^*,\beta^*,T^*_h)\\
\leq &\sup_{(\beta,T_h)\in \mathit{J}}|\hat{\Lambda}(\alpha^*,\beta,T_h)-\Lambda(\alpha^*,\beta,T_h)|+\Lambda(\alpha^*,\beta^*,T^*_h)\leq \Delta+ \Lambda(\alpha^*,\beta^*,T^*_h)
\end{align*}
Above all, we have proved that for $0\leq \Delta \leq C\bar{\mu}((\beta^*)^2+(T_h^*)^2)$, 
\begin{align*}
%\label{collect:upper}
\mathbb{P}\Big(\frac{1}{p}\min_{\wv \in S_w}F_n(\wv)\leq \Delta+ \Lambda(\alpha^*,\beta^*,T^*_h) \Big)\geq 1-2\mathit{Q}(\Delta; \alpha^*,\beta^*/2,3\beta^*/2, 3\alpha^*\beta^*/T_h^*).
\end{align*}
\paragraph{The lower bound:}
Again by Lemma \ref{lemma:cgmt-upper-lower}, we aim to lower bound 
\begin{align*}
\min_{\substack{0\leq \alpha \leq 2\sqrt{m_n}+t \\ |\alpha-\sqrt{m_n}|\geq t}} \max_{\beta \geq 0, T_h>0} \hat{\Lambda}(\alpha,\beta,T_h)\geq  \min_{\substack{0\leq \alpha \leq 2\sqrt{m_n}+t \\ |\alpha-\sqrt{m_n}|\geq t}} \hat{\Lambda}(\alpha,\beta^*,T^*_h).
\end{align*}
Denote $\hat{\alpha}=\argmin_{\substack{0\leq \alpha \leq 2\sqrt{m_n}+t \\ |\alpha-\sqrt{m_n}|\geq t}}\hat{\Lambda}(\alpha,\beta^*,T^*_h), \bar{\alpha}=\argmin_{\substack{0\leq \alpha \leq 2\sqrt{m_n}+t \\ |\alpha-\sqrt{m_n}|\geq t}}\Lambda(\alpha,\beta^*,T^*_h)$. Using Lemma \ref{lemma:concentration-sub-results} Part \eqref{item:uniform:preconcen} we can have 
\begin{align*}
\hat{\Lambda}(\hat{\alpha},\beta^*,T^*_h)&=(\hat{\Lambda}(\hat{\alpha},\beta^*,T^*_h)-\Lambda(\hat{\alpha},\beta^*,T^*_h))+(\Lambda(\hat{\alpha},\beta^*,T^*_h)-\Lambda(\bar{\alpha},\beta^*,T^*_h))+\Lambda(\bar{\alpha},\beta^*,T^*_h) \\
&\geq -\sup_{0\leq \alpha \leq 2\alpha^*+t}|\hat{\Lambda}(\alpha,\beta^*,T^*_h)-\Lambda(\alpha,\beta^*,T^*_h)|+ \Lambda(\bar{\alpha},\beta^*,T^*_h)\geq -\tilde{\Delta}+\Lambda(\bar{\alpha},\beta^*,T^*_h)
\end{align*}
hold with probability at least $1-\mathit{Q}(\tilde{\Delta};2\alpha^*+t,\beta^*,\beta^*, (2\alpha^*+t)\beta^*/T_h^*)$ for any $\tilde{\Delta}\geq 0$. This implies 
\begin{align*}
%\label{collect:lower}
&\mathbb{P}\Big(\frac{1}{p} \min_{\wv \in S_w \cap H_{t}} F_n(\wv)\geq -\tilde{\Delta}+\Lambda(\bar{\alpha},\beta^*,T^*_h) \Big)  \\
\geq &1- 2\mathit{Q}(\tilde{\Delta};2\alpha^*+t,\beta^*,\beta^*, (2\alpha^*+t)\beta^*/T_h^*), ~~ ~~\forall \tilde{Q}\geq 0.
\end{align*}
Moreover, since $\alpha^*$ is the global minimizer of $\Lambda(\alpha,\beta^*,T_h^*)$, it is clear that 
\begin{align*}
\Lambda(\bar{\alpha},\beta^*,T_h^*)-\Lambda(\alpha,\beta^*,T_h^*)\geq \frac{1}{2}(t\wedge \alpha^*)^2\cdot \min_{|\alpha-\alpha^*|\leq t}\frac{d^2 (\Lambda(\alpha,\beta^*,T_h^*))}{d\alpha^2}\geq \frac{(t\wedge \alpha^*)^2\gamma \sigma_z^2\delta^3}{ 2\chi^*(\delta(\alpha^*+t)^2 + \delta^2 \sigma_z^2)^{\frac{3}{2}}}:=\mathit{P}_t
\end{align*}
where the last inequality is due to Lemma \ref{lemma:bounds-alpha-second}. Therefore, setting $\Delta=\tilde{\Delta}=\frac{\mathit{P}_t}{3}$ gives us
\begin{align}
&\mathbb{P}\Big(\min_{\wv \in S_w}F_n(\wv)<\min_{\wv \in S_w \cap H_{t}} F_n(\wv)\Big)\nonumber \\
\geq &\mathbb{P}\Big(\frac{1}{p}\min_{\wv \in S_w}F_n(\wv)\leq \Delta+ \Lambda(\alpha^*,\beta^*,T^*_h)\Big)-\mathbb{P}\Big(\frac{1}{p} \min_{\wv \in S_w \cap H_{t}} F_n(\wv)\leq  -\tilde{\Delta}+\Lambda(\bar{\alpha},\beta^*,T^*_h) \Big) \nonumber \\
\geq & 1-2\mathit{Q}(\mathit{P}_t/3; \alpha^*,\beta^*/2,3\beta^*/2, 3\alpha^*\beta^*/T_h^*)-2\mathit{Q}(\mathit{P}_t/3;2\alpha^*+t,\beta^*,\beta^*, (2\alpha^*+t)\beta^*/T_h^*),\label{final:bound:p}
\end{align}
as long as $\frac{\mathit{P}_t}{3} \leq C\bar{\mu}((\beta^*)^2+(T_h^*)^2)$ which will be satisfied by plugging in the lower bound for $\bar{\mu}$ derived in Lemmas \ref{min:eigenH:1} and $\gamma\leq \sigma^*\chi^*$ from Lemma \ref{lemma:solution-analysis}. Finally, the identity $(\alpha^*)^2=\delta((\sigma^*)^2-\sigma_z^2)$ and the monotonic dependency of the function $Q$ on its arguments enable the simplification of the lower bound in \eqref{final:bound:p}.
\end{proof}

\begin{lemma} \label{lemma:bounds-alpha-second}
    We have the following bound:
    \begin{equation*}
       \frac{d^2 (\Lambda(\alpha,\beta^*,T_h^*))}{d\alpha^2}
        \geq
        \frac{\gamma \sigma_z^2\delta^3}{ \chi^*(\delta\alpha^2 + \delta^2 \sigma_z^2)^{\frac{3}{2}}}.
    \end{equation*}
\end{lemma}
\begin{proof}
  Recall the notation $\mathcal{P}_0$ defined after \eqref{eq:tie-set}. We use it here to denote the partitions with respect to $\eta(\xv+\frac{\alpha\beta^*}{T^*_h};\frac{\alpha\gamma}{T^*_h})$. Using \eqref{eq:partial-lambda-alpha} and Lemma \ref{lemma:prox-mse-partial}, it is not hard to verify that
    \begin{align*}
       \frac{d^2 (\Lambda(\alpha,\beta^*,T_h^*))}{d\alpha^2}
        =& \frac{\beta^*\sigma_z^2\delta^3}{ (\delta\alpha^2 + \delta^2 \sigma_z^2)^{\frac{3}{2}}} + \frac{T_h^*}{\alpha^3 p} \bigg(\|\xv\|_2^2 - \mathbb{E} \sum_{\mathcal{I} \in \mathcal{P}_0} \frac{(\sum_{k\in\mathcal{I}} x_k \cdot \sign(x_k + h_k\alpha\beta^*/T_h^* ))^2} {|\mathcal{I}|} \bigg) \\
        \geq & \frac{\beta^*\sigma_z^2\delta^3}{ (\delta\alpha^2 + \delta^2 \sigma_z^2)^{\frac{3}{2}}} + \frac{T_h^*}{\alpha^3 p}\Big(\|\xv\|_2^2-\sum_{\mathcal{I} \in \mathcal{P}_0}\sum_{k\in\mathcal{I}}x_k^2 \Big) \geq  \frac{\beta^*\sigma_z^2\delta^3}{ (\delta\alpha^2 + \delta^2 \sigma_z^2)^{\frac{3}{2}}}.
    \end{align*}
\end{proof}

\begin{lemma}\label{min:eigenH:1}
Consider the function $\Lambda(\alpha^*,\beta,T_h)$ on the region $\mathit{J}:=\{(\beta,T_h): \frac{1}{2}\beta^*\leq \beta \leq \frac{3}{2}\beta^*, \frac{1}{2}T_h^* \leq T_h \leq \frac{3}{2}T_h^*\}$. Let $\lambda_{\min}(\beta,T_h)$ denote the smallest eigenvalue of the negative Hessian matrix of $\Lambda(\alpha^*,\beta,T_h)$ w.r.t. $(\beta,T_h)$. It holds that
\begin{align*}
\max_{(\beta,T_h)\in \mathit{J}}\frac{1}{\lambda_{\min}(\beta, T_h)}\leq \frac{1}{\delta}+\frac{C(\alpha^*)^2(\delta \gamma+\sigma^*\chi^*)}{\delta\chi^*(\sigma^*)^3\|\eta(\hv; 2\chi^*)\|_{\mathcal{L}_2}^2},
\end{align*}
where $C>0$ is some absolute constant. 
\end{lemma}

\begin{proof}
Recall the notation $\mathcal{P}_0$ defined after \eqref{eq:tie-set}. We use it here to refer to the partitions with respect to $\eta(\xv+\frac{\alpha^*\beta}{T_h};\frac{\alpha^*\gamma}{T_h})$. Define $M_{hh} = \frac{1}{p} \mathbb{E} \sum_{\mathcal{I} \in \mathcal{P}_0} \frac{( \sum_{j \in \mathcal{I}}h_js_j)^2}{|\mathcal{I}|}$, $M_{h\lambda} = \frac{1}{p} \mathbb{E} \sum_{\mathcal{I} \in \mathcal{P}_0} \frac{( \sum_{j \in \mathcal{I}}h_js_j)( \sum_{j \in \mathcal{I}}\lambda_{r_j})} {|\mathcal{I}|}$, $M_{\lambda\lambda} = \frac{1}{p} \mathbb{E} \sum_{\mathcal{I} \in \mathcal{P}_0} \frac{( \sum_{j \in \mathcal{I}} \lambda_{r_j})^2} {|\mathcal{I}|}$, where $s_j = \sign(x_j + \frac{\alpha^*\beta}{T_h}h_j)$ and $r_j$ is the rank of $|x_j + \frac{\alpha^*\beta}{T_h}h_j|$ in the sequence $\{|x_j + \frac{\alpha^*\beta}{T_h}h_j|\}_{j=1}^p$. Based on \eqref{eq:partial-lambda-beta}, \eqref{eq:partial-lambda-th} and Lemma \ref{differential:prop} Part \eqref{lemma:item:prox-magic}, with some calculations we can represent the second order derivatives of $\Lambda(\alpha^*,\beta, T_h)$ w.r.t. $(\beta, T_h)$ as 
    \begin{align*}
      &  \frac{\partial^2\Lambda}{\partial\beta^2} = -\delta - \frac{\alpha^*}{T_h} M_{hh},
        \quad
        \frac{\partial^2\Lambda}{\partial\beta\partial T_h} = \frac{\alpha^*\beta}{T_h^2}M_{hh} - \frac{\alpha^*\gamma}{T_h^2} M_{h\lambda}, \\
      &  \frac{\partial^2\Lambda}{\partial T_h^2} = - \frac{\alpha^*\beta^2}{T^3_h} (M_{hh} - 2\gamma \beta^{-1}M_{h\lambda} + \gamma^2\beta^{-2} M_{\lambda\lambda}).
    \end{align*}
    Therefore the determinant and the trace of the negative Hessian take the following forms:
    \begin{align*}
        \det =& \frac{\delta\alpha^* \beta^2}{T^3_h}(M_{hh} - 2\gamma \beta^{-1}M_{h\lambda} + \gamma^2\beta^{-2} M_{\lambda\lambda}) + \frac{(\alpha^*)^2\gamma^2} {T_h^4} (M_{hh} M_{\lambda\lambda} - M_{h\lambda}^2), \nonumber \\
        \mathrm{trace} =& \delta + \frac{\alpha^*}{T_h}M_{hh} + \frac{\alpha^*\beta^2}{T^3_h}(M_{hh} - 2\gamma \beta^{-1}M_{h\lambda} + \gamma^2\beta^{-2} M_{\lambda\lambda}).
    \end{align*}
    Furthermore, referring to Lemma \ref{pp1}, we note that $M_{hh} - 2\gamma \beta^{-1}M_{h\lambda} + \gamma^2\beta^{-2} M_{\lambda\lambda}$ is the derivative of $\|\eta(\xv + \sigma\hv; \sigma\gamma \beta^{-1}) - \xv\|_{\mathcal{L}_2}^2$ w.r.t. $\sigma^2$ evaluated at $\sigma=\frac{\alpha^*\beta}{T_h}$ and hence the following relation holds:
    \begin{align*}
        M_{hh} - 2\gamma \beta^{-1}M_{h\lambda} + \gamma^2\beta^{-2} M_{\lambda\lambda}
        \geq \lim_{\sigma \rightarrow \infty} \frac{\partial \|\eta(\xv + \sigma\hv; \sigma\gamma \beta^{-1}) - \xv\|_{\mathcal{L}_2}^2}{\partial(\sigma^2)}= \|\eta(\hv; \gamma \beta^{-1})\|_{\mathcal{L}_2}^2.
    \end{align*}
    An upper bound for $\frac{1}{\lambda_{\min}(\beta,T_h)}$ can be obtained by 
    \begin{equation*}
        \frac{1}{\lambda_{\min}(\beta,T_h)} \leq \frac{\mathrm{trace}}{\det}
        \leq \frac{1}{\delta} + \frac{\delta T_h^3+ \alpha^*T_h^2} {\delta \alpha^*\beta^2 \|\eta(\hv; \gamma \beta^{-1})\|_{\mathcal{L}_2}^2}
    \end{equation*}
    Finally, given that $(\beta, T_h)\in \mathit{J}$ and the identities $\sigma^*=\frac{\alpha^*\beta^*}{T_h^*}, \chi^*=\frac{\gamma}{\beta^*}$, the claimed result can be obtained from the above bound. 

\end{proof}

\subsubsection{Proof of different scenarios in Theorem \ref{thm:concentration0}} \label{sssec:proof-mse-concen1}

We are in the position to prove the three scenarios in Theorem \ref{thm:concentration0}. The idea is to first derive bounds for $\sigma^*, \chi^*$ under different scenarios, and then apply the master theorem (Theorem \ref{master:thm}) with these bounds to obtain more specific concentration result for each case. Recall the following key quantity in Theorem \ref{thm:concentration0}:
\[
 M_{\lambdav}(\chi^*)= \lim_{\sigma \rightarrow 0} \frac{1}{p} \mathbb{E}\|\eta(\xv/\sigma + \hv; \chi^*) -\xv/\sigma\|_2^2,~~{\rm where}~\hv\sim \mathcal{N}(0,\Iv_p).
\]

\vspace{0.2cm}

\begin{lemma} \label{lemma:solution-analysis}
    Below we summarize the bounds in different cases:
    \begin{enumerate} [(i)]
         \item \label{common:upper:bound}
        Two useful common bounds: (1) $\gamma \leq \sigma^*\chi^*;$ (2) There exists an absolute constant $c>0$ such that $\frac{1}{p}\mathbb{E}\|\eta(\hv;t)\|_2^2\geq ce^{-t^2},~\forall t\geq 0$.
        \item \label{item:delta-leq1-above-pt}
             If $M_{\lambdav}(\chi^*) < \delta$, then we have
            \begin{align*}
             (\sigma^*)^2 \leq \frac{\delta\sigma_z^2}{\delta - M_{\lambdav}(\chi^*)}, \quad \chi^* \leq \sqrt{\frac{\delta - \epsilon}{\epsilon}}\frac{\sqrt{p}}{\|\lambdav\|_2}.
        \end{align*}
      %  Finally, in this case we must have $\gamma \leq \sqrt{\frac{\delta - \epsilon}{\epsilon}} \sqrt{\frac{\delta}{\delta - M_{\lambdav}(\chi^*)}} \frac{\sqrt{p}}{\|\lambdav\|_2} \sigma_z$.
        
        \item \label{item:delta-leq1-below-pt}
         If $M_{\lambdav}(\chi^*) > \delta$, let $\sigma_0$ be the value that satisfies $\delta \sigma_0^2 = \frac{1}{p} \mathbb{E}\big\|\eta\big(\xv + \sigma_0\hv; \sigma_0 \chi^*\big) - \xv \big\|_2^2$, 
         and $b_0 = \frac{\partial}{\partial\sigma^2}\frac{1}{p} \mathbb{E}\big\|\eta\big(\xv + \sigma\hv; \sigma\chi^*\big) - \xv \big\|_2^2 \big|_{\sigma = \sigma_0} $. Then it holds that
            \begin{align*}
             (\sigma^*)^2 \leq \sigma_0^2 + \frac{\delta \sigma_z^2}{\delta - b_0}, \quad \chi^* \leq \sqrt{\frac{M_{\lambdav}(\chi^*) - \epsilon}{\epsilon}}\frac{\sqrt{p}}{\|\lambdav\|_2}.
        \end{align*}
        \item \label{item:delta-geq1-large-noise}
        If $\sigma_z > \frac{\sqrt{2(\delta + 1)} \|\xv\|_2}{\delta \sqrt{p}}$, and $\frac{\gamma}{\sigma_z} > \frac{3}{\|\lambdav\|_2^2 / p} \sqrt{0 \vee \log\frac{16\delta + 8}{\delta^2}}$, then we have 
        \begin{align*} %\label{eq:alpha-large-noise-lower}
           (\sigma^*)^2
        \leq 2\sigma_z^2, \quad   \chi^* \leq  \frac{(2\delta +2)\gamma}{\delta \sigma_z}.
        \end{align*}
           \end{enumerate}
\end{lemma}
\begin{proof}
    From Lemma \ref{saddle:point:thm} Part \eqref{lemma:item:se}, we first restate the equations that $\alpha^*, \sigma^*, \chi^*$ should satisfy:
    \begin{align}
    (\alpha^*)^2&=\delta((\sigma^*)^2-\sigma_z^2) \label{eq:eqsys-alpha-opt}\\
        \delta - \frac{\delta\sigma_z^2}{(\sigma^*)^2} &= \frac{1}{p} \mathbb{E}\big\|\eta\big(\frac{\xv}{\sigma^*} + \hv; \chi^*\big) -  \frac{\xv}{\sigma^*} \big\|_2^2, \label{eq:eqsys-mse-opt} \\
        \frac{\delta\gamma}{\sigma^*\chi^*} &= \delta - \frac{1}{p} \mathbb{E}\big\langle \eta\big(\frac{\xv}{\sigma^*} + \hv; \chi^*\big), \hv \big\rangle, \label{eq:eqsys-cali-opt}
    \end{align}
    These equations will be repeatedly used in the proof of this lemma. 
    
     \textbf{Proof of (\ref{common:upper:bound}).}
    Lemma \ref{slope:prop} Part \eqref{lemma:item:prox-nonexpansive0} implies that $\mathbb{E}\langle \eta(\frac{\xv}{\sigma^*} + \hv; \chi^*), \hv \rangle = \mathbb{E}\langle \eta(\frac{\xv}{\sigma^*} + \hv; \chi^*) - \eta(\frac{\xv}{\sigma^*}; \chi^*), \hv \rangle \geq \mathbb{E}\| \eta(\frac{\xv}{\sigma^*} + \hv; \chi^*) - \eta(\frac{\xv}{\sigma^*}; \chi^*) \|_2^2 \geq 0$, which together with \eqref{eq:eqsys-cali-opt} proves the first bound. Regarding the second one, from Lemma \ref{slope:prop} Part \eqref{item:basic-prox-bound} we have $\frac{1}{p}\mathbb{E}\|\eta(\hv;t)\|_2^2\geq \mathbb{E}(|z|-\lambda_1t)^2_+\geq \mathbb{E}(|z|-t)^2_+$ where $z\sim \mathcal{N}(0,1)$. Then applying Lemma \ref{lemma:numeric_bound} Part \eqref{item:gaussian-ineq1} completes the proof. 
    
    \textbf{Proof of (\ref{item:delta-leq1-above-pt}).}
  According to Lemma \ref{pp1} Part \eqref{item:ff-decrease},  $\mathbb{E}\big\|\eta\big(\frac{\xv}{\sqrt{v}} + \hv; \chi^*\big) -  \frac{\xv}{\sqrt{v}} \big\|_2^2$ is a decreasing function of $v$ over $(0,\infty)$. Hence we can use \eqref{eq:eqsys-mse-opt} to obtain
  \begin{align*}
 \frac{1}{p} \mathbb{E}\|\eta(\hv; \chi^*)\|_2^2 &=\lim_{v\rightarrow \infty} \frac{1}{p} \mathbb{E}\big\|\eta\big(\frac{\xv}{\sqrt{v}} + \hv; \chi^*\big) -  \frac{\xv}{\sqrt{v}} \big\|_2^2\\
  &\leq  \delta - \frac{\delta\sigma_z^2}{(\sigma^*)^2} \leq  \lim_{v\rightarrow 0}\frac{1}{p} \mathbb{E}\big\|\eta\big(\frac{\xv}{\sqrt{v}} + \hv; \chi^*\big) -  \frac{\xv}{\sqrt{v}} \big\|_2^2=M_{\lambdav}(\chi^*),
  \end{align*}
  which yields the bounds for $\sigma^*$. Moreover, by Lemma \ref{lemma:bounding-mse-limit}, we have the following upper bound on $\chi^*$:
    \begin{align}
    \label{same:work}
        \delta \geq M_{\lambdav}(\chi^*)
        \geq \epsilon + \frac{(\chi^*)^2}{p}\sum_{i=1}^k \lambda_i^2,
        \quad \Rightarrow \quad
        \chi^* \leq \sqrt{\frac{\delta - \epsilon}{\frac{1}{p}\|\lambdav_{[1:k]}\|_2^2}}
        \leq \sqrt{\frac{\delta - \epsilon}{\epsilon}}\frac{\sqrt{p}}{\|\lambdav\|_2}.
    \end{align}

    \textbf{Proof of (\ref{item:delta-leq1-below-pt}).}
  Denote $f(v)=\frac{1}{p} \mathbb{E}\big\|\eta\big(\xv + \sqrt{v}\hv; \sqrt{v} \chi^*\big) - \xv \big\|_2^2$, and $g(v)=f(v)+\delta \sigma_z^2$. The tangent line of $f(v)$ at $v=\sigma_0^2$ is $b_0(v-\sigma_0^2)+\delta \sigma_0^2$, hence the tangent line of $g(v)$ at $v=\sigma_0^2$ is $h(v):=b_0(v-\sigma_0^2)+\delta \sigma_0^2+\delta \sigma_z^2$. Since $h(v)$ is above $g(v)$ we have that 
  \begin{align*}
  \delta (\sigma^*)^2 \leq h((\sigma^*)^2)=b_0((\sigma^*)^2-\sigma_0^2)+\delta \sigma_0^2+\delta \sigma_z^2,
  \end{align*}
  which leads to $(\sigma^*)^2 \leq \sigma_0^2 + \frac{\delta \sigma_z^2}{\delta - b_0}$. In regards to the upper bound of $\chi^*$, it follows from \eqref{same:work}.   

    \textbf{Proof of (\ref{item:delta-geq1-large-noise}).} 
    We first prove the bound for $\chi^*$. From Lemma \ref{pp1} Part \eqref{item:concave-mse} and \eqref{eq:eqsys-cali-opt}, we have
    \begin{align*}
    (\sigma^*)^2=\sigma_z^2+\frac{1}{\delta p}\mathbb{E}\big\|\eta\big(\xv + \sigma^* \hv; \sigma^* \chi^*\big) -\xv \big\|_2^2 \leq \sigma_z^2 + \frac{\|\xv\|_2^2}{\delta p} + \frac{ (\sigma^*)^2}{\delta p} \mathbb{E}\|\eta(\hv; \chi^*)\|_2^2,
    \end{align*}
     which by the identity $(\alpha^*)^2=\delta((\sigma^*)^2-\sigma_z^2)$ leads to the following results:
    \begin{align}
      \label{alpha:related:bounds}
       (\sigma^*)^2
        \leq \frac{\delta\sigma_z^2 + \frac{1}{p}\|\xv\|_2^2}{\delta - \frac{1}{p} \mathbb{E} \|\eta(\hv; \chi^*)\|_2^2}, \quad (\alpha^*)^2
        \leq \frac{\frac{\delta}{p} \mathbb{E} \|\eta(\hv; \chi^*)\|_2^2\sigma_z^2 + \frac{\delta}{p}\|\xv\|_2^2}{\delta - \frac{1}{p} \mathbb{E} \|\eta(\hv; \chi^*)\|_2^2}.
    \end{align}
    For any constant $c > 1 \vee \sqrt{\delta}$, suppose if $\chi^* \geq \frac{c\gamma}{(c-1)\sigma^*}$, then using \eqref{eq:eqsys-alpha-opt}-\eqref{eq:eqsys-cali-opt} we have that
    \begin{align*}
        &\frac{\delta \sigma^*}{c \alpha^*} \Big\|\eta\Big(\frac{\xv}{\sigma^*} + \hv; \chi^*\Big) - \frac{\xv}{\sigma^*}\Big\|_{\mathcal{L}_2}
        = \frac{\delta}{c} \\
        \leq & \frac{1}{p}\mathbb{E}\Big \langle \eta\Big(\frac{\xv}{\sigma^*} + \hv; \chi^*\Big), \hv \Big \rangle =\frac{1}{p}\mathbb{E}\Big \langle \eta\Big(\frac{\xv}{\sigma^*} + \hv; \chi^*\Big)-\frac{\xv}{\sigma^*}, \hv \Big \rangle \\
        \leq & \Big\|\eta\Big(\frac{\xv}{\sigma^*} + \hv; \chi^*\Big) - \frac{\xv}{\sigma^*}\Big\|_{\mathcal{L}_2},
    \end{align*}
    where the last inequality is due to Cauchy–Schwarz inequality. The above result leads to
    \begin{equation}
    \label{interm:alpha}
     \frac{\delta \sigma^*}{c \alpha^*} \leq 1  \quad \Rightarrow \quad   (\alpha^*)^2 \geq \frac{\delta^2\sigma_z^2}{c^2 - \delta}.
    \end{equation}
    Moreover, it is straightforward to confirm when $\sigma_z^2 > \frac{(c^2 - \delta) \frac{\|\xv\|_2^2}{p}}{\delta^2 - \frac{c^2}{p}\mathbb{E}\|\eta(\hv; \chi^*)\|_2^2}$ and $\delta^2 > \frac{c^2}{p}\mathbb{E}\|\eta(\hv; \chi^*)\|_2^2$, it holds that $\frac{\delta^2\sigma_z^2}{c^2 - \delta}  > \frac{\frac{\delta}{p}\mathbb{E}\|\eta(\hv; \chi^*)\|_2^2 \sigma_z^2 + \frac{\delta}{p}\|\xv\|_2^2}{\delta - \frac{1}{p}\mathbb{E}\|\eta(\hv; \chi^*)\|_2^2}$. However, this result together with \eqref{interm:alpha} contradicts with the upper bound for $\alpha^*$ derived in \eqref{alpha:related:bounds}. Therefore, we can conclude the following bound $\chi^* \leq  \frac{c\gamma}{(c-1)\sigma^*}$, as long as $\sigma_z^2 > \frac{(c^2 - \delta) \frac{\|\xv\|_2^2}{p}}{\delta^2 - \frac{c^2}{p}\mathbb{E}\|\eta(\hv; \chi^*)\|_2^2}$ and $\delta^2 > \frac{c^2}{p}\mathbb{E}\|\eta(\hv; \chi^*)\|_2^2$. Based on the condition $\sigma_z > \frac{\sqrt{2(\delta + 1)}\|\xv\|_2}{\delta \sqrt{p}}$ in (\ref{item:delta-geq1-large-noise}), these two requirements will be satisfied by setting $c = \sqrt{2\delta + 1}$ if $\|\eta(\hv; \chi^*)\|_{\mathcal{L}_2}^2 < \frac{\delta^2}{4 \delta + 2}$. According to Lemma \ref{slope:prop} (\ref{item:basic-prox-bound}) and Lemma \ref{lemma:numeric_bound}, it is direct to check that $\|\eta(\hv; \chi^*)\|_{\mathcal{L}_2}^2 < \frac{\delta^2}{4 \delta + 2}$ is implied by the condition $\frac{\gamma}{\sigma_z} \geq \frac{3}{\|\lambdav\|_2^2 / p}\sqrt{0 \vee \log\frac{16\delta + 8}{\delta^2}}$ in (\ref{item:delta-geq1-large-noise}), whenever $\chi^* \geq \frac{\gamma}{2\sigma_z}$. 
    
It remains to prove $\chi^* \geq \frac{\gamma}{2\sigma_z}$. In particular, we will show that it holds whenever $\sigma_z \geq \frac{\sqrt{\delta + 1} \|\xv\|_2}{\delta \sqrt{p}}$. Towards this end, suppose $\sigma_z \geq \frac{\|\xv\|_2}{C_1\sqrt{p}}$ for now. Using \eqref{eq:eqsys-mse-opt} we obtain
    \begin{align*}
        \frac{\delta\sigma_z}{\sigma^*}
        \geq \frac{\delta\sigma_z^2}{(\sigma^*)^2}
        =& \delta - \|\hv - \projv_{\mathcal{D}_{\chi^*}}(\xv / \sigma^* + \hv)\|_{\mathcal{L}_2}^2 \nonumber \\
        =& \delta - 1 + \frac{1}{p}\mathbb{E}\langle \hv, \projv_{\mathcal{D}_{\chi^*}}(\xv / \sigma^* + \hv) \rangle + \frac{1}{p}\mathbb{E}\langle \xv / \sigma^* + \hv, \projv_{\mathcal{D}_{\chi^*}}(\xv / \sigma^* + \hv) \rangle \nonumber \\
        &- \|\projv_{\mathcal{D}_{\chi^*}}(\xv / \sigma^* + \hv)\|_{\mathcal{L}_2}^2 - \frac{1}{p} \mathbb{E} \langle \xv / \sigma^*, \projv_{\mathcal{D}_{\chi^*}}(\xv / \sigma^* + \hv) \rangle \nonumber \\
        \overset{(a)}{\geq}& \delta - 1 + \frac{1}{p}\mathbb{E}\langle \hv, \projv_{\mathcal{D}_{\chi^*}}(\xv / \sigma^* + \hv) \rangle - \frac{1}{p}\mathbb{E}\langle \xv / \sigma^*, \projv_{\mathcal{D}_{\chi^*}}(\xv / \sigma^* + \hv) \rangle \nonumber \\
        \overset{(b)}{\geq}& \delta - 1 + \frac{1}{p}\mathbb{E}\langle \hv, \projv_{\mathcal{D}_{\chi^*}}(\xv / \sigma^* + \hv) \rangle - \frac{\|\xv\|_2}{\sigma^*\sqrt{p}}\sqrt{1 + \frac{\|\xv\|_2^2}{(\sigma^*)^2 p}} \nonumber \\
        \overset{(c)}{\geq}& \delta - 1 + \frac{1}{p}\mathbb{E}\langle \hv, \projv_{\mathcal{D}_{\chi^*}}(\xv / \sigma^* + \hv) \rangle - \frac{\|\xv\|_2}{\sigma^*\sqrt{p}}\sqrt{1 + C_1^2},
    \end{align*}
    where step (a) holds since the projection set $\mathcal{D}_{\chi^*}$ is a closed convex set containing the origin; (b) is by the Cauchy-Schwarz inequality; and (c) is due to the fact $\sigma^*\geq \sigma_z$. This implies that
    \begin{equation*}
        \frac{\delta \sigma_z + \sqrt{1 + C_1^2}\|\xv\|_2 / \sqrt{p}}{\sigma^*}
        \geq \delta - 1 + \frac{1}{p}\mathbb{E}\langle \hv, \projv_{\mathcal{D}_{\chi^*}}(\xv / \sigma^* + \hv) \rangle.
    \end{equation*}
    The result combined with \eqref{eq:eqsys-cali-opt} yields
    \begin{equation*}
        \delta\gamma
        = \sigma^*\chi^*\big(\delta - 1 + \frac{1}{p}\mathbb{E}\langle \hv, \projv_{\mathcal{D}_{\chi^*}}(\xv / \sigma^* + \hv) \rangle\big)
        \leq (\delta \sigma_z + \sqrt{1 + C_1^2} \|\xv\|_2 / \sqrt{p}) \chi^*,
    \end{equation*}
    which shows that
    \begin{equation*}
        \chi^*
        \geq \frac{\delta \gamma}{\delta \sigma_z + \sqrt{1 + C_1^2}\|\xv\|_2 / \sqrt{p}}
        \geq \frac{\delta}{\delta + C_1\sqrt{1 + C_1^2}} \frac{\gamma}{\sigma_z}.
    \end{equation*}
    By setting $C_1^2 \leq \frac{1}{2}(\sqrt{4\delta^2 + 1} - 1)$, we have $C_1\sqrt{1 + C_1^2} \leq \delta$, and hence $\chi^* \geq \frac{\gamma}{2\sigma_z}$. One feasible choice is $C_1 = \frac{\delta}{\sqrt{\delta + 1}}$. In summary, so far we have proved that $\chi^*\leq \frac{\sqrt{2\delta +1}\gamma}{(\sqrt{2\delta+1}-1)\sigma^*}$. Moreover, it can be easily verified that $\frac{\sqrt{2\delta +1}}{(\sqrt{2\delta+1}-1)}\leq \frac{2\delta +2}{\delta}$ and hence $\chi^*\leq \frac{(2\delta+2)\gamma}{\delta \sigma^*}\leq \frac{(2\delta+2)\gamma}{\delta \sigma_z}$. Regarding the upper bound for $\sigma^*$, we have showed in the preceding arguments that
    \[
    (\sigma^*)^2
        \leq \frac{\delta\sigma_z^2 + \frac{1}{p}\|\xv\|_2^2}{\delta - \frac{1}{p} \mathbb{E} \|\eta(\hv; \chi^*)\|_2^2},~~ ~~\frac{1}{p}\|\eta(\hv; \chi^*)\|_2^2 \leq \frac{\delta^2}{4 \delta + 2}
    \]
    These results combined with the condition $\sigma_z > \frac{\sqrt{2(\delta + 1)} \|\xv\|_2}{\delta \sqrt{p}}$ yield
    \begin{align*}
    (\sigma^*)^2\leq \frac{\delta \sigma_z^2 +\sigma_z^2\frac{\delta^2}{2\delta+2}}{\delta-\frac{\delta^2}{4\delta+2}}=\frac{4\delta +2}{2\delta+2}\sigma_z^2\leq 2\sigma_z^2.
    \end{align*}

\end{proof}

\vspace{0.5cm}

Finally, with some straightforward calculations, combining Theorem \ref{master:thm} with the bounds in Lemma \ref{lemma:solution-analysis} Parts \eqref{common:upper:bound}\eqref{item:delta-leq1-above-pt} completes the proof of Theorem \ref{thm:concentration0} (\ref{item:main-low-noise-abovePT0}); combining Theorem \ref{master:thm} with the bounds in Lemma \ref{lemma:solution-analysis} Parts \eqref{common:upper:bound}\eqref{item:delta-leq1-below-pt} proves Theorem \ref{thm:concentration0} (\ref{item:main-low-noise-belowPT}); combining Theorem \ref{master:thm} with the bounds in Lemma \ref{lemma:solution-analysis} Parts \eqref{common:upper:bound}\eqref{item:delta-geq1-large-noise} finishes the proof of Theorem \ref{thm:concentration0} (\ref{item:main-large-noise-deltaB1}).

\subsection{Proofs of Proposition \ref{lem:best_slope_small_noise} and Theorems \ref{thm:noiseless-phase-transition} and \ref{thm:large-noise-ridge-better}}
\label{sec:proof:thm-large-noise}

Recall that $e_{\lambdav}(\gamma^*_{\lambdav},\sigma_z)=\frac{1}{p}\mathbb{E}\|\eta(\xv + \sigma^* \hv;\sigma^* \chi^*) - \xv\|_2^2$, where $\gamma_{\lambdav}^* = \argmin_{\gamma > 0} e_{\lambdav}(\gamma, \sigma_z)$ and the pair $(\sigma^*, \chi^*)$ is obtained from the equations
\begin{align}
    (\sigma^*)^2 =& \sigma_z^2 + \frac{1}{\delta p}\mathbb{E}\|\eta(\xv + \sigma^*\hv; \sigma^*\chi^*) - \xv\|^2, \label{eq:state-evolution:again1} \\
    \gamma^*_{\lambdav} =& \sigma^*\chi^* \Big(1 - \frac{1}{\delta \sigma^* p}\mathbb{E} \langle \eta(\xv + \sigma^*\hv; \sigma^*\chi^*), \hv \rangle \Big). \label{eq:state-evolution:again2}
\end{align}

The main proof for Theorems \ref{thm:noiseless-phase-transition} and \ref{thm:large-noise-ridge-better} is to analyze the above state evolution equations as $\sigma_z\rightarrow 0$ or $\sigma_z\rightarrow \infty$. The quantity $\mathbb{E}\|\eta(\xv + \sigma^*\hv; \sigma^*\chi^*) - \xv\|^2$ plays a critical role in the analysis. Lemma \ref{pp1} below characterizes several important properties of this quantity that will be useful in the proof. 

\begin{lemma} \label{pp1}
    For any fixed $\chi > 0$, define the function $f:\mathbb{R}_+ \rightarrow \mathbb{R}_+$,
    \begin{align*} \label{important:risk:fun}
        f(v)=\mathbb{E}\|\eta(\xv+\sqrt{v}\hv; \sqrt{v}\chi)-\xv\|_2^2,
    \end{align*}
    where $\hv \sim \mathcal{N}(\bm{0}, \Iv_{p})$. Then $f(v)$ has the following properties:
    \begin{enumerate}[(i)]
        \item \label{item:ff-continuity} $f(v)$ is continuous at $v=0$ and has derivatives of all orders on $(0, +\infty)$.
        \item \label{item:ff-increase} $f(v)$ is strictly increasing over $[0, +\infty)$
        \item \label{item:ff-decrease} $\frac{f(v)}{v}$ is decreasing over $(0, +\infty)$, and strictly decreasing if $\xv \neq \bm{0}$.
        \item \label{item:concave-mse} $f(v) \leq v\mathbb{E}\|\eta(\hv; \chi)\|_2^2  + \|\xv\|_2^2$.
    \end{enumerate}
\end{lemma}

%\begin{lemma} \label{lemma:concave-mse}
 %   Let $f(v) = \big\|\eta(\xv + \sqrt{v}\hv; \sqrt{v}\chi) - \xv\big\|_2^2$. Then, we have the following upper bound for $f(v)$:
 %   \begin{equation*} 
  %      f(v) \leq \|\eta(\hv; \chi)\|_2^2 v + \|\xv\|_2^2 
  %  \end{equation*}
%\end{lemma}

\begin{proof}
\emph{\textbf{Part (\ref{item:ff-continuity}):}}  Observe that 
\[
f(v)=v\mathbb{E}\|\eta (\xv/\sqrt{v} + \hv;\chi ) -\xv/\sqrt{v}\|_2^2, \quad \mbox{~for~} v>0.
\]
To show $f(v)$ is smooth over $(0,\infty)$, it is sufficient to show  for each $1 \leq i \leq p$, $\mathbb{E} \eta^2_i(\xv/\sqrt{v}+\hv;\chi )$ and $\mathbb{E} x_i \eta_i(\xv/\sqrt{v} + \hv; \chi)$ are both smooth for $v \in (0, +\infty)$. We have
\begin{align*}
\mathbb{E} \eta^2_i(\xv/\sqrt{v}+\hv;\chi )=(2\pi)^{-p/2}\int \eta_i^2 (\hv;\chi) e^{-\frac{\|\hv-\xv/\sqrt{v}\|_2^2}{2}}d\hv.
\end{align*}
Given that $\eta^2_i(\xv/\sqrt{v}+\hv;\chi )\leq 2\|\xv\|_2^2/v + 2\|\hv\|_2^2$, we can apply the mean value theorem and the Dominated Convergence Theorem (DCT) to conclude the existence of derivatives of all orders for $\mathbb{E} \eta^2_i(\xv/\sqrt{v}+\hv;\chi )$. Similar arguments work for $\mathbb{E}x_i \eta_i(\xv/\sqrt{v}+\hv; \chi)$. We next show the continuity of $f(v)$ at $v=0$. From Lemma \ref{lemma:dual-ball-diam} we have
\begin{equation*}
    \sup_{v>0} \mathbb{E}\|\eta (\xv/\sqrt{v} + \hv; \chi ) - \xv/\sqrt{v} \|_2^2
    \leq p + \chi^2\|\lambdav\|_2^2.
\end{equation*}
Hence $|f(v)| \leq (p + \chi^2\|\lambdav\|_2^2)\cdot |v|$, yielding that $\lim_{v\rightarrow 0} f(v) = 0$. \\

\noindent \emph{\textbf{Part (\ref{item:ff-increase}):}}
Recall the notation $\mathcal{I}$, $\mathcal{P}$ and $\mathcal{P}_0$ defined in and after \eqref{eq:tie-set}. Let $r_j$ be the rank of $|x_j + \sqrt{v}h_j|$ in the sequence $\{|x_i + \sqrt{v}h_i|\}_{i=1}^p$. Using the form of $\eta_i$ presented in Lemma \ref{property:primal} Part (\ref{lemma:item:prox-form}), combined with DCT and Lemma \ref{lemma:prox-mse-partial} we can compute the derivative $f'(v)$, 
\begin{align*}
    f'(v)
    =&
    \frac{1}{v}\mathbb{E}\Bigg( \|\eta(\xv + \sqrt{v}\hv; \sqrt{v}\chi)\|_2^2
    - 2\langle \xv, \eta(\xv + \sqrt{v}\hv; \sqrt{v}\chi) \rangle + \sum_{\mathcal{I} \in \mathcal{P}_0} \frac{1}{|\mathcal{I}|}\Big(\sum_{j \in \mathcal{I}}
    x_j\cdot \mbox{sign}(x_j+\sqrt{v}h_j)\Big)^2\Bigg) \nonumber \\
    =&
    \mathbb{E} \sum_{\mathcal{I} \in \mathcal{P}_0} \frac{1}{|\mathcal{I}|}
    \bigg( \sum_{j \in \mathcal{I}} (h_j\cdot \mbox{sign}(x_j+\sqrt{v}h_j)- \chi\lambda_{r_j}) \bigg)^2 > 0.
\end{align*}
Therefore, $f'(v)>0$ for $v\in (0,+\infty)$. Also $f(v)$ is continuous at $v=0$
from Part (\ref{item:ff-continuity}). Thus $f(v)$ is strictly increasing over $[0,+\infty)$. \\

\noindent \emph{\textbf{Part (\ref{item:ff-decrease}):}} Utilizing the result from Part (\ref{item:ff-increase}), we compute the derivative when $v>0$,
\begin{align} \label{negative:one}
    \Big(\frac{f(v)}{v}\Big)'
    =& \frac{f'(v)v-f(v)}{v^2}
    = -\frac{1}{v^2}\mathbb{E} \Bigg[ \|\xv\|_2^2 - \sum_{\mathcal{I} \in \mathcal{P}_0} \frac{1}{|\mathcal{I}|}\Big(\sum_{j \in \mathcal{I}} x_j \cdot \sign(x_j + \sqrt{v} h_j)\Big)^2 \Bigg] \nonumber \\
    =& -\frac{1}{v^2} \mathbb{E}\Bigg[\sum_{\mathcal{I} \in \mathcal{P} \backslash \mathcal{P}_0} \sum_{j \in \mathcal{I}} x_j^2 + \sum_{\mathcal{I} \in \mathcal{P}_0} \Big(\sum_{j \in \mathcal{I}} x_j^2\Big) - \frac{1}{|\mathcal{I}|}\Big(\sum_{j \in \mathcal{I}} x_j \cdot \sign(x_j + \sqrt{v} h_j)\Big)^2 \Bigg]\leq 0,
\end{align}
where the last inequality is due the arithmetic-mean square-mean inequality. We can further argue that the strict inequality holds in \eqref{negative:one} when $\xv\neq \bm{0}$. This is because Lemma \ref{p1} implies that
$\eta(\xv+\sqrt{v}\hv; \sqrt{v}\chi )=0$ if and only if
\begin{equation*}
    \hv \in O_{\xv}\triangleq \Big\{\hv\in\mathbb{R}^p: \sum_{i=1}^j |\xv/\sqrt{v}+\hv|_{(i)}\leq \chi\sum_{i=1}^j \lambda_i, 1\leq j \leq p \Big\}. 
\end{equation*}
The set $O_{\xv}$ is convex and has positive Lebesgue measure. We can then continue from \eqref{negative:one} to obtain
\[
 \Big(\frac{f(v)}{v}\Big)'=\frac{f'(v)v-f(v)}{v^2} \leq -\frac{\|\xv\|_2^2}{v^2}\cdot \mathbb{P}(\hv \in O_{\xv})<0.
\]

\noindent  \emph{\textbf{Part (\ref{item:concave-mse}):}} Let $g(a) =\mathbb{E}\| \eta(a\xv + \hv; \chi)\|_2^2$. We can have
    \begin{align*}
        f(v) - v \mathbb{E}\|\eta(\hv; \chi)\|_2^2 
        =& \frac{\mathbb{E}\|\eta(\frac{\xv}{\sqrt{v}} + \hv; \chi)\|_2^2 - \frac{2}{\sqrt{v}} \mathbb{E}\langle \eta(\frac{\xv}{\sqrt{v}} + \hv; \chi), \xv \rangle - \mathbb{E}\|\eta(\hv; \chi)\|_2^2} {\frac{1}{v}} + \|\xv\|_2^2 \nonumber \\
        =& - \frac{g(0) - g\Big(\frac{1}{\sqrt{v}}\Big) + \frac{1}{\sqrt{v}} g'\Big(\frac{1}{\sqrt{v}}\Big)} {\frac{1}{v}} + \|\xv\|_2^2
        \leq  \|\xv\|^2,
    \end{align*}
    where the second equality is from Lemma \ref{differential:prop} Part \eqref{lemma:item:prox-magic}, and the last inequality is due to the convexity of $g(a)$ from Lemma \ref{slope:prop} Part \eqref{lemma:item:prox-nonincreasing}.
    
\end{proof}

The equations \eqref{eq:state-evolution:again1} and \eqref{eq:state-evolution:again2} that we aim to analyze seem rather complicated, because the regularization parameter $\gamma^*_{\lambdav}$ is chosen to be the optimal one instead of an arbitrarily given value. Lemma \ref{optimal:fixed:point:eq} shows us  that the choice of the optimal tuning simplifies the equations to some extent, and sets the stage for the noise sensitivity analysis.

\begin{lemma} \label{optimal:fixed:point:eq}
If $\sigma^*$ is the unique solution to the equation
\begin{equation} \label{fixed:point:eq:optimal}
    \sigma^2 = \sigma_z^2+\frac{1}{\delta p} \inf_{\chi>0}\mathbb{E}\|\eta(\xv+\sigma h; \sigma \chi)-\xv\|_2^2,
\end{equation}
then we have
\begin{equation} \label{optimal:amse:formula}
    e_{\lambdav}(\gamma^*_{\lambdav},\sigma_z)=\delta((\sigma^*)^2-\sigma_z^2). 
\end{equation}
\end{lemma}

\begin{proof}
We first prove \eqref{fixed:point:eq:optimal} has a unique solution. Denote
\[
G(\sigma)=\frac{\sigma_z^2}{\sigma^2}+\frac{1}{\delta p} \inf_{\chi>0}\mathbb{E}\|\eta(\xv/\sigma+ h;  \chi)-\xv/\sigma\|_2^2.
\]
Then \eqref{fixed:point:eq:optimal} is equivalent to $G(\sigma)=1$. Lemma \ref{pp1} Part (\ref{item:ff-decrease}) shows that $\mathbb{E}\|\eta(\xv/\sigma+ h;  \chi)-\xv/\sigma\|_2^2$ is a decreasing function of $\sigma$ over $(0,\infty)$. As a result, so is $\inf_{\chi>0}\mathbb{E}\|\eta(\xv/\sigma+ h;  \chi)-\xv/\sigma\|_2^2$. Hence $G(\sigma)$ is a continuous and strictly decreasing function for $\sigma \in (0,\infty)$. Moreover,
\begin{align*}
0\leq G(\sigma)\leq \frac{\sigma_z^2}{\sigma^2}+\frac{1}{\delta p}\lim_{\chi\rightarrow \infty}\mathbb{E}\|\eta(\xv/\sigma+ h;  \chi)-\xv/\sigma\|_2^2=\frac{\sigma_z^2}{\sigma^2}+\frac{\|\xv\|_2^2}{\delta p \sigma^2},
\end{align*}
yielding that $\lim_{\sigma\rightarrow \infty}G(\sigma)=0$. It is also clear that $\lim_{\sigma \rightarrow 0}G(\sigma)=+\infty$. Thus, $G(\sigma)=1$ has a unique solution $\sigma=\sigma^*$. It remains to prove \eqref{optimal:amse:formula}. Consider any given $\gamma>0$. We have 
\begin{align*}
    e_{\lambdav}(\gamma,\sigma_z)=\frac{1}{p}\mathbb{E}\|\eta(\xv+\bar{\sigma} \hv; \bar{\sigma} \bar{\chi})-\xv\|_2^2=\delta(\bar{\sigma}^2-\sigma_z^2),
\end{align*}
with $(\bar{\sigma},\bar{\chi})$ being the solution to \eqref{eq:state-evolution} and \eqref{eq:calibration}. Equation \eqref{eq:state-evolution} can be rewritten as 
\[
1= \frac{\sigma_z^2}{\bar{\sigma}^2}+\frac{1}{\delta p}\mathbb{E}\|\eta(\xv/\bar{\sigma}+ h;  \bar{\chi})-\xv/\bar{\sigma}\|_2^2,
\]
with which we obtain
\[
G(\sigma^*)=1\geq  \frac{\sigma_z^2}{\bar{\sigma}^2}+\frac{1}{\delta p} \inf_{\chi >0}\mathbb{E}\|\eta(\xv/\bar{\sigma}+ h; \chi)-\xv/\bar{\sigma}\|_2^2=G(\bar{\sigma}),
\]
which implies that $\sigma^*\leq \bar{\sigma}$ due to the monotonicity of $G(\sigma)$. Hence, 
\[
\delta((\sigma^*)^2-\sigma_z^2) \leq \delta(\bar{\sigma}^2-\sigma_z^2) =e_{\lambdav}(\gamma,\sigma_z), \quad \forall \gamma >0.
\]
Finally, we need show the above lower bound is attained by $e_{\lambdav}(\gamma^*,\sigma_z)$ for some value $\gamma^*$. Define 
\begin{align}
\label{first:pick}
\chi^*=\argmin_{\chi>0} \mathbb{E}\|\eta(\xv/\sigma^*+ h;  \chi)-\xv/\sigma^*\|_2^2.
\end{align}
Note that $\chi^*$ might not be unique and it can be any minimizer. We then pick the following tuning:
\begin{align}
\label{second:pick}
\gamma^*=\chi^*\sigma^*  \big(1 - \frac{1}{\delta p}\mathbb{E}[
    \nabla \cdot \eta(\xv + \sigma^*\hv; \sigma^*\chi^*)]\big).
\end{align}
Based on \eqref{first:pick} and \eqref{second:pick} together with the result $G(\sigma^*)=1$, it is straightforward to verify that 
\[
e_{\lambdav}(\gamma^*,\sigma_z)=\delta((\sigma^*)^2-\sigma_z^2).
\]
\end{proof}

Next we prove Theorem \ref{thm:noiseless-phase-transition} and Proposition \ref{lem:best_slope_small_noise}. Therein we need to first characterize the connection between $M_{\lambdav}$ and \eqref{eq:state-evolution:again1}, of which the proof is delayed to Lemma \ref{pp4} after we finish the main proof.

\subsubsection{Proof of Theorem \ref{thm:noiseless-phase-transition} and Proposition \ref{lem:best_slope_small_noise}} \label{sssec:proof-low-noise}
In this section we prove the results in the low noise scenario.

\begin{proof}[Proof of Theorem \ref{thm:noiseless-phase-transition}]
    Lemma \ref{optimal:fixed:point:eq} proves that $e_{\lambdav}(\gamma^*_{\lambdav},\sigma_z)=\delta((\sigma^*)^2-\sigma_z^2)$ with $\sigma=\sigma^*$ being the solution to the equation
    \begin{align}
    \label{fixed:point:eq:optimal:before}
    \sigma^2=\sigma_z^2+\frac{1}{\delta p} \inf_{\chi>0}\mathbb{E}\|\eta(\xv+\sigma h; \sigma \chi)-\xv\|_2^2.
    \end{align}
    The first part of the proof is to analyze $\sigma^*$ when $\sigma_z\rightarrow 0$.
    
    \begin{enumerate} [(i)]
    \item \emph{The case $\delta<M_{\lambdav}$.}
    We prove that in this case $\lim_{\sigma_z\rightarrow 0} e_{\lambdav} (\gamma^*_{\lambdav}, \sigma_z)>0$. It is equivalent to show $\lim_{\sigma_z\rightarrow 0}\sigma^*>0$. Suppose this is not true. Then from \eqref{fixed:point:eq:optimal:before} we obtain 
    \[
    \frac{1}{p}\inf_{\chi>0} \mathbb{E}\|\xv/\sigma^*+h;\chi)-\xv/\sigma^*\|_2^2<\delta.
    \]
    According to lemma \ref{pp4}, letting $\sigma_z\rightarrow 0$ on both sides of the above inequality yields that $M_{\lambdav}\leq \delta$. This is a contradiction. 
    \item \emph{The case $\delta>M_{\lambdav}$.} Lemma \ref{pp1} Part (\ref{item:ff-decrease}) together with \eqref{fixed:point:eq:optimal:before} gives us that
    \begin{align}
    \frac{(\sigma^*)^2-\sigma_z^2}{(\sigma^*)^2}&=\frac{1}{\delta p}\inf_{\chi>0} \mathbb{E}\|\xv/\sigma^*+\hv;\chi)-\xv/\sigma^*\|_2^2 \label{reformulate:se} \\
    &\leq \frac{1}{\delta p} \lim_{\sigma\rightarrow 0}\inf_{\chi>0} \mathbb{E}\|\xv/\sigma+\hv;\chi)-\xv/\sigma\|_2^2=\frac{M_{\lambdav}}{\delta}, \nonumber
    \end{align}
    where the last equality is due to Lemma \ref{pp4}. Hence,
    \[
    0\leq (\sigma^*)^2\leq \frac{\sigma_z^2}{1-M_{\lambdav} / \delta}\rightarrow 0, \quad \text{as } \sigma_z\rightarrow 0.
    \]
    Now given that $\lim_{\sigma_z\rightarrow 0}\sigma^*=0$, letting $\sigma_z\rightarrow 0$ on both sides of \eqref{reformulate:se} delivers 
    \[
    \lim_{\sigma_z\rightarrow 0}\frac{(\sigma^*)^2}{\sigma_z^2}=\frac{\delta}{\delta - M_{\lambdav}},
    \]
    leading to $\lim_{\sigma_z\rightarrow 0} \frac{e_{\lambdav} (\gamma^*_{\lambdav},\sigma_z)} {\sigma^2_z}
    = \frac{\delta M_{\lambdav}}{\delta -M_{\lambdav}}$.
    \end{enumerate}
\end{proof}

\begin{proof}[Proof of Proposition \ref{lem:best_slope_small_noise}]
    For this part of the proof, we show that the quantity 
    \[
        M_{\lambdav}
        = \inf_{\alpha>0}\bigg\{\underbrace{k + \alpha^2 \sum_{i=1}^k \lambda_i^2+\mathbb{E}\|\eta(\hv_{[k+1:p]}; \alpha, \lambdav_{[k+1:p]})\|_2^2}_{:=h(\lambdav,\alpha)}\bigg\}
    \]
    is minimized when $\lambda_1=\cdots=\lambda_p$. Define the set
    \[
    \mathcal{W}_{\bar{\lambda}}=\{\lambdav\in \mathbb{R}^p: \lambda_1\geq \lambda_2 \geq \cdots \geq \lambda_k\geq \bar{\lambda}\geq \lambda _{k+1}\geq \cdots \geq \lambda_p \geq 0\}.
    \]
    For any $\lambdav\in \mathcal{W}_{\bar{\lambda}}$, it is clear that $\sum_{i=1}^k\lambda_i^2\geq \bar{\lambda}^2$. Moreover, according to Lemma \ref{p1}, 
    \[
    \eta(\hv_{[k+1:p]}; \alpha, \lambdav_{[k+1:p]}) = \hv_{[k+1:p]} - \projv_{\tilde{\mathcal{D}}_\alpha}(\hv_{[k+1:p]}), 
    \]
    where $\tilde{\mathcal{D}}_\alpha \subset \mathbb{R}^{p-k}$ is the dual SLOPE norm ball of radius $\alpha$ with the weight sequence $\lambdav_{[k+1:p]}$. Clearly, among the choices of $\lambdav \in \mathcal{W}_{\bar{\lambda}}$, $\tilde{\mathcal{D}}_{\alpha}$ becomes the largest convex set $\tilde{\mathcal{D}}_\alpha$ when $\lambda_i = \bar{\lambda}$, $i=k+1,\ldots, p$, which in turn implies that the residual norm $\|\eta(\hv_{[k+1:p]}; \alpha; \lambdav_{[k+1:p]})\|_2$ is minimized with the same selection. We therefore have shown that
    \begin{equation*}
        \min_{ \lambdav \in \mathcal{W}_{\bar{\lambda}}}h(\lambdav,\alpha)
        = k + k\alpha^2 \bar{\lambda}^2  + (p-k) \mathbb{E}\eta_{\ell_1}^2(z;\alpha \bar{\lambda}),
    \end{equation*}
    where $\eta_{\ell_1}(z;\alpha\bar{\lambda})=\mbox{sign}(z)(|z|-\alpha \bar{\lambda})_+$ is
    the soft thresholding operator and $z\sim \mathcal{N}(0,1)$. The equation above holds for any $\bar{\lambda} \geq 0$, we thus can conclude that
    \begin{align*}
        \inf_{\lambdav:\lambda_1\geq \cdots \geq \lambda_p \geq 0}M_{\lambdav} &=\inf_{\alpha>0, \bar{\lambda}\geq 0}\inf_{\lambdav \in \mathcal{W}_{\bar{\lambda}}} h(\lambdav,\alpha)
        =\inf_{\alpha>0, \bar{\lambda}\geq 0} \bigg\{k + k\alpha^2 \bar{\lambda}^2  + (p-k) \mathbb{E}\eta_{\ell_1}^2(z;\alpha \bar{\lambda}) \bigg\} \\
       &= \inf_{\alpha>0} \bigg\{k + k\alpha^2  + (p-k) \mathbb{E}\eta_{\ell_1}^2(z;\alpha) \bigg\},
    \end{align*}
    which is precisely the $M_{\lambdav}$ when all the elements of $\lambdav$ are equal. 
\end{proof}

\begin{lemma} \label{pp4}
Suppose $\xv \in \mathbb{R}^p$ does not have non-zero tied components with $\|\xv\|_0 = k$. Then, it holds that
\begin{align*}
  \lim_{v \rightarrow 0} \inf_{\chi >0}\mathbb{E}\|\eta(\xv/\sqrt{v}+\hv;\chi)-\xv/\sqrt{v}\|_2^2=M_{\lambdav}.
\end{align*}
\end{lemma}

\begin{proof}
For any given $v>0$, define the optimal value for $\chi$ as
\begin{equation}\label{op:tune}
\chi(v) = \argmin_{\chi > 0} \mathbb{E}\|\eta(\xv/\sqrt{v}+\hv;\chi)-\xv/\sqrt{v}\|_2^2.
\end{equation}  
    When there are multiple solutions, we define $\chi(v)$ as the one with the smallest value. We first assume the limit $\lim_{v\rightarrow 0}\chi(v)=\alpha^*\in [0,\infty]$ exists, but will validate this assumption later. Recall the definition of the dual-norm ball $\mathcal{D}_\gamma$ of SLOPE norm in \eqref{eq:dual-ball}. Here, we consider the projection of $\xv/\sqrt{v} + \hv$ on $\mathcal{D}_{\chi(v)}$. Suppose $\alpha^* = \infty$. Since $\|\xv/\sqrt{v}+\hv\|_2\rightarrow \infty$ as $v \rightarrow 0$, we obtain
    \begin{equation*}
        \|\projv_{\mathcal{D}_{\chi(v)}}(\xv/\sqrt{v} + \hv)\|_2  \rightarrow \infty,
        \quad \text{as } v \rightarrow 0.
    \end{equation*}
    Hence, from Lemma \ref{p1} we conclude that as $v\rightarrow 0$, we have
\begin{equation*}
    \big \|\eta (\xv/\sqrt{v}+\hv;\chi(v))-\xv/\sqrt{v} \big \|_2
    =
   \big\|\projv_{\mathcal{D}_{\chi(v)}}  (\xv/\sqrt{v}+\hv ) -\hv \big\|_2 \rightarrow +\infty,
\end{equation*}
with which Fatou's lemma yields that
\begin{equation*}
    \varliminf_{v \rightarrow 0} \mathbb{E} \big\|\eta(\xv/\sqrt{v}+\hv;\chi(v))-\xv/\sqrt{v}\big \|_2^2
    \geq
    \mathbb{E} \varliminf_{v \rightarrow 0} \big\|\eta(\xv/\sqrt{v}+\hv;\chi(v))-\xv/\sqrt{v}\big \|_2^2 = +\infty.
\end{equation*}
This contradicts with the boundedness due to the definition of $\chi(v)$:
\begin{equation*}
\mathbb{E}\big \|\eta (\xv/\sqrt{v}+\hv;\chi(v))-\xv/\sqrt{v}\big \|_2^2 \leq \mathbb{E}\big \|\eta (\xv/\sqrt{v}+\hv ; 0)-\xv/\sqrt{v}\big \|_2^2=p.
\end{equation*}
Hence $\alpha^* \in [0,\infty)$ and $\chi(v)$ is bounded. Lemma \ref{lemma:dual-ball-diam} gives us that
\[
    \big \|\eta (\xv/\sqrt{v}+\hv;\chi(v))-\xv/\sqrt{v} \big \|^2_2
    \leq 2\chi^2(v) \|\lambdav\|^2_2 + 2\|\hv\|_2^2.
\]
Thus DCT enables us to obtain
\begin{align}
\lim_{v\rightarrow 0}\mathbb{E}\big \|\eta (\xv/\sqrt{v}+\hv;\chi(v))-\xv/\sqrt{v}\big \|_2^2=\mathbb{E}\lim_{v\rightarrow 0}\big \|\eta (\xv/\sqrt{v}+\hv;\chi(v))-\xv/\sqrt{v}\big \|_2^2.
\label{dct:limit}
\end{align}
To compute the limit on the right-hand side of the above equation, we apply Lemma \ref{lemma:bounding-mse-limit} and obtain that
\begin{equation} \label{optimal:amse:limit}
    \lim_{v\rightarrow 0}\mathbb{E}\big \|\eta (\xv/\sqrt{v}+\hv;\chi(v))-\xv/\sqrt{v}\big \|_2^2
    = k + (\alpha^*)^2\|\lambdav_{[1:k]}\|_2^2 + \mathbb{E}\|\eta(\hv_{[k+1:p]}; \alpha^*, \lambdav_{[k+1:p]})\|_2^2.
\end{equation}
Define $g(\alpha) := k + \alpha^2\|\lambdav_{[1:k]}\|_2^2 + \mathbb{E}\|\eta(\hv_{[k+1:p]}; \alpha, \lambdav_{[k+1:p]})\|_2^2$. Since $\chi(v)$ is defined as the optimal tuning, it has to hold that $\alpha=\alpha^*$ minimizes $g(\alpha)$. Finally, we need to prove the existence of $\lim_{v \rightarrow 0}\chi(v)$ that we assumed at the beginning of the proof. We take an arbitrarily convergent sequence $\{\chi(v_n)\}_{n=1}^{\infty}$ with $v_n \rightarrow 0$, as $n\rightarrow \infty$. Denote $\lim_{n\rightarrow \infty}\chi(v_n)=\tilde{\alpha}$. Note that the preceding arguments hold for any such sequence as well. Thus $\alpha=\tilde{\alpha}$ minimizes $g(\alpha)$ over $(0,\infty)$. The proof will be completed if we can show $g(\alpha)$ has a unique minimizer. According to Lemma \ref{lemma:prox-square-partial}, it is direct to compute 
\begin{align*}
    g'(\alpha)
    =& 2\alpha \sum_{i=1}^k\lambda_i^2-\frac{2}{\alpha}\mathbb{E}\big[\langle \eta(\hv_{[k+1:p]}; \alpha, \lambdav_{[k+1:p]}), \hv\rangle -\|\eta(\hv_{[k+1:p]}; \alpha, \lambdav_{[k+1:p]})\|_2^2 \big] \nonumber \\
    =& 2\alpha \sum_{i=1}^k\lambda_i^2 - 2\mathbb{E} \|\eta(\hv_{[k+1:p]}; \alpha, \lambdav_{[k+1:p]})\|_{\lambdav_{[k+1:p]}},
\end{align*}
where in the last equality we applied Lemma \ref{slope:prop} (\ref{lemma:item:prox-identity1}). It is not hard to see that $g'(\alpha)$ is increasing with $g'(0) = -2\mathbb{E}\|\eta(\hv_{[k+1:p]}; 0, \lambdav_{[k+1:p]})\|_{\lambdav_{[k+1:p]}}$ and $g'(\infty) = \infty$. Thus $g(\alpha)$ is strictly convex and has a unique minimizer.
\end{proof}

\subsubsection{Proof of Theorem \ref{thm:large-noise-ridge-better}}
\label{proof:of:theorem:large:noise}

According to Lemma \ref{optimal:fixed:point:eq}, the key step is to analyze the equation 
\begin{align}
\label{optimal:fixed:one:more}
\sigma^2=\sigma_z^2+\frac{1}{\delta p} \inf_{\chi>0}\mathbb{E}\|\eta(\xv+\sigma h; \sigma \chi)-\xv\|_2^2,
\end{align}
when $\sigma_z\rightarrow \infty$. Let $\sigma=\sigma^*$ be the solution to the above equation. First observe that $\forall \sigma>0$,
\begin{align}
\label{optimal:tuning:consequence}
\inf_{\chi>0}\mathbb{E}\|\eta(\xv+\sigma h; \sigma \chi)-\xv\|_2^2\leq \lim_{\chi \rightarrow \infty} \mathbb{E}\|\eta(\xv+\sigma h; \sigma \chi)-\xv\|_2^2=\|\xv\|_2^2.
\end{align}
This result combined with \eqref{optimal:fixed:one:more} yields
\begin{align*}
1\leq \frac{(\sigma^*)^2}{\sigma_z^2}\leq 1 +\frac{\|\xv\|_2^2}{\delta p \sigma_z^2},
\end{align*}
from which letting $\sigma_z\rightarrow \infty$ we obtain
\begin{align}
\label{limit:equal:one}
\lim_{\sigma_z\rightarrow \infty}\frac{(\sigma^*)^2}{\sigma_z^2}=1.
\end{align}
Moreover, adopting the notation from Lemma \ref{large:noise:optimal:tuning:value} we know
\begin{align*}
    e_{\lambdav}(\gamma^*_{\lambdav},\sigma_z)-\frac{\|\xv\|_2^2}{p}=\frac{1}{p}\Big[\mathbb{E}\|\eta(\xv+\sigma^*\hv;\sigma^*\chi(\sigma^*))\|_2^2-2\mathbb{E}\langle\eta(\xv+\sigma^*\hv;\sigma^*\chi(\sigma^*)),\xv \rangle\Big]
    :=\Delta(\sigma^*).
\end{align*}
Since $\Delta(\sigma^*) \leq 0$ implied by \eqref{optimal:tuning:consequence}, it holds that
\begin{align} \label{risk:upper:bound:form}
    |\Delta(\sigma^*)|
    \leq& \frac{2}{p}\mathbb{E}\langle\eta(\xv+\sigma^*\hv;\sigma^*\chi(\sigma^*)),\xv \rangle
    \leq \frac{2\sigma^*\|\xv\|_2}{\sqrt{p}} \|\eta(\xv/\sigma^*+\hv;\chi(\sigma^*))\|_{\mathcal{L}_2} \nonumber \\
    \leq& \frac{2\sigma^*\|\xv\|_2}{\sqrt{p}} \Big[\frac{1}{p} \sum_{i=1}^p \mathbb{E} (|x_i/\sigma^*+h_i| - \chi(\sigma^*) \|\lambdav\|_2^2 / p)_+^2 \Big]^{1/2},
\end{align}
where the third inequality is due to Lemma \ref{slope:prop} (\ref{item:basic-prox-bound}). As we will show in Lemma \ref{large:noise:optimal:tuning:value}, $\chi(\sigma^*) = \Omega(\sigma^*)$. This guarantees that as $\sigma^* \rightarrow \infty$, we will have $\frac{\|\xv\|_\infty}{\sigma^*} \leq \frac{\chi(\sigma^*) \|\lambdav\|_2^2}{2p}$. Using Gaussian tail inequality in Lemma \ref{lemma:numeric_bound}, it is hence straightforward to calculate that for each $i=1,\ldots, p$, as $\sigma^*\rightarrow \infty$,
\begin{align*}
    \mathbb{E}(|x_i/\sigma^*+h_i|-\chi(\sigma^*)\|\lambdav\|_2^2 / p)_+^2
    \leq \mathbb{E}[|h_i| - (\chi(\sigma^*)\|\lambdav\|_2^2 / p - |x_i|/\sigma^*)]_+^2
    \leq O(e^{-\frac{1}{4}\chi^2(\sigma^*)\|\lambdav\|_2^4 / p^2}).
\end{align*}
Based on Lemma \ref{large:noise:optimal:tuning:value}, the above result together with \eqref{limit:equal:one} and \eqref{risk:upper:bound:form} completes the proof.

\begin{lemma} \label{large:noise:optimal:tuning:value}
Suppose $\xv \neq \bm{0}$. Define 
\[
\chi(\sigma) = \argmin_{\chi > 0} \mathbb{E}\|\eta(\xv/\sigma+\hv;\chi)-\xv/\sigma\|_2^2.
\]
It holds that 
\[
\chi(\sigma)=\Omega(\sigma), \quad \mbox{~as~}\sigma \rightarrow \infty.
\]
\end{lemma}

\begin{proof}
We first claim that $\chi(\sigma)\rightarrow \infty$, as $\sigma \rightarrow \infty$. Otherwise, consider a sequence $\sigma_n\rightarrow \infty$ such that $\chi(\sigma_n)\rightarrow \chi^*\in [0,\infty)$, as $n\rightarrow \infty$. Then Dominated Convergence Theorem enables us to compute
\[
\lim_{n\rightarrow\infty}\mathbb{E}\|\eta(\xv/\sigma_n+\hv;\chi(\sigma_n))-\xv/\sigma_n\|_2^2=\mathbb{E}\|\eta(h;\chi^*)\|_2^2>0.
\]
On the other hand, by the definition of $\chi(\sigma_n)$, we obtain
\begin{align*}
\lim_{n\rightarrow\infty}\mathbb{E}\|\eta(\xv/\sigma_n+\hv;\chi(\sigma_n))-\xv/\sigma_n\|_2^2 \leq \lim_{n\rightarrow \infty}\frac{\|\xv\|_2^2}{\sigma_n^2}=0.
\end{align*}
This is a contradiction. We next analyze the rate of $\chi(\sigma)$. As we have shown in \eqref{optimal:tuning:consequence}, $\mathbb{E}\|\eta(\xv/\sigma+\hv;\chi(\sigma))-\xv/\sigma\|_2^2 \leq \frac{1}{\sigma^2}\|\xv\|_2^2$, it holds that $\forall \sigma >0$,
\begin{align} \label{the:basic:fact}
    \mathbb{E}\|\eta(\xv/\sigma+\hv;\chi(\sigma))\|_2^2\leq \frac{2}{\sigma}\mathbb{E}\langle\eta(\xv/\sigma+\hv;\chi(\sigma)),\xv \rangle.
\end{align}
With a change of variables, we can rewrite the terms as
\begin{align*}
    \mathbb{E}\|\eta(\xv/\sigma+\hv;\chi(\sigma))\|_2^2
    =& \frac{\chi^{p+2}(\sigma)}{(2\pi)^{p / 2}}\int \|\eta(\hv;1)\|_2^2\cdot \exp\Big(-\frac{\chi^2(\sigma)}{2}\|\hv-\frac{\xv}{\sigma \chi(\sigma)}\|_2^2\Big)d\hv.\\
    \mathbb{E}\langle\eta(\xv/\sigma+\hv;\chi(\sigma)),\xv \rangle
    =& \frac{\chi^{p+1}(\sigma)}{(2\pi)^{p / 2}} \int \langle \eta(\hv;1), \xv\rangle \cdot \exp\Big(-\frac{\chi^2(\sigma)}{2}\|\hv-\frac{\xv}{\sigma \chi(\sigma)}\|_2^2\Big)d\hv.
\end{align*}
By Laplace's approximation of multi-dimensional integrals \citep{wong2001asymptotic}, we can conclude that
\[
    \frac{\mathbb{E}\|\eta(\xv/\sigma+\hv;\chi(\sigma))\|_2^2}{\mathbb{E}\langle\eta(\xv/\sigma+\hv;\chi(\sigma)),\xv \rangle} \propto \frac{\sigma}{\chi(\sigma)}, \quad \text{as }\sigma \rightarrow \infty.
\]
Therefore, if $\chi(\sigma)=o(\sigma)$, the above result will contradict with \eqref{the:basic:fact}.
\end{proof}

\subsection{Basic properties of the proximal operator of SLOPE norm} \label{ssec:preliminary}
In this section, we prove various useful properties related to the proximal operator $\eta$ which is defined in \eqref{primal:prox}. The first property is a dual characterization of the primal definition of $\eta$.

\begin{lemma} \label{p1}
    The primal convex problem \eqref{primal:prox} has the dual form
    \begin{equation} \label{intro2}
        \vv^* \in \argmin_{\vv \in \mathcal{D}_{\gamma}} \|\uv- \vv\|_2^2,  
    \end{equation}
    where $\mathcal{D}_\gamma$ is defined in \eqref{eq:dual-ball}. Furthermore, strong duality holds, and the primal and dual solution pair $(\eta(\uv;\gamma), \vv^*)$ is unique and satisfies
    \begin{equation*}
        \vv^*=\uv-\eta(\uv;\gamma).
    \end{equation*}
\end{lemma}

\begin{proof}
    First of all, it is clear that \eqref{primal:prox} is strictly convex and
    $\eta(\uv;\gamma)$ is unique. The optimization \eqref{intro2} can be
    considered as projecting the point $\uv \in \mathbb{R}^p$ onto the closed
    convex set $\mathcal{D}_{\gamma}$, thus a unique solution $\vv^*$ exists.
    Now we connect the primal form and the dual form using the classical
    Fenchel duality framework. Let $\wv = \uv - \xv$. By substituing in $\wv$ and adding a
    Lagrangian multiplier $\vv$ for the constraint $\wv = \uv - \xv$, we obtain the
    following equivalent form of \eqref{primal:prox}:
    \begin{equation*}
        \max_{\vv} \min_{\xv, \wv} \frac{1}{2}\|\wv\|_2^2 + \gamma
        \|\xv\|_{\lambdav} - \langle \vv, \wv - \uv + \xv \rangle
        =
        \max_{\vv} \min_{\xv, \wv} \bigg\{ \frac{1}{2}\|\wv\|_2^2 - \langle
        \vv, \wv \rangle \bigg\} + \big\{ \gamma \|\xv\|_{\lambdav} - \langle
        \vv, \xv \rangle \big\} + \langle \vv, \uv \rangle.
    \end{equation*}
    
    The optimal $\wv^* = \vv$. Regarding minimizing over $\xv$, we have\footnote{Here we use the fact that $\|\vv\|_* = \max_{\|\uv\| \leq 1} \langle \uv, \vv \rangle$ for any norm $\|\cdot\|$ and its dual norm $\|\cdot\|_*$ in a Hilbert space.} 
    \[
    \min_{\xv} \gamma\|\xv\|_{\lambdav} - \langle \vv, \xv \rangle = - \|\xv\|_{\lambdav} \max_{\xv} \{\langle \vv, \xv / \|\xv\|_{\lambdav} \rangle - \gamma \} = \mathbb{I}_{\mathcal{D}_\gamma}(\vv). 
    \]
    Now the above Lagragian form reduces to
    \begin{equation*}
        -\frac{1}{2}\|\uv\|_2^2 + \min_{\vv \in \mathcal{D}_\gamma} \frac{1}{2}\|\uv - \vv\|_2^2,
    \end{equation*}
    which naturally leads to the optimal solution
    \begin{equation*}
        \vv^* = \projv_{\mathcal{D}_\gamma}(\uv),
    \end{equation*}
    The strong duality holds in this case, implying that
    \begin{equation*}
        \vv^* = \wv^* = \uv - \xv^* = \uv - \eta(\uv; \gamma).
    \end{equation*}

    The last piece of the proof deals with the characterization of $\mathcal{D}_\gamma$ in \eqref{eq:dual-ball}. We will use the relation $\|\vv\|_{\lambdav*}=\max_{\|\av\|_{\lambdav} \leq 1} \langle \av, \vv \rangle$, Without loss of generality, we assume $v_1 \geq \ldots v_p \geq 0$ (otherwise we permute the order and swap the signs of the components of $\av$ accordingly). It is not hard to see that the optimization problem can be rewritten as:
    \begin{equation*}
        \max_{\av} \langle \av, \vv \rangle,
        \qquad \text{subject to} \quad
        \|\av\|_{\lambdav} \leq 1, \quad
        a_1 \geq \ldots \geq a_p \geq 0.
    \end{equation*}
    It is equivalent to re-parameterize $\av$ using a vector $\bv$ with $a_i = \sum_{j=i}^p b_j$ and $b_j \geq 0$. Transforming the above constraints as Lagrange multipliers and optimizing over $\bv$, we get the following dual problem:
    \begin{equation*}
        \min_{\theta, \theta_i} \theta,
        \qquad \text{subject to} \quad
        \sum_{i=1}^j v_i - \theta \sum_{i=1}^j \lambda_i + \theta_j \leq 0,
        \quad \theta_j \geq 0,
        \quad \forall 1 \leq j \leq p,
        \quad \theta \geq 0.
    \end{equation*}
    Obviously given $\{\theta_j\}$,
    \begin{equation*}
        \hat{\theta} = \max_j \bigg\{ \frac{\sum_{i=1}^j v_i + \theta_j}{\sum_{i=1}^j \lambda_i} \bigg\}.
    \end{equation*}
    To further minimize over $\theta_j$, obviously we should set $\theta_j=0$ for all $j$ and the optimal value, $\|\vv\|_{\lambdav*}$, equals:
    \begin{equation*}
        \|\vv\|_{\lambdav*} = \max_j \bigg\{ \frac{\sum_{i=1}^j v_i}{\sum_{i=1}^j \lambda_i} \bigg\}.
    \end{equation*}
    As a corollary of this result, we may characterize $\mathcal{D}_\gamma$ as
    \begin{equation*}
        \mathcal{D}_\gamma
        = \{\vv: \|\vv\|_{\lambdav*} \leq \gamma\}
        = \Big\{\vv: \sum_{i=1}^j v_i \leq \gamma \sum_{i=1}^j \lambda_j, \quad \forall 1 \leq j \leq p \Big\}.
    \end{equation*}
\end{proof}

The primal form \eqref{primal:prox} and the dual form \eqref{intro2} enable us to obtain several useful properties of $\eta(\uv;\gamma)$. We select some of them to present here. We first analyze the primal form \eqref{primal:prox} to derive some properties of $\eta(\uv;\gamma)$.

\begin{lemma}\label{property:primal}
    Consider any given $\uv \in \mathbb{R}^p$ with $u_1\geq u_2\geq \cdots \geq u_p\geq 0$. The following results hold:
    \begin{enumerate}[(i)]
        \item \label{lemma:item:prox-scalar} $\eta(t\uv; t\gamma)=t\eta(\uv,\gamma)$ for $t\geq 0$.
        \item \label{lemma:item:prox-order} $\eta_1(\uv;\gamma)\geq \eta_2(\uv;\gamma)\geq  \cdots \geq \eta_p(\uv;\gamma)\geq 0$.
        \item \label{lemma:item:prox-compare} $u_i \geq \eta_i(\uv;\gamma), 1\leq i \leq p$.
        \item \label{lemma:item:prox-form} $\eta_j(\uv;\gamma) = \frac{[\sum_{i \in \mathcal{I}_j}(u_i-\gamma \lambda_i)]_+}{|\mathcal{I}_j|}$, where $\mathcal{I}_j$ is defined in \eqref{eq:tie-set}.
    \end{enumerate}
\end{lemma}

\begin{proof}
    Part (\ref{lemma:item:prox-scalar}) is because: $\eta(t \uv;t\gamma)=\argmin_{\xv}\frac{1}{2}\|\xv-t \uv\|_2^2+t \gamma \|\xv\|_{\lambdav}= \argmin_{\xv}\frac{1}{2}\|\xv/t-\uv\|_2^2+\gamma \|\xv/t\|_{\lambdav}.$ Part (\ref{lemma:item:prox-order}) is taken from Proposition 2.2 in \cite{bogdan2015slope}. For Part (\ref{lemma:item:prox-compare}), note that $\uv - \eta(\uv; \gamma) = \projv_{\mathcal{D}_\gamma}(\uv)$ with $\mathcal{D}_\gamma$ being symmetric around $\bm{0}$ and $u_i\geq 0$, we have $[\projv_{\mathcal{D}_\gamma}(\uv)]_i\geq 0$. This implies (\ref{lemma:item:prox-compare}). We now prove Part (\ref{lemma:item:prox-form}). First consider $\eta_j(\uv;\gamma)>0$. Denote $\mathcal{I}_j^{\min}=\min \{i: i \in \mathcal{I}_j\}$, $\mathcal{I}_j^{\max}=\max \{i: i \in \mathcal{I}_j\}$ and $\eta_i(u;\gamma)=a>0, i \in \mathcal{I}_j$. There exists a sufficiently small $\delta >0$, such that (we adopt the notation $\eta_{0}(u;\gamma)=+\infty, \eta_{p+1}(u;\gamma)=0$)
    \begin{equation*}
        \eta_{\mathcal{I}_j^{\max}+1}(u;\gamma)<a -\delta<a + \delta < \eta_{\mathcal{I}_j^{\min}-1}(u;\gamma).
    \end{equation*}
    
    Define a vector $\zv(b) \in\mathbb{R}^p:  z_i(b)=\eta_i( \uv;\gamma)$ for $1\leq i \leq \mathcal{I}_j^{\min}-1, \mathcal{I}_j^{\max}+1\leq i \leq p$, and $ z_i(b)=b$ for other $i$'s. Since $\eta(\uv;\gamma)$ is the minimizer of \eqref{primal:prox} we know
    \begin{equation*}
        \frac{1}{2}\|\uv-\eta(\uv;\gamma)\|_2^2 + \gamma \langle \lambdav, \eta(\uv;\gamma) \rangle \leq \frac{1}{2}\|\uv-\zv(b)\|_2^2+\gamma \langle \lambdav, \zv(b) \rangle,
    \end{equation*}
    holds for any $b \in [a- \delta, a+ \delta]$. Due to the choice of $\zv(b)$, we can further simplify the above inequality to obtain
    \begin{equation*}
        \frac{1}{2}\sum_{i=\mathcal{I}_j^{\min} }^{\mathcal{I}_j^{\max}} (u_i-a)^2 + \gamma \sum_{i=\mathcal{I}_j^{\min} }^{\mathcal{I}_j^{\max}} \lambda_i a \leq \frac{1}{2}\sum_{i=\mathcal{I}_j^{\min} }^{\mathcal{I}_j^{\max}} (u_i-b)^2 + \gamma \sum_{i=\mathcal{I}_j^{\min} }^{ \mathcal{I}_j^{\max}} \lambda_i b:=G(b),
    \end{equation*}
    where $b \in [a-\delta, a+\delta]$. Hence, $a$ is a local minima of the quadratic function $G(\cdot)$. Therefore
    \begin{equation*}
        0=\frac{dG(b)}{d b}\Big |_{b=a}=\sum_{i=\mathcal{I}_j^{\min} }^{\mathcal{I}_j^{\max}}(a-u_i+\gamma \lambda_i),
        \quad \Rightarrow \quad
        a = \frac{1}{\mathcal{I}_{j}^{\max} - \mathcal{I}_{j}^{\min} + 1} \sum_{i=\mathcal{I}_{j}^{\min}} ^{\mathcal{I}_{{j}}^{\max}} u_i - \gamma\lambda_i.
    \end{equation*}
    Regarding $\eta_j(\uv;\gamma)=0$, we can use the same arguments to conclude that $0$ is local minima of $G(\cdot)$ in $[0, \delta]$. So $\frac{d G(b)}{d b}\Big |_{b=0}\geq 0$ leads to the result.
\end{proof}

Next we show some properties relevant to the Lipschitz continuity, convexity and norm bounds of $\eta$, which are largely due to the dual form \eqref{intro2}. 

\begin{lemma}\label{slope:prop}
    For any $\uv \in \mathbb{R}^p$, the proximal operator $\eta(\uv;\gamma)$ satisfies,
    \begin{enumerate}[(i)]
        \item \label{lemma:item:prox-nonexpansive0} $\|\eta(\uv_1;\gamma)-\eta(\uv_2;\gamma )\|_2^2 \leq \langle \uv_1 - \uv_2, \eta(\uv_1;\gamma)-\eta(\uv_2;\gamma) \rangle \leq \|\uv_1 - \uv_2\|_2^2$;
        \item \label{lemma:item:prox-identity1} $\frac{1}{2}\|\uv-\eta(\uv;\gamma)\|_2^2+\gamma \|\eta(\uv;\gamma)\|_{\lambdav} = \frac{1}{2}(\|\uv\|_2^2-\|\eta(\uv;\gamma)\|_2^2)$
        \item \label{lemma:item:prox-nonincreasing}$\|\eta(\uv;\gamma)\|_2^2$ is convex in $\uv$ and non-increasing in $\gamma$.
        \item \label{lemma:item:prox-lip2} $\|\eta(\uv;\gamma_1)-\eta(\uv;\gamma_2)\|_2 \leq \|\lambdav \|_2|\gamma_1-\gamma_2|.$
        \item \label{item:basic-prox-bound} $ \|\eta(\uv; \gamma, \lambda_1 \bm{1})\|_2^2\leq \|\eta(\uv; \gamma, \lambdav)\|_2^2 \leq \|\eta(\uv; \gamma, \frac{\|\lambdav\|_1}{p} \bm{1})\|_2^2$.
    \end{enumerate}
\end{lemma}

\begin{proof}
    To prove (\ref{lemma:item:prox-nonexpansive0}), from Lemma \ref{p1} we know that $\projv_{\mathcal{D}_{\gamma}}(\uv_1)= \uv_1-\eta(\uv_1;\gamma)$ and $\projv_{\mathcal{D}_{\gamma}}(\uv_2)= \uv_2-\eta(\uv_2;\gamma)$. The property of projection onto a convexity body implies that
    \begin{equation*}
        \langle \uv_1-\projv_{\mathcal{D}_{\gamma}}(\uv_1), \projv_{\mathcal{D}_{\gamma}}(\uv_2) - \projv_{\mathcal{D}_{\gamma}}(\uv_1) \rangle \leq 0,
        \quad
        \langle \uv_2-\projv_{\mathcal{D}_{\gamma}}(\uv_2), \projv_{\mathcal{D}_{\gamma}}(\uv_1) - \projv_{\mathcal{D}_{\gamma}}(\uv_2) \rangle \leq 0.
    \end{equation*}
    
    Adding the two inequalities above up gives the first inequality of (\ref{lemma:item:prox-nonexpansive0}). The second one is by a simple use of Cauchy-Schwarz inequality. Part (\ref{lemma:item:prox-identity1}) is the strong duality property.
    
    For Part (\ref{lemma:item:prox-nonincreasing}), the equation in Part (\ref{lemma:item:prox-identity1}) is equivalent to
    \begin{equation*}
        \max_{\xv} \langle \uv, \xv \rangle -\frac{1}{2}\|\xv\|_2^2 - \gamma
        \|\xv\|_{\lambdav} = \frac{1}{2}\|\eta(\uv;\gamma)\|_2^2.
    \end{equation*}
    The term on the left-hand side is the maximum of a series of linear functions in $\uv$, hence convex. The monotonicity in $\gamma$ is obvious.

    Part (\ref{lemma:item:prox-lip2}): We first prove the inequality holds for $\uv$ that satisfies Lemma \ref{differential:prop2}. In this case we know there are finite number of discontinuity points of $\eta$ w.r.t. $\gamma$. Hence for all such $\uv$,
    \begin{align*}
        &\|\eta(\uv;\gamma +\Delta)-\eta(\uv;\gamma)\|_2
        =
        \bigg\|\Delta \int_0^1\frac{\partial \eta(\uv;\gamma+ t\Delta)}{\partial \gamma} dt \bigg\|_2 \\
        \leq&
        |\Delta| \int_0^1 \bigg\|\frac{\partial \eta(\uv;\gamma+ \Delta t)}{\partial \gamma} \bigg\|_2 dt \overset{(a)}{=}
        |\Delta| \int_0^1\sqrt{\textstyle{\sum}_{\mathcal{I} \in \mathcal{P}_0} \frac{1}{|\mathcal{I}|} (\textstyle{\sum}_{i\in \mathcal{I}} \lambda_i)^2} dt \\
        \leq&
        |\Delta| \int_0^1\sqrt{\textstyle{\sum}_{\mathcal{I} \in \mathcal{P}_0} \sum_{i\in \mathcal{I}} \lambda_i^2} dt
        \leq |\Delta| \|\lambdav\|_2,
    \end{align*}
    where $(a)$ is due to Lemma \ref{differential:prop2} (\ref{lemma:item:deri}). For other
    $\uv$'s, since they all belong to a Lebesgue measure zero set, there exists
    a sequence $\uv_m \rightarrow \uv$ and $\uv_m$ satisfies Part
    (\ref{lemma:item:prox-lip2}). Hence Part (\ref{lemma:item:prox-lip2}) holds for other $\uv$'s as well due to the continuity of $\eta(\cdot;\gamma)$.

    Part (\ref{item:basic-prox-bound}): For the upper bound, according to Lemma \ref{p1} it is sufficient to show $\big\{\vv: |v|_{(1)} \leq \gamma\|\lambdav\|_1/p\big\}\subseteq \mathcal{D}_{\gamma}$. According to the structure of $\mathcal{D}_{\gamma}$ in \eqref{eq:dual-ball}, the above set relation can be proved if $j\|\lambdav\|_1/p\leq \sum_{i=1}^j\lambda_i, \forall 1\leq j \leq p$. This is true because $\{\lambda_i\}_{i=1}^p$ is a non-increasing sequence. Regarding the lower bound, it is sufficient to show $\mathcal{D}_{\gamma} \subseteq \big\{\vv: |v|_{(1)} \leq \gamma \lambda_1\big\}$ which is obvious from the definition of $\mathcal{D}_{\gamma}$.
\end{proof}

The next two lemmas study the differentiability of $\eta(\uv;\gamma)$ that are useful in the proof. According to Lemma \ref{slope:prop} (\ref{lemma:item:prox-nonexpansive0}), $\eta(\uv;\gamma)$ is Lipschitz continuous, hence differentiable almost everywhere (with respect to $\uv$). In fact, from Lemma \ref{property:primal} (\ref{lemma:item:prox-form}), it seems possible to calculate the derivatives of $\eta(\uv;\gamma)$ outside a set of Lebesgue measure zero. Towards that goal, we slightly extend the notation of the partition $\mathcal{P}$ of $[p]$ to $\mathcal{P}(\uv, \gamma)$ to mark the dependency of the partition on $\uv$ and $\gamma$. $\mathcal{P}_0$ and $\mathcal{I}$ are extended in a similar fashion. Recall that $\mathcal{P}, \mathcal{P}_0$ are defined after \eqref{eq:tie-set}. 
\begin{lemma} \label{differential:prop}    
    Given any $\gamma>0$, there exists a Lebesgue measure zero set
    $\mathcal{L}_{\gamma} \subset \mathbb{R}^p$ such that for each $\uv \in \mathcal{L}^c_{\gamma}$,
    \begin{enumerate}[(i)]
        \item \label{lemma:item:prox-constant} There exists a sufficiently small ball $\mathcal{B}_{\epsilon}(\uv)=\{\tilde{\uv}: \|\tilde{\uv}-\uv\|_2\leq \epsilon\}$ such that the partition $\mathcal{P}(\tilde{\uv};\gamma)$ remains the same over $\mathcal{B}_{\epsilon}(\uv)$.
        \item \label{lemma:item:prox-smooth} $\eta(\cdot ;\gamma)$ is differentiable at $\uv$.
        \item \label{lemma:item:prox-magic} $\forall \vv,\tilde{\vv} \in \mathbb{R}^p$, $\sum_{i=1}^p \tilde{v}_i \langle \nabla \eta_i(\uv;\gamma), \vv \rangle = \sum_{\mathcal{I} \in \mathcal{P}_0} \frac{1}{|\mathcal{I}|}(\sum_{i \in \mathcal{I}} \tilde{v}_i\cdot \sign(u_i))(\sum_{i \in \mathcal{I}} v_i\cdot \sign(u_i))$. In particular, $\sum_{i=1}^p \eta_i(\uv;\gamma) \langle \nabla \eta_i(\uv;\gamma), \vv \rangle=\sum_{i=1}^p v_i\eta_i(\uv;\gamma)$.
    \end{enumerate}
\end{lemma}

\begin{proof}
    Part (\ref{lemma:item:prox-constant}), since the dual SLOPE norm ball is a polygon (with many faces),
    the orthogonal space of each face cut the entire space into many small
    regions, where the projection within each region is differentiable.
    Obviously the union of the boundaries of these regions is of measure 0. Let
    $S$ be the union of these boundaries. Then $S^c \subset \mathbb{R}^p$ is an
    open set, within which the projection is differentiable. This further
    implies the differentiability of $\eta(\uv; \gamma)$ in $\uv$ in $S^c$.
    
    Part (\ref{lemma:item:prox-smooth}) is a simple result of Part (\ref{lemma:item:prox-constant}) and Lemma \ref{property:primal}
    (\ref{lemma:item:prox-form}). For Part (\ref{lemma:item:prox-magic}), according to Part (\ref{lemma:item:prox-constant})
    and Lemma \ref{property:primal} (\ref{lemma:item:prox-form}), it is clear that
    \begin{align*}
        \nabla \eta_j(\uv;\gamma)=\bm 0,
        \quad
        \text{if } \eta_j(\uv;\gamma)=0,
        \qquad ~~
        [\nabla \eta_j(\uv;\gamma)]_i=\begin{cases}
            \frac{\sign(u_i)\cdot \sign(u_j)}{|\mathcal{I}_j|}, & i \in \mathcal{I}_j \\  
            0, & i \notin \mathcal{I}_j
        \end{cases},
        \quad
        \text{if } \eta_j(\uv;\gamma)\neq 0.
    \end{align*}
    The identity in Part (\ref{lemma:item:prox-magic}) can then be directly verified based on the above results. 
\end{proof}

\begin{lemma}\label{differential:prop2}
    Given almost any $\uv\in \mathbb{R}^p$, there exists a Lebesgue measure
    zero set $\mathcal{L}_{\uv} \subset \mathbb{R}_{++}$ such that for each $\gamma \in \mathcal{L}_{\uv}^c$
    \begin{enumerate}[(i)]
        \item \label{lemma:item:remain-const}The partition $\mathcal{P}(\uv; \tilde{\gamma})$ remains the same for all $\tilde{\gamma}\in [\gamma-\epsilon, \gamma+\epsilon]$ with $\epsilon$ sufficiently small.
        \item \label{lemma:item:deri} $\eta(\uv;\cdot)$ is differentiable at $\gamma$. Assuming $u_1\geq u_2\geq \cdots \geq u_p \geq 0$, for all $1\leq j \leq p$,
            \begin{equation*}
                \frac{\partial \eta_j(\uv;\gamma)}{\partial \gamma}=
                \begin{cases}
                    0 & \mbox{~if~}\eta_j(\uv;\gamma) =0 \\
                    \frac{-\sum_{i\in \mathcal{I}_j}\lambda_i}{|\mathcal{I}_j|} & \mbox{~otherwise}
                \end{cases}
            \end{equation*}
    \end{enumerate}
\end{lemma}

\begin{proof}
    Part (\ref{lemma:item:remain-const}): 
    Without loss of generality, we consider $u_1>u_2>\cdots>u_p>0$ and $\tilde{\gamma}=\gamma+\Delta$. Choosing $\Delta$ small enough gives that
    \begin{align*}
       & \eta_1(\uv;\gamma)\geq \eta_2(\uv;\gamma) \geq \cdots \geq \eta_{p-k}(\uv;\gamma)>0=\cdots=\eta_p(\uv;\gamma),\\
      &\eta_1(\uv;\tilde{\gamma})\geq \eta_2(\uv;\tilde{\gamma}) \geq \cdots \geq \eta_{p-\tilde{k}}(\uv;\tilde{\gamma})>0=\cdots=\eta_p(\uv;\tilde{\gamma}),
    \end{align*}
    where $k$ and $\tilde{k}$ are the number of zero components that $\eta(\uv;\gamma)$ and $\eta(\uv;\tilde{\gamma})$ have, respectively. The key inequality is,
    \begin{align}
        \|\eta(\uv;\tilde{\gamma})-\eta(\uv;\gamma)\|_2
        \overset{(a)}{=}&
        \bigg\|\frac{\tilde{\gamma}}{\gamma}\eta(\frac{\gamma}{\tilde{\gamma}}\uv;\gamma)-\eta(\uv;\gamma)\bigg\|_2 \leq \frac{\tilde{\gamma}}{\gamma}\bigg\|\eta(\frac{\gamma}{\tilde{\gamma}}\uv;\gamma)-\eta(\uv;\gamma)\bigg\|_2+\frac{|\tilde{\gamma}-\gamma|}{\gamma} \|\eta(\uv;\gamma)\|_2 \nonumber \\
        \overset{(b)}{\leq}&
        \frac{\tilde{\gamma}}{\gamma }\bigg\|\frac{\gamma}{\tilde{\gamma}}\uv-\uv\bigg\|_2+\frac{|\tilde{\gamma}-\gamma|}{\gamma} \|\uv\|_2=\frac{2|\tilde{\gamma}-\gamma|}{\gamma}\|\uv\|_2, \label{keyuse:proof}
    \end{align}
    where $(a)$ is by Lemma \ref{property:primal}
    (\ref{lemma:item:prox-scalar}) and $(b)$ is due to Lemma \ref{slope:prop}
    (\ref{lemma:item:prox-nonexpansive0}). Then \eqref{keyuse:proof} enables us to choose $\Delta$ small enough so that $\tilde{k}\leq k$. For the rest of the proof, we have 
    \begin{enumerate}[(1)]
        \item
        We first show $\tilde{k}=k$, which is equivalent to 
        \begin{align*}
            \sum_{j=p-k+1}^p|\eta(\uv;\tilde{\gamma})|^2=0,
        \end{align*}
        when $\Delta$ is small. Suppose this is not true. Then there exist $\Delta_n\rightarrow 0$ and $p-k+1\leq j_{\Delta_n}\leq p$ such that $\eta_{\Delta_n}(\uv;\tilde{\gamma})\neq 0$. Lemma \ref{property:primal} Part (\ref{lemma:item:prox-form}) gives that $\eta_{\Delta_n}(\uv;\tilde{\gamma})=\frac{\sum_{i\in \mathcal{I}_{j_{\Delta_n}}(\uv; \tilde{\gamma})}(u_i-\tilde{\gamma}\lambda_i)}{|\mathcal{I}_{j_{\Delta_n}}(\uv; \tilde{\gamma})|}$, and the inequality \eqref{keyuse:proof} implies that 
        \begin{align*}
        \lim_{\Delta\rightarrow 0}\sum_{j=p-k+1}^p|\eta(\uv;\tilde{\gamma})|^2=0.
        \end{align*}
        These result combined with the fact that $\lim_{n\rightarrow \infty} \frac{\Delta_n\sum_{i\in \mathcal{I}_{j_{\Delta_n}}(\uv; \tilde{\gamma})}\lambda_i}{|\mathcal{I}_{j_{\Delta_n}}(\uv; \tilde{\gamma})|}=0$ yield 
        \begin{align} \label{what:are}
            \lim_{n\rightarrow \infty} \sum_{i\in \mathcal{I}_{j_{\Delta_n}}(\uv; \tilde{\gamma})} (u_i-\gamma\lambda_i)=0.
        \end{align}
        Consider the set $\mathcal{L}_1=\{\gamma \in \mathbb{R}_{++}: \sum_{i\in \mathcal{K}}(u_i-\gamma \lambda_i)=0 \mbox{~for some~}\mathcal{K}\subseteq \{1,2,\ldots, p\}\}$. Since $u_i>0$ for all $1\leq i \leq p$, $\mathcal{L}_1$ has finite elements thus of Lebesgue measure zero. Hence, as long as $\gamma \in \mathcal{L}^c_1$, \eqref{what:are} is impossible to hold.
        
        \item We next show $\mathcal{I}_j(\uv; \gamma) = \mathcal{I}_j(\uv; \tilde{\gamma})$ for $1 \leq j \leq p-k$, where these sets are defined in \eqref{eq:tie-set}. Lemma \ref{property:primal} Part (iv) and the inequality \eqref{keyuse:proof} together imply that for each $1\leq j\leq p-k$,
        \begin{align} \label{middle:eq:step}
            \lim_{\Delta\rightarrow 0} \Big|\frac{\sum_{i\in \mathcal{I}_j(\uv; \gamma)} (u_i - \gamma \lambda_i)}{|\mathcal{I}_j(\uv; \gamma)|} - \frac{\sum_{i\in \mathcal{I}_j(\uv; \tilde{\gamma})} (u_i - \gamma \lambda_i)} {|\mathcal{I}_j(\uv; \tilde{\gamma})|} \Big|=0.
        \end{align}
        Now define the vector $\hv^{\Delta}\in \mathbb{R}^p$ so that for each $1\leq i \leq p$,
        \begin{align*}
        h_i^{\Delta}=
        \begin{cases}
        0 & \text{if } i \notin \mathcal{I}_j(\uv; \gamma), \; \& \; i \notin \mathcal{I}_j(\uv; \tilde{\gamma}), \\
        \frac{1}{|\mathcal{I}_j(\uv; \gamma)|}, &\text{if }i \in \mathcal{I}_j(\uv; \gamma) \; \& \; i\notin \mathcal{I}_j(\uv; \tilde{\gamma}), \\
        \frac{-1}{|\mathcal{I}_j(\uv; \tilde{\gamma})|}, &\text{if } i \notin \mathcal{I}_j(\uv; \gamma) \; \& \; i\in \mathcal{I}_j(\uv; \tilde{\gamma}), \\
        \frac{1}{|\mathcal{I}_j(\uv; \gamma)|}- \frac{1}{|\mathcal{I}_j(\uv; \tilde{\gamma})|}, & \text{otherwise}.
        \end{cases}
        \end{align*}
        Then, \eqref{middle:eq:step} can be rewritten as $\lim_{\Delta\rightarrow 0}\langle \hv^{\Delta},\uv-\gamma \lambdav \rangle=0$. Consider the set $\mathcal{L}_2=\{\gamma \in \mathbb{R}_{++}:   \langle \hv^{\Deltav}, \uv-\gamma \lambdav \rangle =0 \mbox{~for some~}\hv^{\Deltav}\neq \bm 0\}$. We know such set has finite elements as long as $\uv$ does not belong to the Lebesgue measure zero set $ \{\uv:  \langle \hv^{\Deltav}, \uv \rangle=0 \mbox{~for some~}\hv^{\Deltav}\neq \bm 0 \}$. Moreover, since the set $\{\hv^{\Delta}\in \mathbb{R}^p: \Delta \mbox{~is small}\}$ is finite, it holds that $\min_{\hv^{\Delta}\neq \mathbf{0}} \langle \hv^{\Delta},\uv-\gamma \lambdav \rangle>0$ for small $\Delta$ when $\gamma \in \mathcal{L}_2^c$. This combined with \eqref{middle:eq:step} implies that $\hv^{\Delta}=\mathbf{0}$ when $\Delta$ is small enough. 
    \end{enumerate}
    Part (\ref{lemma:item:deri}): It is a simple result of Part
    (\ref{lemma:item:remain-const}) and Lemma
    \ref{property:primal} (\ref{lemma:item:prox-form}).
\end{proof}

\begin{lemma} \label{lemma:dual-ball-diam}
    We have the following result for the diameter of the dual norm ball $\mathcal{D}_1$:
    \begin{equation*}
        \max \{\|\zv\|_2 : \zv \in \mathcal{D}_1\} = \|\lambdav\|_2.
    \end{equation*}
    This implies that $\mathbb{E}\|\eta(\xv + \hv; \chi) - \xv\|_2^2 \leq p + \chi^2\|\lambdav\|_2^2$.
\end{lemma}
\begin{proof} 
 We first have
    \begin{align*}
     \max_{\zv \in \mathcal{D}_1}\|\zv\|_2= \max_{\zv \in \mathcal{D}_1}\max_{\|\uv\|_2\leq 1} \langle \uv, \zv \rangle=\max_{\|\uv\|_2\leq 1} \max_{\zv \in \mathcal{D}_1} \langle \uv, \zv \rangle=\max_{\|\uv\|_2\leq 1} \|\uv\|_{\lambdav}\leq \|\lambdav\|_2,
    \end{align*}
    where the last inequality is due to Cauchy-Schwarz inequality. On the other hand, $\zv =\lambdav \in \mathcal{D}_1$ and $\|\zv\|_2 = \|\lambdav\|_2$. To justify the rest of the conclusions, we note that
    \begin{equation*}
        \mathbb{E}\|\eta(\xv + \hv; \chi) - \xv\|_2^2
        = \mathbb{E}\|\hv - \projv_{\mathcal{D}_\chi}(\xv + \hv)\|_2^2
        \leq p + \mathbb{E}\|\projv_{\mathcal{D}_\chi}(\xv + \hv)\|_2^2
        \leq p + \chi^2\|\lambdav\|_2^2.
    \end{equation*}
    where we used the fact that $\mathbb{E} \langle \hv, \projv_{\mathcal{D}_\chi}(\xv + \hv) \rangle = p - \mathbb{E} \langle \hv, \eta(\xv + \hv; \chi) - \eta(\xv; \chi) \rangle \geq 0$ due to Lemma \ref{slope:prop} (\ref{lemma:item:prox-nonexpansive0}).
\end{proof}

\begin{lemma} \label{lemma:bounding-mse-limit}
Let $k = \|\xv\|_0$ and suppose $\xv\in \mathbb{R}^p$ does not have tied non-zero elements. We have the following characterization of the limiting quantity:
    \begin{equation*}
        \lim_{\sigma \rightarrow 0} \|\eta(\xv / \sigma + \hv; \chi) - \xv / \sigma\|_{\mathcal{L}_2}^2
        = \frac{k}{p} + \frac{\chi^2}{p}\|\lambdav_{[1:k]}\|_2^2 + \|\eta(\hv_{[k+1:p]}; \chi, \lambdav_{[k+1:p]})\|_{\mathcal{L}_2}^2.
    \end{equation*}
\end{lemma}
\begin{proof}
Without loss of generality, suppose $|x_1| > \ldots > |x_k| > x_{k+1} = \ldots = x_p = 0$. Then as $\sigma \rightarrow 0$, the gap between any two consecutive terms of $\{\big|\frac{x_i} {\sigma} + h_i\big|\}_{i=1}^k$ converges to infinity. As a result, the proximal operator on this part becomes componentwise soft-thresholding. On the other hand, the rest $p-k$ components interact with $\lambdav_{[k+1:p]}$ to form a proximal operator independently from the first $k$ components. This leads to the following observation:
    \begin{align}
        &\lim_{\sigma \rightarrow 0} \eta_i(\frac{\xv}{\sigma} + \hv; \chi, \lambdav) - \frac{x_i}{\sigma}
        = h_i - \chi \lambda_i\sign(x_i), \quad 1 \leq i \leq k, \label{eq:low-noise-decomp-non0} \\
        &\lim_{\sigma \rightarrow 0} \eta_{[k+1 : p]}(\frac{\xv}{\sigma} + \hv; \chi, \lambdav) - \frac{\xv_{[k+1 : p]}}{\sigma}
        = \eta(\hv_{[k+1:p]}; \chi, \lambdav_{[k+1:p]}). \label{eq:low-noise-decomp-0}
    \end{align}
    It is important to note that here $h_i$'s are not ordered and hence $\hv_{[k+1:p]} \sim \mathcal{N}(0, \Iv_{p-k})$ and is independent from $h_i$ for $i \leq k$. This indicates the identity below    \begin{equation*}
        \lim_{\sigma \rightarrow 0} \|\eta(\xv / \sigma + \hv; \chi) - \xv / \sigma\|_{2}^2
        = \|\hv_{[1:k]}\|_2^2 + \chi^2 \|\lambdav_{[1:k]}\|_2^2 -2\chi\sum_{i=1}^k\lambda_ih_i\sign(x_i)+ \|\eta(\hv_{[k+1:p]}; \chi, \lambdav_{[k+1:p]})\|_2^2,
    \end{equation*}
which combined with dominated convergence theorem completes the proof. 
\end{proof}

\begin{lemma} \label{lemma:prox-square-partial}
    Let $f(a, b) = \|\eta(\xv + a\hv; b)\|_2^2$, then at those differentiable points of $f$, we have the following
    equations for the partial derivatives of $f$
    \begin{align}
        \frac{\partial f}{\partial a} =& 2\big\langle \eta(\xv + a\hv; b), \hv \big\rangle, \label{eq:prox-partial1} \\
        \frac{\partial f}{\partial b} =& \frac{2}{b}\big\| \eta(\xv + a\hv; b) \|_2^2 - \frac{2}{b} \big\langle \eta(\xv + a\hv; b), \xv + a\hv \big\rangle. \label{eq:prox-partial2}
    \end{align}
\end{lemma}
\begin{proof}
  By a simple application of the chain rule, $\frac{\partial f}{\partial a}=2\sum_{i=1}^p\eta_i \langle \nabla \eta_i, \hv  \rangle$. Setting $\tilde{\vv}=\eta, \vv=\hv$ in Lemma \ref{differential:prop} Part (\ref{lemma:item:prox-magic}) yields $\sum_{i=1}^p\eta_i \langle \nabla \eta_i, \hv  \rangle= \langle  \eta, \hv  \rangle$, leading to \eqref{eq:prox-partial1}. Regarding \eqref{eq:prox-partial2}, since $f(a, b) = b^2 \|\eta(\xv / b + a \hv / b; 1)\|_2^2$, using Lemma \ref{differential:prop} Part (\ref{lemma:item:prox-magic}) in a similar way we obtain
    \begin{align*}
        \frac{\partial f}{\partial b}
        =& 2b \|\eta(\xv / b + a \hv / b; 1)\|_2^2 + 2b^2 \langle \eta(\xv / b + a \hv / b; 1), -\xv / b^2 - a\hv / b^2 \rangle, \nonumber \\
        =& \frac{2}{b} \|\eta(\xv + a \hv; b)\|_2^2 - \frac{2}{b} \langle \eta(\xv + a \hv; b), \xv + a\hv \rangle.
    \end{align*}
\end{proof}

\begin{lemma}\label{lemma:prox-mse-partial}
    Let $G(a) = \|\eta(\xv + ac\hv; ab) - \xv\|_2^2$, then we have
    \begin{equation*}
        G'(a)
        =
        \frac{2}{a}\Big(\|\eta(\xv + ac\hv; ab)\|_2^2 - 2\langle \eta(\xv + ac\hv; ab), \xv \rangle + \sum_{\mathcal{I} \in \mathcal{P}_0}\frac{(\sum_{k \in \mathcal{I}} x_k \cdot \sign(x_k + ach_k))^2} {|\mathcal{I}|} \Big),
    \end{equation*}
    where the notation $\mathcal{P}_0$ is defined after \eqref{eq:tie-set}, and we use it here to denote the partitions with respect to $\eta(\xv + ac\hv; ab)$.
\end{lemma}

\begin{proof}
    Since we only care about the derivative, we first ignore the constant term and rewrite $G(a) = \|\eta(\xv + ac\hv; ab)\|_2^2 - 2 \langle \eta(\xv + ac\hv; ab), \xv \rangle$. By Lemma \ref{lemma:prox-square-partial}, it is not hard to verify that
    \begin{equation} \label{eq:prox-norm2-deriv}
        \frac{d}{da}\|\eta(\xv + ac\hv; ab)\|_2^2,
        = \frac{2}{a}\|\eta(\xv + ac\hv; ab)\|_2^2 - \frac{2}{a}\langle \eta(\xv + ac\hv; ab), \xv \rangle.
    \end{equation}

    Now for the second term, we have that
    \begin{align*}
        \frac{d}{da}\langle \eta(\xv + ac\hv; ab), \xv \rangle
        =& \frac{d}{da} a\langle \eta(\xv / a + c\hv; b), \xv \rangle \nonumber \\
        =& \frac{1}{a}\langle \eta(\xv + ac\hv; ab), \xv \rangle + a\sum_{k=1}^p x_k \langle \nabla\eta_k(\xv/a + c\hv; b), - \xv / a^2 \rangle \nonumber \\
        =& \frac{1}{a}\langle \eta(\xv + ac\hv; ab), \xv \rangle - \frac{1}{a}\sum_{\mathcal{I} \in \mathcal{P}_0} \frac{1}{|\mathcal{I}|}\Big(\sum_{k \in \mathcal{I}} x_k \cdot \sign(x_k + ac h_k) \Big)^2,
    \end{align*}
    where the last equality is due to Lemma \ref{differential:prop} Part (\ref{lemma:item:prox-magic}). The proof is completed by combining the above two parts.
\end{proof}

\subsection{Reference materials} \label{ssec:reference}
In this section, we summarize a few results which have been proved in previous works
and are used in our paper.

\subsubsection{Convex Gaussian Min-max Theorem (CGMT)} \label{sssec:cgmt}
The Convex Gaussian Min-max Theorem (CGMT) provides a powerful tool to analyze SLOPE estimator under
i.i.d. Gaussian designs. Denote
\begin{align*}
    \Phi(\Gv):=&\min_{\wv\in S_{\wv}} \max_{\uv \in S_{\uv}, \vv \in S_{\vv}} \uv^\top \Gv\wv+\psi(\wv, \uv, \vv), \\
    \phi(\gv, \hv):=& \min_{\wv\in S_{\wv}}\max_{\uv \in S_{\uv}, \vv \in S_{\vv}} \|\wv\|_2\gv^\top \uv+\|\uv\|_2\hv^\top \wv+\psi(\wv, \uv, \vv), \\
    \tilde{\phi}(\gv, \hv):=& \max_{\uv \in S_{\uv}, \vv \in S_{\vv}} \min_{\wv\in S_{\wv}} \|\wv\|_2\gv^\top \uv+\|\uv\|_2\hv^\top \wv+\psi(\wv, \uv, \vv),
\end{align*}
where $\Gv \in \mathbb{R}^{n\times p}, \hv \in \mathbb{R}^p, \gv \in \mathbb{R}^n$ have independent standard normal entries.

\begin{theorem}\label{thm:cgmt}
(CGMT). Suppose $S_{\wv}, S_{\uv}, S_{\vv}$ are all non-empty compact sets, and $\psi(\wv,\uv, \vv)$ is continuous on $S_{\wv}\times S_{\uv} \times S_{\vv}$, then the following results hold:
\begin{enumerate}[(i)]
    \item \label{thm:item:cgmt-upper} For all $c\in \mathbb{R}$, 
    \begin{equation*}
        \mathbb{P}(\Phi(G)\leq c) \leq 2 \mathbb{P}(\phi(\gv, \hv) \leq c).
    \end{equation*}
\item \label{thm:item:cgmt-lower} Further assume that $S_{\wv}, S_{\uv}, S_{\vv}$ are convex sets, and $\psi(\wv,\uv, \vv)$ is convex on $S_{\wv}$ and concave on $S_{\uv} \times S_{\vv}$. Then for all $c\in \mathbb{R}$,
    \begin{equation*}
        \mathbb{P}(\Phi(G)\geq c) \leq 2 \mathbb{P}(\tilde{\phi}(\gv, \hv) \geq c).
    \end{equation*}
\end{enumerate}
\end{theorem}

The above results are essentially taken from Theorem 3 in \cite{thrampoulidis2015regularized}. The minor difference is that the current version involves an extra vector $\vv$, and $\tilde{\phi}(\gv, \hv)$ appears in Part (\ref{thm:item:cgmt-lower}) instead of $\phi(\gv, \hv)$. By a rather straightforward inspection of the proof in \cite{thrampoulidis2015regularized}, these changes continue to hold.

\subsubsection{Concentration inequalities results}
We list some well known concentration results. In these theorems, $C, c$ are used to denote absolute constants.

\begin{theorem}[Bernstein's inequality] \label{thm:concen-bernstein}
    Let $x_1,\ldots, x_n$ be independent, mean zero, sub-exponential random variables. Then for every $t\geq 0$, we have
    \begin{equation*}
        \mathbb{P}\bigg(\bigg|\sum_{i=1}^nx_i \bigg|\geq t\bigg) \leq
        2\exp\bigg[-c\cdot \min \bigg(\frac{t^2}{\sum_{i=1}^n\|x_i\|^2_{\psi_1}},
        \frac{t}{\max_i \|x_i\|_{\psi_1}} \bigg) \bigg],
    \end{equation*}
    where $\|\cdot\|_{\psi_1}$ is the sub-exponential norm defined as $\|x\|_{\psi_1}=\inf\{t>0: \mathbb{E}e^{|x|/t}\leq 2\}$.
\end{theorem}
Please refer to Theorem 2.8.1 in \cite{vershynin2018high} for a proof.

\begin{theorem}[Gaussian concentration] \label{thm:lip-gauss}
    Consider a random vectror $X \sim \mathcal{N}(0, \Iv_p)$ and a Lipschitz
    function $f:\mathbb{R}^p \rightarrow \mathbb{R}$, then
    \begin{equation*}
        \mathbb{P}(|f(X) - \mathbb{E}f(X)| > t)
        \leq 2 e^{-\frac{c t^2}{\|f\|_{\mathrm{Lip}}^2}},
        \quad \forall t \geq 0.
    \end{equation*}
\end{theorem}
See Theorem 5.2.2 in \cite{vershynin2018high} for a proof.

\begin{theorem}[Matrix deviation inequality] \label{thm:matdev}
    Let $\Av$ be an $m\times n$ matrix whose rows $\Av_i$ are independent, isotropic and sub-Gaussian random vectors in $\mathbb{R}^n$. Then for any subset $\mathit{T}\subseteq 
    \mathbb{R}^n$, we have for any $u\geq 0$, the event 
    \begin{equation*}
\sup_{\xv \in \mathit{T}}\Big|\|\Av \xv\|_2-\sqrt{m}\|\xv\|_2\Big|\leq CK^2(w(\mathit{T})+u\cdot {\rm rad}(\mathit{T}))
    \end{equation*}
    holds with probability at least $1-2e^{-u^2}$. Here, $K=\max_{i}\|\Av_i\|_{\psi_2}$, and $w(\mathit{T}), {\rm rad}(\mathit{T})$ are defined as:
    \[
    w(\mathit{T})=\mathbb{E}\sup_{\xv \in \mathit{T}}\langle \gv, \xv \rangle,~\gv \sim \mathcal{N}(0,\Iv_n); \quad {\rm rad}(\mathit{T})=\sup_{\xv \in \mathit{T}}\|\xv\|_2.
    \]
\end{theorem}
See Theorem 9.1.1 and Exercise 9.1.8 in \cite{vershynin2018high} for a proof.

\subsubsection{Other results}

\begin{theorem}[Saddle Point Theorem] \label{saddle:thm}
Let $X$ and $Z$ be two nonempty convex subsets of $\mathbb{R}^n$ and $\mathbb{R}^m$, respectively; and $\phi: X \times Z \mapsto \mathbb{R}$ be a function such that $\phi(\cdot, z)$ is convex and closed over $X$ for each $z\in Z$, and $-\phi(x,\cdot)$ is convex and closed over $Z$ for each $x\in X$. If for some $\bar{x}\in X, \bar{z}\in Z, \bar{c}\in \mathbb{R}$, the levels sets 
\begin{equation*}
    \{x \in X: \phi(x, \bar{z})\leq \bar{c}\},
    \quad
    \{z \in Z: \phi(\bar{x},z) \geq \bar{c}\},
\end{equation*}
are nonempty and compact, then the set of saddle points of $\phi$ is nonempty and compact. 
\end{theorem}

The above theorem is Proposition 5.5.7 in \cite{bertsekas2009convex}.

%\begin{lemma} \label{lemma:C-gamma}
 %   We have that
 %   \begin{equation*}
  %      \frac{\#\{|x_i| \geq \frac{\|\xv\|_2}{\sqrt{2p}}\}}{p}  \geq \frac{\|\xv\|_2^2}{2\|\xv\|_\infty^2 p}.
 %   \end{equation*}
%\end{lemma}
%\begin{proof}
 %   Let $C < \frac{\|\xv\|_2}{\sqrt{p}}$ be a positive number. With a similar argument as those made in the proof of Lemma \ref{lemma:large-weights-proportional}, we can prove that $\frac{\#\{|x_i| \geq C\}}{p} \geq \frac{\frac{\|\xv\|_2^2}{p} - C^2}{\|\xv\|_\infty^2 - C^2}$. The result then follows.
%\end{proof}

\begin{lemma} \label{lemma:numeric_bound}
    Let $Z \sim \mathcal{N}(0, 1)$ and $x\geq 0$ be a constant. Recall $\Phi(\cdot)$ and $\phi(\cdot)$ are the cdf and pdf of a standard normal respectively. We have the following inequalities:
    \begin{enumerate}[(i)]
        \item \label{item:gaussian-ineq1} $\frac{1}{2} \mathbb{E}(|Z| - x)_+^2
        = (1 + x^2)\Phi(-x) - x\phi(x) \geq  c\phi(\sqrt{2} x)$ for an absolute constant $c>0$.
        \item \label{item:gaussian-ineq2} $(1 + x^2)\Phi(-x) - x\phi(x)  \leq \Phi(-x) \leq \frac{2\phi(x)}{1 + x}$.
    \end{enumerate}
\end{lemma}
\begin{proof}
Part \eqref{item:gaussian-ineq1}: The equation can be easily confirmed by simple integral calculations. For the inequality, we first note that $\lim_{x\rightarrow \infty}\frac{(1 + x^2)\Phi(-x) - x\phi(x)}{\phi(\sqrt{2} x)}\rightarrow \infty$ by by L'Hopital rule, hence there exists a constant $x_0>0$ such that $(1 + x^2)\Phi(-x) - x\phi(x)\geq \phi(\sqrt{2} x)$ for all $x\geq x_0$. When $x<x_0$, we have $\frac{1}{2} \mathbb{E}(|Z| - x)_+^2\geq \frac{x_0^2}{2}\mathbb{P}(|Z|\geq 2x_0)$. Hence, we can set $c=1 \wedge \frac{\sqrt{2\pi}x_0^2}{2}\mathbb{P}(|Z|\geq 2x_0)>0$. 

Part \eqref{item:gaussian-ineq2}: The first inequality is equivalent to $x\Phi(-x)\leq \phi(x)$. This holds because $\Phi(-x)=\int_{x}^{\infty}\phi(z)dz\leq x^{-1}\int_{x}^{\infty}z\phi(z)dz= x^{-1}\phi(x)$. For the second one, it is sufficient to show $f(x):=2\phi(x)-(1+x)\Phi(-x)\geq 0, \forall x\geq 0$. First it is straightforward to check that $\lim_{x\rightarrow \infty}\frac{2\phi(x)}{(1+x)\Phi(-x)}=2$. Thus there exists a constant $x_0>0$ such that $f(x)>0$ for all $x\geq x_0$. It is also clear that $f(0)>0$. Finally, if the global minimizer over $(0,x_0)$ is an interior point denoted by $x^*$. Then it satisfies $0=f'(x^*)=(1-x^*)\phi(x^*)-\Phi(-x^*)$, which implies that $f(x^*)=2\phi(x^*)-(1+x^*)\Phi(-x^*)=(1+(x^*)^2)\phi(x^*)>0$.
\end{proof}

\bibliographystyle{plain}

\bibliography{mc}
%
% once the .bbl file has been generated then place the text in your article.

% To get the numbered reference style the author should use [numbib]
%as an option in the document class.  For example: \documentclass[numbib]{imaiai}

\end{document}